\documentclass[twoside]{article}
\usepackage[margin=.9in]{geometry}
\usepackage{microtype}
\usepackage{graphicx}
\usepackage{subfigure}
\usepackage{booktabs} 
\usepackage{natbib}
\usepackage{algorithm}
\usepackage[noend]{algorithmic}
\usepackage{amsmath,amsthm,amsfonts,amssymb}
\usepackage{amsmath}
\usepackage{bbm}
\usepackage{color}
\usepackage{mathrsfs}
\usepackage{enumitem}
\usepackage{bm}
\usepackage{multirow}
\usepackage{booktabs}
\usepackage{makecell}
\usepackage{graphicx}
\usepackage{subfigure}
\usepackage{caption}
\usepackage{thmtools}
\usepackage{times}

\newcommand{\tld}{\tilde{d}}

\newcommand{\TV}{\text{TV}}
\newcommand{\sumton}{\frac{1}{n}\sum_{i=1}^{n}}
\newcommand{\sumtor}{\sum_{i=1}^{r}}
\newcommand{\barLambda}{\bar{\mLambda}}
\newcommand{\mR}{\mathbb{R}}
\newcommand{\n}{\mathbf{n}}

\newcommand{\pv}{P_{\V}}
\newcommand{\pu}{P_{\U}}
\newcommand{\pt}{P_{\cT}}

\newcommand{\barGamma}{\bar{\mGamma}}
\newcommand{\bb}{\mathbf{b}}
\newcommand{\bepsilon}{\bm{\epsilon}}
\newcommand{\bkappa}{\bar{\kappa}}

\newcommand{\tkappa}{\tilde{\kappa}}
\newcommand{\tnu}{\tilde{\nu}}

\newcommand{\bE}{\mathbf{E}}
\newcommand{\C}{\mathbf{C}}
\newcommand{\hB}{\hat{\B}}

\newcommand{\halpha}{\hat{\balpha}}

\newcommand{\ran}[1]{\text{range}{#1}}
\newcommand{\Cmax}{C_{\max}}
\newcommand{\Cmin}{C_{\min}}
\newcommand{\cond}{C_{\text{cond}}}

\newcommand{\KL}{\text{KL}}
\newcommand{\btheta}{\bm{\theta}}
\newcommand{\xstar}{\mathbf{x}_{\star}}
\newcommand{\balpha}{\bm{\alpha}}
\newcommand{\bbeta}{\bm{\beta}}

\newcommand{\Gr}{\text{Gr}_{r,d}(\mR)}

\newcommand{\sG}{\text{sG}}
\newcommand{\sE}{\text{sE}}

\usepackage[colorlinks,citecolor=blue,urlcolor=black,linkcolor=blue,pdfborder={0 0 0}]{hyperref}

\usepackage[capitalize]{cleveref}  
\crefname{nlem}{Lemma}{Lemmas}
\crefname{nprop}{Proposition}{Propositions}
\crefname{ncor}{Corollary}{Corollaries}
\crefname{nthm}{Theorem}{Theorems}
\crefname{exa}{Example}{Examples}
\crefname{assumption}{Assumption}{Assumptions}
\crefname{equation}{}{}

\theoremstyle{plain}
\newtheorem{theorem}{Theorem}

\newtheorem{assumption}{Assumption}
\newtheorem{lemma}{Lemma}
\newtheorem{proposition}[theorem]{Proposition}

\theoremstyle{definition}

\newtheorem{definition}{Definition}

\newcommand{\vect}[1]{\ensuremath{\mathbf{#1}}}
\newcommand{\diag}{\text{diag}}
\newcommand{\mat}[1]{\ensuremath{\mathbf{#1}}}
\newcommand{\dd}{\mathrm{d}}
\newcommand{\grad}{\nabla}
\newcommand{\hess}{\nabla^2}
\newcommand{\argmin}{\mathop{\rm argmin}}

\newcommand{\Ind}[1]{\mathbbm{1}{#1}}

\newcommand{\tr}{\mathrm{tr}}

\newcommand{\trans}{^{\top}}
\newcommand{\poly}{\mathrm{poly}}
\newcommand{\polylog}{\mathrm{polylog}}

\newcommand{\abs}[1]{|{#1}|}
\newcommand{\norm}[1]{\|{#1} \|}
\newcommand{\fnorm}[1]{\|{#1} \|_{\text{F}}}

\newcommand{\mE}{\mathbb{E}}
\renewcommand{\Pr}{\mathbb{P}}

\newcommand{\la}{\langle}
\newcommand{\ra}{\rangle}

\newcommand{\tlO}{\tilde{O}}

\newcommand{\tlTheta}{\tilde{\Theta}}

\renewcommand{\H}{\mathcal{H}}
\newcommand{\G}{\mathcal{G}}

\newcommand{\sigstarl}{\sigma^\star_1}
\newcommand{\sigstarr}{\sigma^\star_r}

\newcommand{\A}{\mat{A}}
\newcommand{\B}{\mat{B}}
\newcommand{\Bone}{\hat{\mat{B}}}
\newcommand{\Btwo}{\hat{\mat{B}}_{\perp}}

\newcommand{\mS}{\mat{S}}

\newcommand{\I}{\mat{I}}
\newcommand{\D}{\mat{D}}
\newcommand{\U}{\mat{U}}
\newcommand{\V}{\mat{V}}
\newcommand{\M}{\mat{M}}
\newcommand{\N}{\mat{N}}
\newcommand{\W}{\mat{W}}
\newcommand{\X}{\mat{X}}
\newcommand{\Y}{\mat{Y}}
\newcommand{\R}{\mat{R}}
\newcommand{\Z}{\mat{Z}}

\newcommand{\E}{\mat{E}}
\newcommand{\F}{\mat{F}}
\newcommand{\mSigma}{\mat{\Sigma}}
\newcommand{\mLambda}{\mat{\Lambda}}
\newcommand{\mGamma}{\mat{\Gamma}}
\newcommand{\e}{\vect{e}}
\renewcommand{\u}{\vect{u}}
\renewcommand{\v}{\vect{v}}

\newcommand{\x}{\vect{x}}
\newcommand{\y}{\vect{y}}
\newcommand{\z}{\vect{z}}

\newcommand{\cN}{\mathcal{N}}

\newcommand{\cT}{\mathcal{T}}
\newcommand{\cW}{\mathcal{W}}

\newcommand{\cn}{\kappa}
\newcommand{\nn}{\nonumber}

\title{\textbf{Provable Meta-Learning of Linear Representations}}
\author{
Nilesh Tripuraneni \\
University of California, Berkeley\\
\texttt{nilesh\_tripuraneni@berkeley.edu}\\
\and
Chi Jin\\
Princeton University\\
\texttt{chij@princeton.edu}\\
\and
Michael I. Jordan\\
University of California, Berkeley\\
 \texttt{jordan@cs.berkeley.edu}
}
\date{}
\begin{document}

\maketitle

\begin{abstract}

Meta-learning, or learning-to-learn, seeks to design algorithms that can utilize previous experience to rapidly learn new skills or adapt to new environments. Representation learning---a key tool for performing meta-learning---learns a data representation that can transfer knowledge across multiple tasks, which is essential in regimes where data is scarce. Despite a recent surge of interest in the practice of meta-learning, the theoretical underpinnings of meta-learning algorithms are lacking, especially in the context of learning transferable representations. In this paper, we focus on the problem of multi-task linear regression---in which multiple linear regression models share a common, low-dimensional linear representation. Here, we provide provably fast, sample-efficient algorithms to address the dual challenges of (1) learning a common set of features from multiple, related tasks, and (2) transferring this knowledge to new, unseen tasks. Both are central to the general problem of meta-learning. Finally, we complement these results by providing information-theoretic lower bounds on the sample complexity of learning these linear features.
\end{abstract}


\section{Introduction}
\label{sec:prelim}

The ability of a learner to transfer knowledge between tasks is crucial for robust, sample-efficient inference and prediction. One of the most well-known examples of such \emph{transfer learning} has been in few-shot image classification where the idea is to initialize neural network weights in early layers using ImageNet pre-training/features, and subsequently re-train the final layers on a new task \citep{donahue2014decaf, vinyals2016matching}. However, the need for methods that can learn data representations that generalize to multiple, unseen tasks has also become vital in other applications, ranging from deep reinforcement learning \citep{baevski2019cloze} to natural language processing \citep{ando2005framework, liu2019multi}.  Accordingly,
researchers have begun to highlight the need to develop (and understand) generic algorithms for transfer (or meta) learning applicable in diverse domains \citep{finn2017model}. Surprisingly, however, despite a long line of work on transfer learning, there is limited theoretical characterization of the underlying problem.  Indeed, there are few efficient algorithms for feature learning that \textit{provably} generalize to new, unseen tasks. Sharp guarantees are even lacking in the \textit{linear} setting. 

In this paper, we study the problem of meta-learning of features in a linear model in which multiple tasks share a common set of low-dimensional features.
Our aim is twofold. First, we ask: given a set of diverse samples from $t$ different tasks how we can efficiently (and optimally) learn a common feature representation? Second, having learned a common feature representation, how can we use this representation to improve sample efficiency in a new ($t+1$)st task where data may be scarce?\footnote{This problem is sometimes referred to as learning-to-learn (LTL).} 

Formally, given an (unobserved) linear feature matrix $\B = ( \bb_1, \ldots, \bb_r )\in \mR^{d\times r}$ with orthonormal columns, our statistical model for data pairs $(\x_i, y_i)$ is:
\begin{equation}\label{eq:model_main}
y_i = \x_i\trans \B\balpha_{t(i)} + \epsilon_i \quad ; \quad \bbeta_{t(i)} = \B \balpha_{t(i)},
\end{equation}
where there are $t$ (unobserved) underlying task parameters $\balpha_j$ for $j \in \{1, \hdots, t \}$. Here $t(i) \in \{1, \hdots, t \}$ is the index of the task associated with the $i$th datapoint, $\x_i \in \mR^d$ is a random covariate, and $\epsilon_i$ is additive noise. We assume the sequence $\{\balpha_{t(i)}\}_{i=1}^\infty$ is independent of all other randomness in the problem. In this framework, the aforementioned questions reduce to recovering $\B$ from data from the first $\{1, \hdots, t\}$ tasks, and using this feature representation to recover a better estimate of a new task parameter, $\bbeta_{t+1} = \B \balpha_{t+1}$, where $\balpha_{t+1}$ is also unobserved.
 
 Our main result targets the problem of learning-to-learn (LTL), and shows how a feature representation $\Bone$ learned from $t$ diverse tasks can improve learning on an unseen ($t+1$)st task which shares the same underlying linear representation. We informally state this result  below.\footnote{\cref{thm:main_informal} follows immediately from combining \cref{thm:B_recovery,,thm:lr_transfer}; see \cref{thm:main_formal} for a formal statement.}
\begin{theorem}[Informal]
	Suppose we are given $n_1$ total samples from $t$ diverse and normalized tasks which are used in \cref{algo:LF_learn} to learn a feature representation $\Bone$, and $n_2$ samples from a new ($t+1$)st task which are used along with $\Bone$ and \cref{algo:LF_newtask} to learn the parameters $\halpha$ of this new ($t+1$)st task. Then, the parameter $\Bone \halpha$ has the following excess prediction error on a new test point $\xstar$ drawn from the training data covariate distribution:
	\begin{align}
		\mE_{\xstar}[\langle \xstar, \Bone \halpha - \B \balpha_{t+1} \rangle^2] \leq \tlO \left(\frac{dr^2}{n_1} + \frac{r}{n_2} \right),	 \label{eq:main_informal}
	\end{align}
	with high probability over the training data.
	\label{thm:main_informal}
\end{theorem}
The naive complexity of linear regression which ignores the information from the previous $t$ tasks has complexity $O(\frac{d}{n_2})$. \cref{thm:main_informal} suggests that ``positive" transfer from the first $\{1, \hdots, t \}$ tasks to the final ($t+1$)st task can dramatically reduce the sample complexity of learning when $r \ll d$ and $\frac{n_1}{n_2} \gg r^2$; that is, when (1) the complexity of the shared representation is much smaller than the dimension of the underlying space and (2) when the ratio of the number of samples used for feature learning to the number of samples present for a new unseen task exceeds the complexity of the shared representation. We believe that the LTL bound in \cref{thm:main_informal} is the first bound, even in the \textit{linear} setting, to sharply exhibit this phenomenon (see \cref{sec:related_work} for a detailed comparison to existing results). Prior work provides rates for which the leading term in \cref{eq:main_informal} decays as $\sim \frac{1}{\sqrt{t}}$, not as $\sim \frac{1}{n_1}$. We identify structural conditions on the design of the covariates and diversity of the tasks that allow our algorithms to take full advantage of \textit{all} samples available when learning the shared features. Our primary contributions in this paper are to:
\vspace{-.1cm}
\begin{itemize}[leftmargin=.25cm]
    \item Establish that all local minimizers of the (regularized) empirical risk induced by \cref{eq:model_main} are close to the true linear representation up to a small, statistical error. This provides strong evidence that first-order algorithms, such as gradient descent \citep{jin2017escape}, can efficiently recover good feature representations (see \cref{sec:lf_gd}). 
	\item Provide a method-of-moments estimator which can efficiently aggregate information across multiple differing tasks to estimate $\B$---even when it may be information-theoretically impossible to learn the parameters of any given task (see \cref{sec:lf_mom}).
	\item Demonstrate the benefits and pitfalls of transferring learned representations to new, unseen tasks by analyzing the bias-variance trade-offs of the linear regression estimator based on a biased, feature estimate (see \cref{sec:new_task}).
	\item Develop an information-theoretic lower bound for the problem of feature learning, demonstrating that the aforementioned estimator is a close-to-optimal estimator of $\B$, up to logarithmic and conditioning/eigenvalue factors in the matrix of task parameters (see \cref{assump:task}). To our knowledge, this is the first information-theoretic lower bound for representation learning in the multi-task setting  (see \cref{sec:lf_lb}).
\end{itemize}

\subsection{Related Work}
\label{sec:related_work}
While there is a vast literature on papers proposing multi-task and transfer learning methods, the number of theoretical investigations is much smaller. An important early contribution is due to \citet{baxter2000model}, who studied a model where tasks with shared representations are sampled from the same underlying environment. \citet{pontil2013excess} and \citet{maurer2016benefit}, using tools from empirical process theory, developed a generic and powerful framework to prove generalization bounds in multi-task and learning-to-learn settings that are related to ours. Indeed, the closest guarantee to that in our \cref{thm:main_informal} that we are aware of is \citet[Theorem 5]{maurer2016benefit}. Instantiated in our setting, \citet[Theorem 5]{maurer2016benefit} provides an LTL guarantee showing that the excess risk of the loss function with learned representation on a new datapoint is bounded by $\tlO(\frac{r \sqrt{d}}{\sqrt{t}} + \sqrt{\frac{r}{n_2}})$, with high probability.  There are several principal differences between our work and results of this kind. First, we address the algorithmic component (or computational aspect) of meta-learning while the previous theoretical literature generally assumes access to a global empirical risk minimizer (ERM). Computing the ERM in these settings requires solving a \emph{nonconvex} optimization problem that is in general NP hard. Second, in contrast to \citet{maurer2016benefit}, we also provide guarantees for feature recovery in terms of the parameter estimation error---measured directly in the distance in the feature space. 

Third, and most importantly, in \citet{maurer2016benefit}, the leading term capturing the complexity of learning the feature representation decays \textit{only in} $t$ \textit{but not in} $n_1$ (which is typically much larger than $t$). Although, as they remark, the $1/\sqrt{t}$ scaling they obtain is in general unimprovable in their setting, our results leverage assumptions on the distributional similarity between the underlying covariates $\x$ and the potential diversity of tasks to improve this scaling to $1/n_1$. That is, our algorithms make benefit of \textit{all} the samples in the feature learning phase. We believe that for many settings (including the linear model that is our focus) such assumptions are natural and that our rates reflect the practical efficacy of meta-learning techniques. Indeed, transfer learning is often successful even when we are presented with only a few training tasks but with each having a significant number of samples per task (e.g., $n_1 \gg t$).\footnote{See \cref{fig:easy_task_n_train_change} for a numerical simulation relevant to this setting.} 

There has also been a line of recent work providing guarantees for gradient-based meta-learning (MAML) \citep{finn2017model}. \citet{finn2019online, khodak2019provable, khodak2019adaptive}, and \citet{denevi2019learning} work in the framework of online convex optimization (OCO)  and use a notion of (a potentially data-dependent) task similarity that assumes closeness of all tasks to a single fixed point in parameter space to provide guarantees. In contrast to this work, we focus on the setting of  learning a \textit{representation} common to all tasks in a generative model. The task model parameters need not be close together in our setting.

In concurrent work, \citet{du2020few} obtain results similar to ours for multi-task linear regression and provide comparable guarantees for a two-layer ReLU network using a notion of training task diversity akin to ours. Their generalization bound for the two-layer ReLU network uses a distributional assumption over meta-test tasks, but they provide bounds for linear regression holding for both random and fixed meta-test tasks\footnote{In a setting matching \cref{thm:main_informal}, they provide a guarantee of $\tlO \left(dr^2/n_1 + tr^2/n_1 + r/n_2 \right)$ for the ERM when $n_1 \gtrsim dr$ under sub-Gaussian covariate/Gaussian additive noise assumptions. \cref{thm:main_informal} holds for the method-of-moments/linear regression estimator when $n_1 \gtrsim dr^2$ using a Gaussian covariate/sub-Gaussian additive noise assumption; the bound is free of the additional $\tlO(tr^2/n_1)$ term which does not vanish as $t \to \infty$ for fixed $t/n_1$.}. 
They provide purely statistical guarantees---assuming access to an ERM oracle for nonconvex optimization problems. Our focus is on providing sharp statistical rates for efficient algorithmic procedures (i.e., the method-of-moments and local minima reachable by gradient descent). Finally, we also show a (minimax)-lower bound for the problem of feature recovery (i.e., recovering $\B$).



\section{Preliminaries}
Throughout, we will use bold lower-case letters (e.g., $\x$) to refer to vectors and bold upper-case letters to refer to matrices (e.g., $\X$). We exclusively use $\B \in \mR^{d \times r}$ to refer to a matrix with orthonormal columns spanning an $r$-dimensional feature space, and $\B_{\perp}$ to refer a matrix with orthonormal columns spanning the orthogonal subspace of this feature space. The norm $\Vert \cdot \Vert$ appearing on a vector or matrix refers to its $\ell_2$ norm or spectral norm respectively. The notation $\fnorm{\cdot}$ refers to a Frobenius norm. $\langle \x, \y \rangle$ is the Euclidean inner product, while $\langle \M, \N \rangle = \tr(\M \N^\top)$ is the inner product between matrices. Similarly, $\sigma_{\max}(\M)$ and $\sigma_{\min}(\M)$ refer to the maximum and minimum singular values of a matrix $\M$.

Generically, we will use ``hatted" vectors and matrices (e.g., $\halpha$ and $\Bone$) to refer to (random) estimators of their underlying population quantities. We will use $\gtrsim$, $\lesssim$, and $\asymp$ to denote greater than, less than, and equal to up to a universal constant and use $\tlO$ to denote an expression that hides polylogarithmic factors in all problem parameters. Our use of $O$, $\Omega$, and $\Theta$ is otherwise standard.

Formally, an orthonormal feature matrix $\B$ is an element of an equivalence class (under right rotation) of a representative lying in $\Gr$---the Grassmann manifold \citep{edelman1998geometry}. The Grassmann manifold, which we denote as $\Gr$, consists of the set of $r$-dimensional subspaces within an underlying $d$-dimensional space. To define distance in $\Gr$ we define the notion of a principal angle between two subspaces $p$ and $q$. If $\E$ is an orthonormal matrix whose columns form an orthonormal basis of $p$ and $\F$ is an orthonormal matrix whose columns form an orthonormal basis of $q$, then a singular value decomposition of $\E^\top \F = \U \D \V^\top$ defines the principal angles as:
\begin{align*}
	\D = \diag(\cos \theta_1, \cos \theta_2, \hdots, \cos \theta_k),
\end{align*}
where $0 \leq \theta_k \leq \hdots \leq \theta_1 \leq \frac{\pi}{2}$. The distance of interest for us will be the subspace angle distance $\sin \theta_1$, and for convenience we will use the shorthand $\sin \theta(\E, \F)$ to refer to it. With some abuse of notation we will use $\B$ to refer to an explicit orthonormal feature matrix and the subspace in $\Gr$ it represents. We now detail several assumptions we use in our analysis.
\begin{assumption}[Sub-Gaussian Design and Noise]
  \label{assump:design}
The i.i.d. design vectors $\x_i$ are zero mean with covariance $\mE[\x \x^\top]=\I_d$ and are $\I_d$-sub-gaussian, in the sense that $\mE[\exp(\v^\top \x_i)] \leq \exp \left( \frac{\Vert \v \Vert^2}{2} \right)$ for all $\v$. Moreover, the additive noise variables $\epsilon_i$ are i.i.d. sub-gaussian with variance parameter $1$ and are independent of $\x_i$.
\end{assumption}
Throughout, we work in the setting of random design linear regression, and in this context \cref{assump:design} is standard. Our results do not critically rely on the identity covariance assumption although its use simplifies several technical arguments. 
In the following we define the population task diversity matrix as $\A = (\balpha_1, \hdots, \balpha_t)^\top \in \mR^{t \times r}$, $\nu= \sigma_r(\frac{\A^\top \A}{t})$, the average condition number as $\bkappa = \frac{\tr(\frac{\A^\top \A}{t})}{r \nu}$, and the worst-case condition number as $\kappa=\sigma_{1}(\frac{\A^\top \A}{t})/\nu$.
\begin{assumption}[Task Diversity and Normalization]
\label{assump:task}
The $t$ underlying task parameters $\balpha_j$ satisfy $\norm{\balpha_j} = \Theta(1)$ for all $j \in \{1, \hdots, t\}$. Moreover, we assume $\nu > 0$.
\end{assumption}

Recovering the feature matrix $\B$ is impossible without structural conditions on $\A$. Consider the extreme case in which $\balpha_{1}, \hdots, \balpha_{t}$ are restricted to span only the first $r-1$ columns of the column space of the feature matrix $\B$. None of the data points $(\x_i, y_i)$ contain any information about the $r$th column-feature which can be any arbitrary vector in the complementary $d-r-1$ subspace. In this case recovering $\B$ accurately is information-theoretically impossible. The parameters $\nu$, $\bkappa$, and $\kappa$ capture how ``spread out" the tasks $\balpha_j$ are in the column space of $\B$. The condition $\norm{\balpha_j}=\Theta(1)$ is also standard in the statistical literature and is equivalent to normalizing the signal-to-noise (snr) ratio to be $\Theta(1)$\footnote{Note that for a well-conditioned population task diversity matrix where $\bar{\kappa} \leq \kappa \leq O(1)$, our snr normalization enforces that $\tr(\A^\top \A/t)=\Theta(1)$ and $\nu \geq \Omega(\frac{1}{r})$.}. In linear models, the snr is defined as the square of the $\ell_2$ norm of the underlying parameter divided by the variance of the additive noise.

Our overall approach to meta-learning of representations consists of two phases that we term ``meta-train'' and ``meta-test''. First, in the meta-train phase (see \cref{sec:mt_train}), we provide algorithms to learn the underlying linear representation from a set of diverse tasks. Second, in the meta-test phase (see \cref{sec:new_task}) we show how to transfer these learned features to a new, unseen task to improve the sample complexity of learning. Detailed proofs of our main results can be found in the Appendix.

\section{Meta-Train: Learning Linear Features}
\label{sec:mt_train}
Here we address both the algorithmic and statistical challenges of provably learning the linear feature representation $\B$. 

\subsection{Local Minimizers of the Empirical Risk}
\label{sec:lf_gd}
The remarkable, practical success of first-order methods for training nonconvex optimization problems (including meta/multi-task learning objectives) motivates us to study the optimization landscape of the empirical risk induced by the model in \eqref{eq:model_main}. We show in this section that \emph{all local minimizers} of a regularized version of empirical risk recover the true linear representation up to a small statistical error.

Jointly learning the population parameters $\B$ and $(\balpha_1, \hdots, \balpha_t)^\top$ defined by \cref{eq:model_main} is reminiscent of a matrix sensing/completion problem.  We leverage this connection for our analysis, building in particular on results from \citet{ge2017no}. Throughout this section we assume that we are in a uniform task sampling model---at each iteration the task $t(i)$ for the $i$th datapoint is uniformly sampled from the $t$ underlying tasks. We first recast our problem in the language of matrices, by defining the matrix we hope to recover as $\M_{\star} = (\balpha_1, \hdots, \balpha_t)^\top \B^\top \in \mR^{t \times d}$. Since $\text{rank}(\M_\star) = r$, we let $ \X^{\star} \D^{\star} (\Y^{\star})^\top = \text{SVD}(\M_\star)$, and denote $\U^{\star}=\X^{\star} (\D^{\star})^{1/2} \in \mR^{t\times r}$, $\V^{\star} = (\D^{\star})^{1/2} \Y^{\star} \in \mR^{d \times r}$. 
In this notation, the responses of the regression model are written as follows:
\begin{equation}
	y_i = \langle \e_{t(i)} \x_i^\top, \M_{\star} \rangle + \epsilon_i.
\end{equation}
To frame recovery as an optimization problem we consider the Burer-Monteiro factorization of the parameter $\M = \U \V^\top$ where $\U \in \mR^{t \times r}$ and $\V \in \mR^{d\times r}$. This motivates the following objective:
\begin{equation}
    \label{eq:objective_main}
    \min_{\U \in \mR^{t \times r}, \V \in \mR^{d \times r}} f(\U, \V)  =  
	\frac{2t}{n} \sum_{i=1}^{n} (y_i- \langle \e_{t(i)} \x_i^\top , \U \V^\top \rangle)^2 +  \frac{1}{2}\fnorm{\U^\top \U-\V^\top \V}^2.
\end{equation}
The second term in \cref{eq:objective_main} functions as a regularization to prevent solutions which send $\fnorm{\U} \to 0$ while $\fnorm{\V} \to \infty$ or vice versa. If the value of this objective \cref{eq:objective_main} is small we might hope that an estimate of $\B$ can be extracted from the column space of the parameter $\V$, since the column space of $\V^{\star}$ spans the same subspace as $\B$. Informally, our main result states that all local minima of the regularized \textit{empirical} risk are in the neighborhood of the optimal $\V^{\star}$, and have subspaces that approximate $\B$ well. Before stating our result we define the constraint set, which contains incoherent matrices with reasonable scales, as follows:
\begin{equation}
\label{eq:constraint_main}
\cW = \{~ (\U, \V) ~|~ \max_{i\in[t]} \norm{\e_i\trans \U}^2 \le \frac{C_0 \bkappa r \sqrt{\kappa \nu} }{\sqrt{t}},
\quad \norm{\U}^2 \le C_0 \sqrt{t \kappa \nu}, \quad \norm{\V}^2 \le C_0 \sqrt{t \kappa \nu} ~\},
\end{equation}
for some large constant $C_0$. Under \cref{assump:task}, this set contains the optimal parameters. Note that $\U^{\star}$ and $\V^{\star}$ satisfy the final two constraints by definition and \cref{lem:diverse_to_incoh} can be used to show that \cref{assump:task} actually implies that $\U^{\star}$ is incoherent, which satisfies the first constraint. Our main result follows. 
\begin{theorem}\label{thm:main_landscape_convert}
Let \cref{assump:design,,assump:task} hold in the uniform task sampling model. If the number of samples $n_1$ satisfies $n_1 \gtrsim \polylog(n_1,d,t) (\kappa r)^4 \max\{t, d\}$, then, with probability at least $1-1/\poly(d)$, we have that given any local minimum $(\U, \V) \in \text{int}(\cW)$ of the objective \cref{eq:objective_main}, the column space of $\V$, spanned by the orthonormal feature matrix $\Bone$, satisfies:
\begin{equation*}
\sin \theta(\Bone, \B) \leq O \left( \frac{1}{\sqrt{\nu}} \sqrt{\frac{\max\{t, d\} r\log n_1}{n_1}} \right).
\end{equation*}
\end{theorem}
We make several comments on this result:
\begin{itemize}[leftmargin=.5cm]
\item The guarantee in \cref{thm:main_landscape_convert} suggests that all local minimizers of the regularized empirical risk \eqref{eq:objective_main} will produce a linear representation at a distance at most $\tlO(\sqrt{\max\{t, d\}r/n_1})$ from the true underlying representation. \cref{thm:feature_lb_yb} guarantees that any estimator (including the empirical risk minimizer) must incur error $\gtrsim \sqrt{d r/n_1}$. Therefore, in the regime $t \leq O(d)$, all local minimizers are statistically close-to-optimal, up to logarithmic factors and conditioning/eigenvalue factors in the task diversity matrix.
	\item Combined with a recent line of results showing that (noisy) gradient descent can efficiently escape strict saddle points to find local minima \citep{jin2017escape}, \cref{thm:main_landscape_convert} provides strong evidence that first-order methods can efficiently meta-learn linear features.\footnote{To formally establish computational efficiency, we need to further verify the smoothness and the strict-saddle properties of the objective function \eqref{eq:objective_main} (see, e.g., \cite{jin2017escape}).}
\end{itemize}
The proof of \cref{thm:main_landscape_convert} is technical so we only sketch the high-level ideas. The overall strategy is to analyze the Hessian of the objective \cref{eq:objective_main} at a stationary point $(\U, \V)$ in $\text{int}(\cW)$ to exhibit a direction $\Delta$ of negative curvature which can serve as a direction of local improvement pointing towards $\M^{\star}$ (and hence show $(\U, \V)$ is not a local minimum). Implementing this idea requires surmounting several technical hurdles including (1) establishing various concentration of measure results (e.g., RIP-like conditions) for the sensing matrices $ \e_{t(i)} \x_i^\top$ unique to our setting and (2) handling the interaction of the optimization analysis with the regularizer and noise terms. Performing this analysis establishes that under the aforementioned conditions all local minima in $\text{int}(\cW)$ satisfy $
	\norm{\U\V^\top - \M^\star}_F\leq O(\sqrt{t\frac{\max\{t, d\} r\log n_1}{n_1}})
$
(see \cref{thm:main_landscape}).
Guaranteeing that this loss is small is not sufficient to ensure recovery of the underlying features. Transferring this guarantee in the Frobenius norm to a result on the subspace angle critically uses the task diversity assumption (see \cref{lem:frob_to_angle}) to give the final result.

\subsection{Method-of-Moments Estimator}
\label{sec:lf_mom}
\begin{algorithm}[!bt]
\caption{MoM Estimator for Learning Linear Features}\label{algo:LF_learn}
\begin{algorithmic}
\renewcommand{\algorithmicrequire}{\textbf{Input: }}
\renewcommand{\algorithmicensure}{\textbf{Output: }}
\REQUIRE $\{ (\x_i, y_i) \}_{i=1}^{n_1}$.
\STATE $\Bone \D_1 \Bone \trans \leftarrow$ top-$r$ SVD of $\frac{1}{n_1} \cdot \sum_{i=1}^{n_1} y_i^2 \x_i\x_i\trans$
\STATE \textbf{return} $\Bone$
\end{algorithmic}
\end{algorithm}
Next, we present a method-of-moments algorithm to recover the feature matrix $\B$ with sharper statistical guarantees. An alternative to optimization-based approaches such as maximum likelihood estimation, the method-of-moments is among the oldest statistical techniques~\citep{pearson1894contributions} and has recently been used to estimate parameters in latent variable models~\citep{anandkumar2012method}.

As we will see, the technique is well-suited to our formulation of multi-task feature learning. We present our estimator in \cref{algo:LF_learn}, which simply computes the top-$r$ eigenvectors of the matrix $(1/n_1)  \sum_{i=1}^{n_1} y_i^2 \x_i\x_i\trans$. Before presenting our result, we define the averaged empirical task matrix as $\barLambda = \sumton \balpha_{t(i)} \balpha_{t(i)}^\top$ where $\tnu = \sigma_{r}(\barLambda)$, and $\tkappa = \tr(\barLambda)/(r \tnu)$ in analogy with \cref{assump:task}.
\begin{theorem}
Suppose the $n_1$ data samples $\{ (\x_i, y_i) \}_{i=1}^{n_1}$ are generated from the model in \eqref{eq:model_main} and that \cref{assump:design,,assump:task} hold, but additionally that $\x_i \sim \cN(0, \I_d)$. Then, if $n_1 \gtrsim \polylog(d, n_1) rd   \tkappa /\tnu$, the output $\Bone$ of \cref{algo:LF_learn} satisfies
\begin{align*}
	\sin \theta(\Bone, \B) \leq \tlO \left(\sqrt{\frac{\tkappa}{\tnu} \frac{d  r}{n_1}} \right),
\end{align*}
with probability at least $1-O(n_1^{-100})$. Moreover, if the number of samples generated from each task are equal (i.e., $\barLambda = \frac{1}{t} \A^\top \A$), then the aforementioned guarantee holds with $\tkappa=\bkappa$ and $\tnu=\nu$.
\label{thm:B_recovery}
\end{theorem}
We first make several remarks regarding this result.
\begin{itemize}[leftmargin=.5cm]
	\item \cref{thm:B_recovery} is flexible---the only dependence of the estimator on the distribution of samples across the various tasks is factored into the \textit{empirical} task diversity parameters $\tnu$ and $\tkappa$. Under a uniform observation model the guarantee also immediately translates into an analogous statement which holds with the population task diversity parameters $\nu$ and $\bkappa$.
	\item \cref{thm:B_recovery} provides a non-trivial guarantee even in the setting where we only have $\Theta(1)$ samples from each task, but $t=\tlTheta(dr)$. In this setting, recovering the parameters of any given task is information-theoretically impossible. However, the method-of-moments estimator can efficiently aggregate information \textit{across} the tasks to learn $\B$. 
	\item The estimator does rely on the moment structure implicit in the Gaussian design to extract $\B$. However, \cref{thm:B_recovery} has no explicit dependence on $t$ and is close-to-optimal in the constant-snr regime; see \cref{thm:feature_lb_yb} for our lower bound.
\end{itemize}
We now provide a summary of the proof. Under oracle access to the population mean $\mE[\frac{1}{n} \sum_i y_i^2 \x_i \x_i^\top] = (2\barGamma + (1+\tr(\barGamma)) \I_d)$, where $\barGamma = \sumton \B \balpha_{t(i)} \balpha_{t(i)}^\top \B^\top$ (see \cref{lem:estimator_mean}), we can extract the features $\B$ by directly applying PCA to this matrix, under the condition that $\tkappa > 0$, to extract its column space. In practice, we only have access to the samples $\{(\x_i, y_i) \}_{i=1}^{n}$.  \cref{algo:LF_learn} uses the empirical moments $\frac{1}{n} \sum_i y_i^2 \x_i \x_i^\top$ in lieu of the population mean. Thus, to show the result, we argue that $\sumton y_i^2 \x_i\x_i\trans = \mE[\sumton y_i^2 \x_i\x_i\trans] + \E$ where $\norm{\E}$ is a small, stochastic error (see \cref{thm:estimator_conc}). If this holds, the Davis-Kahan $\sin \theta$ theorem \citep{bhatia2013matrix} shows that PCA applied to the empirical moments provides an accurate estimate of $\B$ under perturbation by a sufficiently small $\E$.

 The key technical step in this argument is to show sharp concentration (in spectral norm) of the matrix-valued noise terms, $\E_1$, $\E_2$, and $\E_3$:
\begin{align*}
	& \sumton y_i^2 \x_i\x_i\trans - (2\barGamma + (1+\tr(\barGamma)) \I_d = \E_1+\E_2+\E_3,
\end{align*}
which separate the empirical moment from its population mean (see \cref{lem:first_noise,,lem:second_noise,,lem:third_noise}). The exact forms are deferred to the \cref{app:mom}, but as an example we have that $\E_2 = \sumton 2 \epsilon_i \x_i^\top \B \balpha_{t(i)} \x_i \x_i^\top$. The noise terms $\E_1, \E_2$, and $\E_3$ are neither bounded (in spectral norm) nor are they sub-gaussian/sub-exponential-like, since they contain fourth-order terms in the $\x_i$ and $\epsilon_i$. The important tool we use to show concentration of measure for these objects is a truncation argument along with the matrix Bernstein inequality (see \cref{lem:trunc_mb}).


\section{Meta-Test: Transferring Features to New Tasks}
\label{sec:new_task}
\begin{algorithm}[!t]
\caption{Linear Regression for Learning a New Task with a Feature Estimate}\label{algo:LF_newtask}
\begin{algorithmic}
\renewcommand{\algorithmicrequire}{\textbf{Input:}}
\renewcommand{\algorithmicensure}{\textbf{Output:}}
\REQUIRE $\Bone, \{ (\x_i, y_i) \}_{i=1}^{n_2}$.
\STATE $\halpha \leftarrow (\sum_{i=1}^{n_2} \Bone \x_i \x_i^\top \Bone^\top)^{\dagger} \Bone^\top \sum_{i=1}^{n_2} \x_i y_i$
\STATE \textbf{return} $\halpha$
\end{algorithmic}
\end{algorithm}
Having estimated a linear feature representation $\Bone$ shared across related tasks, our second goal is to transfer this representation to a new, unseen task---the ($t+1$)st task---to improve learning. In the context of the model in \cref{eq:model_main}, the approach taken in \cref{algo:LF_newtask} uses $\Bone$ as a plug-in surrogate for the unknown $\B$, and attempts to estimate $\balpha_{t+1} \in \mR^{r}$.  Formally we define our estimator $\balpha$ as follows:
\begin{align}
	\halpha = \arg \min_{\balpha} \norm{\y - \X \Bone \balpha}^2,	 \label{eq:lr_red}
\end{align}
where $n_2$ samples $(\X, \y)$ are generated from the model in \cref{eq:model_main} from the ($t+1$)st task. Effectively, the feature representation $\Bone$ performs dimension reduction on the input covariates $\X$, allowing us to learn in a lower-dimensional space. Our focus is on understanding the generalization properties of the estimator in \cref{algo:LF_newtask}, since \eqref{eq:lr_red} is an ordinary least-squares objective which can be analytically solved.

 Assuming we have produced an estimate $\Bone$ of the true feature matrix $\B$, we can present our main result on the sample complexity of meta-learned linear regression. 
\begin{theorem}
	Suppose $n_2$ data points, $\{ (\x_i, y_i) \}_{i=1}^{n_2}$, are generated from the model in \eqref{eq:model_main}, where \cref{assump:design} holds, from a single task satisfying $\norm{\balpha_{t+1}}^2 \leq O(1)$. Then, if $\sin \theta(\Bone, \B) \leq \delta$ and $n_2 \gtrsim  r \log n_2$, the output $\halpha$ from \cref{algo:LF_newtask} satisfies
	\begin{align}
		\norm{\Bone \halpha-\B \balpha_{t+1}}^2 \leq \tlO \left( \delta^2 + \frac{r}{n_2} \right), \label{eq:transfer_guarantee}
	\end{align}
	with probability at least $1-O(n_2^{-100})$.
	\label{thm:lr_transfer}
\end{theorem}
Note that $\B \balpha_{t+1}$ is simply the underlying parameter in the regression model in \cref{eq:model_main}. We make several remarks about this result:
\begin{itemize}[leftmargin=.5cm]
	\item \cref{thm:lr_transfer} decomposes the error of transfer learning into two components.  The first term, $\tlO(\delta^2)$, arises from the bias of using an imperfect feature estimate $\Bone$ to transfer knowledge across tasks. The second term, $\tlO(\frac{r}{n_2})$, arises from the variance of learning in a space of reduced dimensionality.
	
	\item Standard generalization guarantees for random design linear regression ensure that the parameter recovery error is bounded by $O(\frac{d}{n_2})$ w.h.p.\ under the same assumptions \citep{hsu2012random}. Meta-learning of the linear representation $\Bone$ can provide a significant reduction in the sample complexity of learning when $\delta^2 \ll \frac{d}{n_2}$ and $r \ll d$. 
	\item Conversely, if $\delta^2 \gg \frac{d}{n_2}$ the bounds in \cref{eq:transfer_guarantee} imply that the overhead of learning the feature representation may overwhelm the potential benefits of transfer learning (with respect to baseline of learning the ($t+1$)st task in isolation). This accords with the well-documented empirical phenomena of ``negative" transfer observed in large-scale deep learning problems where meta/transfer-learning techniques actually result in a degradation in performance on new tasks \citep{wang2019characterizing}. For diverse tasks (i.e. $\kappa \leq O(1)$), using \cref{algo:LF_learn} to estimate $\Bone$ suggests that ensuring $\delta^2 \ll \frac{d}{n_2}$, where $\delta^2 = \tlO(\frac{dr}{\nu n_1})$, requires $\frac{n_1}{n_2} \gg r/\nu$. That is, the ratio of the number of samples used for feature learning to the number of samples used for learning the new task should exceed the complexity of the feature representation to achieve ``positive" transfer.
\end{itemize}
In order to obtain the rate in \cref{thm:lr_transfer} we use a bias-variance analysis of the estimator error $\Bone \halpha-\B \balpha_{t+1}$ (and do not appeal to uniform convergence arguments). Using the definition of $\y$ we can write the error as,
\begin{align*}
	& \Bone\halpha-\B \balpha_0 = \Bone (\Bone^\top \X^\top \X \Bone)^{-1} \Bone \X^\top \X \B \balpha_0 - \B \balpha_0 +  \Bone(\Bone^\top \X^\top \X \Bone)^{-1} \Bone^\top \X^\top \bepsilon.
\end{align*}
The first term contributes the bias term to \cref{eq:transfer_guarantee} while the second contributes the variance term. 
Analyzing the fluctuations of the (mean-zero) variance term can be done by controlling the norm of its square, $\bepsilon^\top \A \bepsilon$, where $\A = \X \Bone (\Bone^\top \X^\top \X \Bone)^{-2} \Bone^\top \X^\top$. We can bound this (random) quadratic form by first appealing to the Hanson-Wright inequality to show w.h.p. that $\bepsilon^\top \A \bepsilon \lesssim \tr(\A) + \tlO(\fnorm{\A} + \norm{\A})$. The remaining randomness in $\A$ can be controlled using matrix concentration/perturbation arguments (see \cref{lem:test_bias}).

With access to the true feature matrix $\Bone$ (i.e., setting $\Bone=\B$) the term $\Bone (\B^\top \X^\top \X \B)^{-1} \B \X^\top \X \B \balpha_0 - \B \balpha_0 = 0$, due to the cancellation in the empirical covariance matrices, $(\B^\top \X^\top \X \B)^{-1} \B \X^\top \X \B = \I_r$. This cancellation of the empirical covariance is essential to obtaining a tight analysis of the least-squares estimator. We cannot rely on this effect in full since $\Bone \neq \B$. However, a naive analysis which splits these terms, $(\Bone^\top \X^\top \X \Bone)^{-1}$ and $\Bone \X^\top \X \B $ can lead to a large increase in the variance in the bound. To exploit the fact $\Bone \approx \B$, we project the matrix $\B$ in the leading $\X \B$ term onto the column space of $\Bone$ and its complement---which allows a partial cancellation of the empirical covariances in the subspace spanned by $\Bone$. The remaining variance can be controlled as in the previous term (see \cref{lem:test_var}).
\section{Lower Bounds for Feature Learning}
\label{sec:lf_lb}
To complement the upper bounds provided in the previous section, in this section we derive information-theoretic limits for feature learning in the model \cref{eq:model_main}. To our knowledge, these results provide the first sample-complexity lower bounds for feature learning, with regards to subspace recovery, in the multi-task setting.  
While there is existing literature on (minimax)-optimal estimation of low-rank matrices (see, for example, \citet{rohde2011estimation}), that work focuses on the (high-dimensional) estimation of matrices, whose only constraint is to be low rank. Moreover, error is measured in the additive prediction norm. In our setting, we must handle the additional difficulties arising from the fact that we are interested in (1) learning a column space (i.e., an element in the $\Gr$) and (2) the error between such representatives is measured in the subspace angle distance. We begin by presenting our lower bound for feature recovery.
\begin{theorem}
	Suppose a total of $n$ data points are generated from the model in \eqref{eq:model_main} satisfying \cref{assump:design} with $\x_i \sim \cN(0, \I_d)$, $\epsilon_i \sim \cN(0, 1)$, with an equal number from each task, and that \cref{assump:task} holds with $\balpha_j$ for each task normalized to $\norm{\balpha_j}=\frac{1}{2}$. Then, there are $\balpha_j$ for $r \leq \frac{d}{2}$ and $n \geq \max \left(\frac{1}{8\nu}, r(d-r) \right)$ so that:
	\begin{align*}
		\inf_{\Bone} \sup_{\B \in \Gr} \sin \theta(\Bone, \B) \geq \Omega \left( \max \left(\sqrt{\frac{1}{\nu}} \sqrt{\frac{1}{n}}, \sqrt{\frac{dr}{n}} \right) \right),
	\end{align*} 
 with probability at least $\frac{1}{4}$, where the infimum is taken over the class of estimators that are functions of the $n$ data points.
 \label{thm:feature_lb_yb}
\end{theorem}
Again we make several comments on the result.
\begin{itemize}[leftmargin=.25cm]
	\item The result of \cref{thm:feature_lb_yb} shows that the estimator in \cref{algo:LF_learn} provides a close-to-optimal estimator of the feature representation parameterized by $\B$--up to logarithmic and conditioning factors (i.e. $\kappa, \nu$)\footnote{Note in the setting that $\kappa \leq O(1)$, $\nu \sim  \frac{1}{r}$.} in the task diversity matrix--that is independent of the task number $t$. Note that under the normalization for $\balpha_i$, as $\kappa \to \infty$ (i.e. the task matrix $\A$ becomes ill-conditioned) we have that $\nu \to 0$. So the first term in \cref{thm:feature_lb_yb} establishes that task diversity is necessary for recovery of the subspace $\B$.
	\item The dimension of  $\Gr$ (i.e., the number of free parameters needed to specify a feature set) is $r(d-r) \geq \Omega(dr)$ for $d/2 \geq r$; hence the second term in \cref{thm:feature_lb_yb} matches the scaling that we intuit from parameter counting.
	\item  Obtaining tight dependence of our subspace recovery bounds on conditioning factors in the task diversity matrix (i.e. $\kappa, \nu$) is an important and challenging research question. We believe the gap between in conditioning/eigenvalue factors between \cref{thm:B_recovery} and \cref{thm:feature_lb_yb} on the $\sqrt{dr/n}$ term is related to a  problem that persists for classical estimators in linear regression (i.e. for 
	the Lasso estimator in sparse linear regression). Even in this setting, a gap remains with respect to condition number/eigenvalue factors of the data design matrix $\X$, between existing upper and lower bounds (see \citet[Section 7]{chen2016bayes}, \citet[Theorem 1, Theorem 2]{raskutti2011minimax} and \citet{zhang2014lower} for example). In our setting the task diversity matrix $\A$ enters into the problem in a similar fashion to the data design matrix $\X$ in these aforementioned settings.
\end{itemize}
The dependency on the task diversity parameter $\frac{1}{\nu}$ (the first term in \cref{thm:feature_lb_yb}) is achieved by constructing a pair of feature matrices and an ill-conditioned task matrix $\A$ that cannot discern the direction along which they defer. The proof strategy to capture the second term uses a $f$-divergence based minimax technique from \citet{guntuboyina2011lower}, similar in spirit to the global Fano (or Yang-Barron) method.  
\begin{lemma}\citep[Theorem 4.1]{guntuboyina2011lower}
For any increasing function $\ell : [0, \infty) \to [0, \infty)$,
\begin{align}
\inf_{\hat{\theta}} \sup_{\theta \in \Theta} \Pr_{\theta}[\ell(\rho(\hat{\theta}, \theta)) \geq  \ell(\eta/2)] \geq \sup_{\eta > 0, \epsilon > 0} \left \{ 1- \left(\frac{1}{N(\eta)} + \sqrt{\frac{(1+\epsilon^2) M_C(\epsilon, \Theta)}{N(\eta)}} \right) \right \}. \nonumber
\end{align}
\label{thm:yang_barron}
\end{lemma}
In the context of the previous result $N(\eta)$ denotes a lower bound on the $\eta$-packing number of the metric space $(\Theta, \rho)$. Moreover, $M_C(\epsilon, \Theta)$ is a positive real number for which there exists a set $G$ with cardinality $\leq M_C(\epsilon, S)$ and probability measures $Q_{\alpha}$, $\alpha \in G$ such that $\sup_{\theta \in S} \min_{\alpha \in G} \chi^2(\Pr_{\theta}, Q_{\alpha}) \leq \epsilon^2$, where $\chi^2$ denotes the chi-squared divergence. In words, $M_{C}(\epsilon, S)$ is an upper bound on the $\epsilon$-covering on the space $\{ \Pr_{\theta} : \theta \in S \}$ when distances are measured by the square root of the $\chi^2$-divergence.

There are two key ingredients to using \cref{thm:yang_barron} and obtaining a tight lower bound. First, we must exhibit a large family of distinct, well-separated feature matrices $\{ \B_i \}_{i=1}^M$  (i.e., a packing at scale $\eta$). Second, we must argue this set of feature matrices induces a family of distributions over data $\{ (\x_i, y_i) \}_{B_i}$ which are statistically ``close" and fundamentally difficult to distinguish amongst. This is captured by the fact the $\epsilon$-covering number, measured in the space of \textit{distributions} with divergence measure $D_{f}(\cdot, \cdot)$, is small. The standard (global) Fano method, or Yang-Barron method (see \citet[Ch. 15]{wainwright2019high}), which uses the $\KL$ divergence to measure distance in the space of measures, is known to provide rate-suboptimal lower bounds for parametric estimation problems.\footnote{Even for the simple problems of Gaussian mean estimation the classical Yang-Barron method is suboptimal; see \citet{guntuboyina2011lower} for more details.} Our case is no exception. To circumvent this difficulty we use the framework of \citet{guntuboyina2011lower}, instantiated with the $f$-divergence chosen as the $\chi^2$-divergence, to obtain a tight lower bound.
 
Although the geometry of $\Gr$ is complex, we can adapt results from \citet{pajor1998metric} to provide sharp upper/lower bounds on the metric entropy (or global entropy) of the Grassmann manifold (see \cref{prop:pajor}). At scale $\delta$, we find that the global covering/packing numbers of $\Gr$ in the subspace angle distance scale as $\asymp r(d-r) \log(\frac{1}{\delta})$. This result immediately establishes a lower bound on the packing number in $\Gr$ at scale $\eta$.

The second technical step of the argument hinges on the ability to cover the space of distributions parametrized by $\B$ in the space of measures $\{ \Pr_{\B} : \B \in \Gr \}$---with distance measured by an appropriate $f$-divergence. In order to establish a covering in the space of measures parametrized by $\B$, the key step is to bound the distance $\chi^2 (\Pr_{\B^{1}}, \Pr_{\B^{2}})$ for two different measures over data generated from the model \eqref{eq:model_main} with two different feature matrices $\B^1$ and $\B^2$ (see \cref{lem:mut_info_yb}). This control can be achieved in our random design setting by exploiting the Gaussianity of the marginals over data $\X$ and the Gaussianity of the conditionals of $\y | \X, \B$, to ultimately be expressed as a function of the angular distance between $\B^1$ and $\B^2$. From this, \cref{prop:pajor} also furnishes a bound on the covering number of $\{ \Pr_{\B} : \B \in \Gr \}$ at scale $\epsilon$. Combining these two results with \cref{thm:yang_barron} and tuning the scales $\eta$ and $\epsilon$ appropriately gives the final theorem.

\section{Simulations}
We complement our theoretical analysis with a series of numerical experiments highlighting the benefits (and limits) of meta-learning.\footnote{An open-source Python implementation to reproduce our experiments can be found at \url{https://github.com/nileshtrip/MTL}.}
For the purposes of feature learning we compare the performance of the method-of-moments estimator in \cref{algo:LF_learn} vs. directly optimizing the objective in \cref{eq:objective_main}. Additional details on our set-up are provided in \cref{sec:app_expt}.  
We construct problem instances by generating Gaussian covariates and noise as $\x_i \sim \cN(0, \I_d)$, $\epsilon_i \sim \cN(0, 1)$, and the tasks and features used for the first-stage feature estimation as $\balpha_i \sim \frac{1}{\sqrt{r}} \cdot \cN(0, \I_r)$, with $\B$ generated as a (uniform) random $r$-dimensional subspace of $\mR^d$. In all our experiments we generate an equal number of samples $n_t$ for each of the $t$ tasks, so $n_1 = t \cdot n_t$. In the second stage we generate a new, ($t+1$)st task instance using the same feature estimate $\B$ used in the first stage and otherwise generate $n_2$ samples, with the covariates, noise and $\balpha_{t+1}$ constructed as before. Throughout this section we refer to features learned via a first-order gradient method as LF-FO and the corresponding meta-learned regression parameter on a new task by meta-LR-FO. We use LF-MoM and meta-LR-MoM to refer to the same quantities save with the feature estimate learned via the method-of-moments estimator. We also use LR to refer to the baseline linear regression estimator on a new task which only uses data generated from that task.  

We begin by considering a challenging setting for feature learning where $d=100$, $r=5$, but $n_t=5$ for varying numbers of tasks $t$. 
\begin{figure}[!hbt]
\centering
\begin{minipage}[c]{.4\linewidth}
\includegraphics[width=\linewidth]{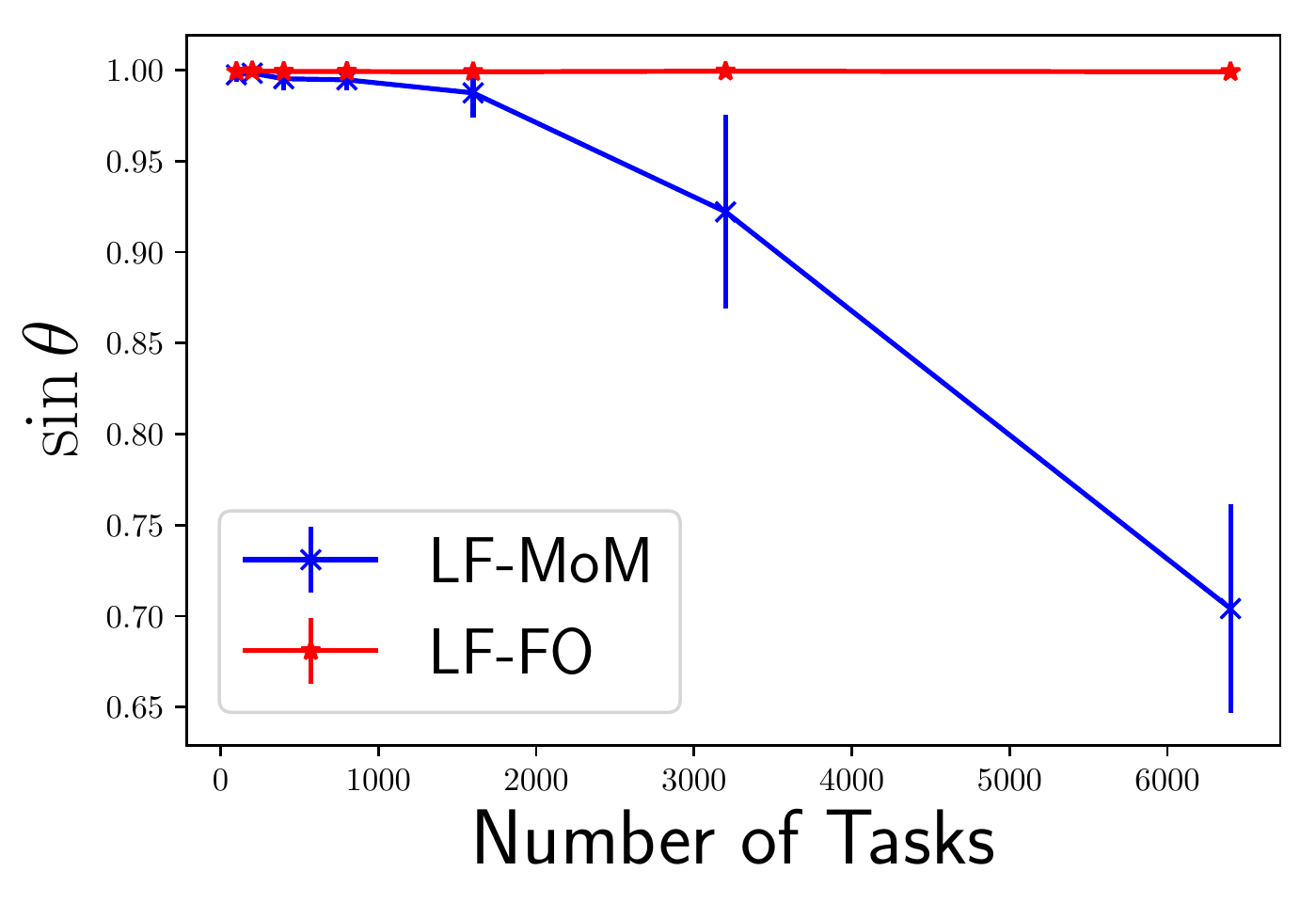}
\end{minipage}
\begin{minipage}[c]{.4\linewidth}
\includegraphics[width=\linewidth]{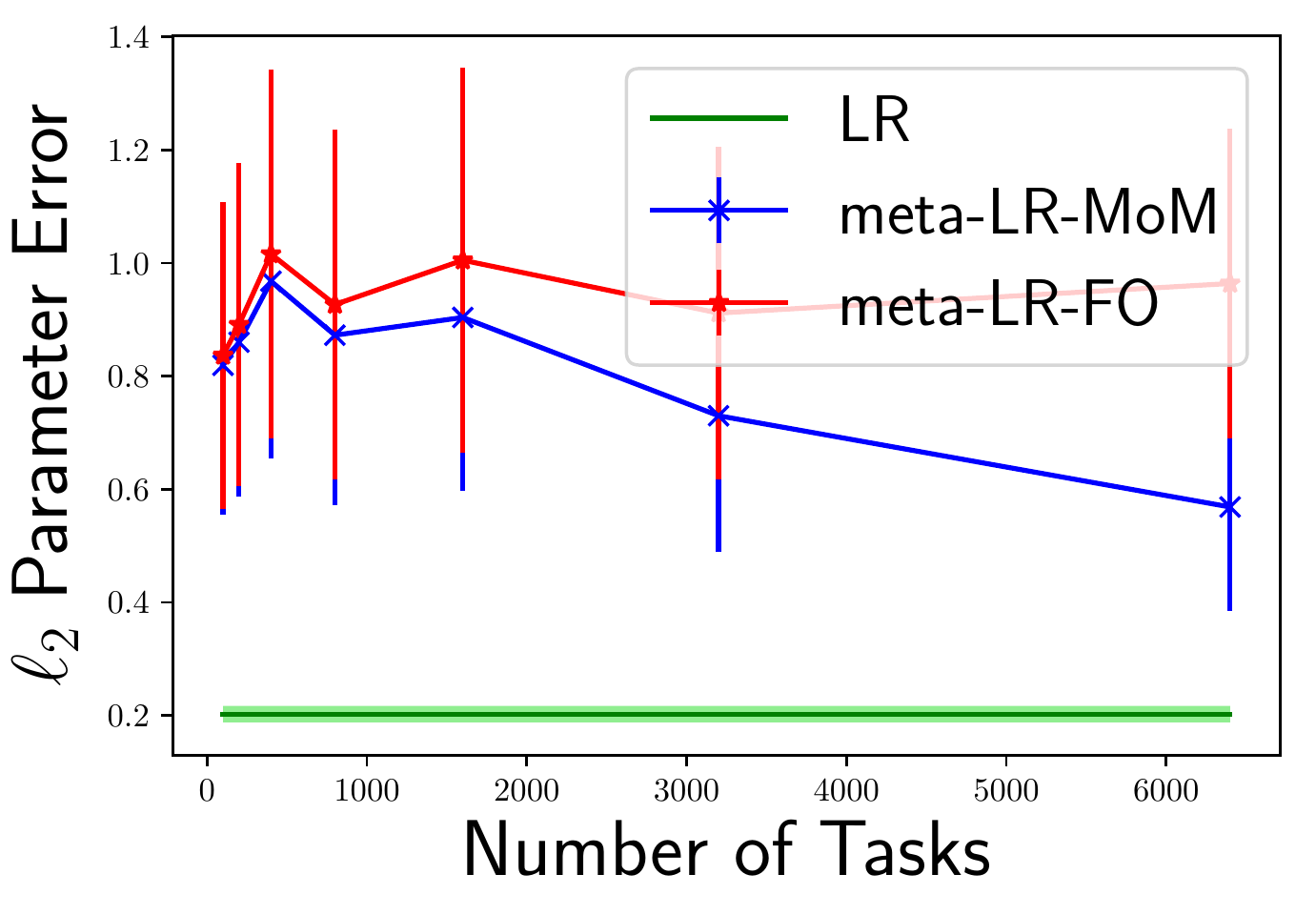}
\end{minipage}
\caption{Left: LF-FO vs. LF-MoM estimator with error measured in the subspace angle distance $\sin \theta(\Bone, \B)$. Right: meta-LR-FO and meta-LR-MoM vs. LR on new task with error measured on new task parameter. Here $d=100$, $r=5$, and $n_t = 5$ while $n_2=2500$ as the number of tasks is varied.
}
\label{fig:hard_task_change}
\end{figure} As \cref{fig:hard_task_change} demonstrates, the method-of-moments estimator is able to aggregate information across the tasks as $t$ increases to slowly improve its feature estimate, even though $n_t \ll  d$. The loss-based approach struggles to improve its estimate of the feature matrix $\B$ in this regime. This accords with the extra $t$ dependence in \cref{thm:main_landscape_convert} relative to \cref{thm:B_recovery}. In this setting, we also generated a ($t+1$)st test task with $d \ll n_2 = 2500$, to test the effect of meta-learning the linear representation on generalization in a new, unseen task against a baseline which simply performs a regression on this new task in isolation. \cref{fig:hard_task_change} also shows that meta-learned regressions perform significantly worse than simply ignoring first $t$ tasks. \cref{thm:lr_transfer} indicates the bias from the inability to learn an accurate feature estimate of $\B$ overwhelms the benefits of transfer learning. 
In this regime $n_2 \gg d$ so the new task can be efficiently learned in isolation.  
We believe this simulation represents a simple instance of the empirically observed phenomena of ``negative" transfer \citep{wang2019characterizing}.

We now turn to the more interesting use cases where meta-learning is a powerful tool. We consider a setting where $d=100$, $r=5$, and $n_t=25$ for varying numbers of tasks $t$. However, now we consider a new, unseen task where data is scarce: $n_2=25 < d$.
\begin{figure}[!hbt]
\centering
\begin{minipage}[c]{.4\linewidth}
\includegraphics[width=\linewidth]{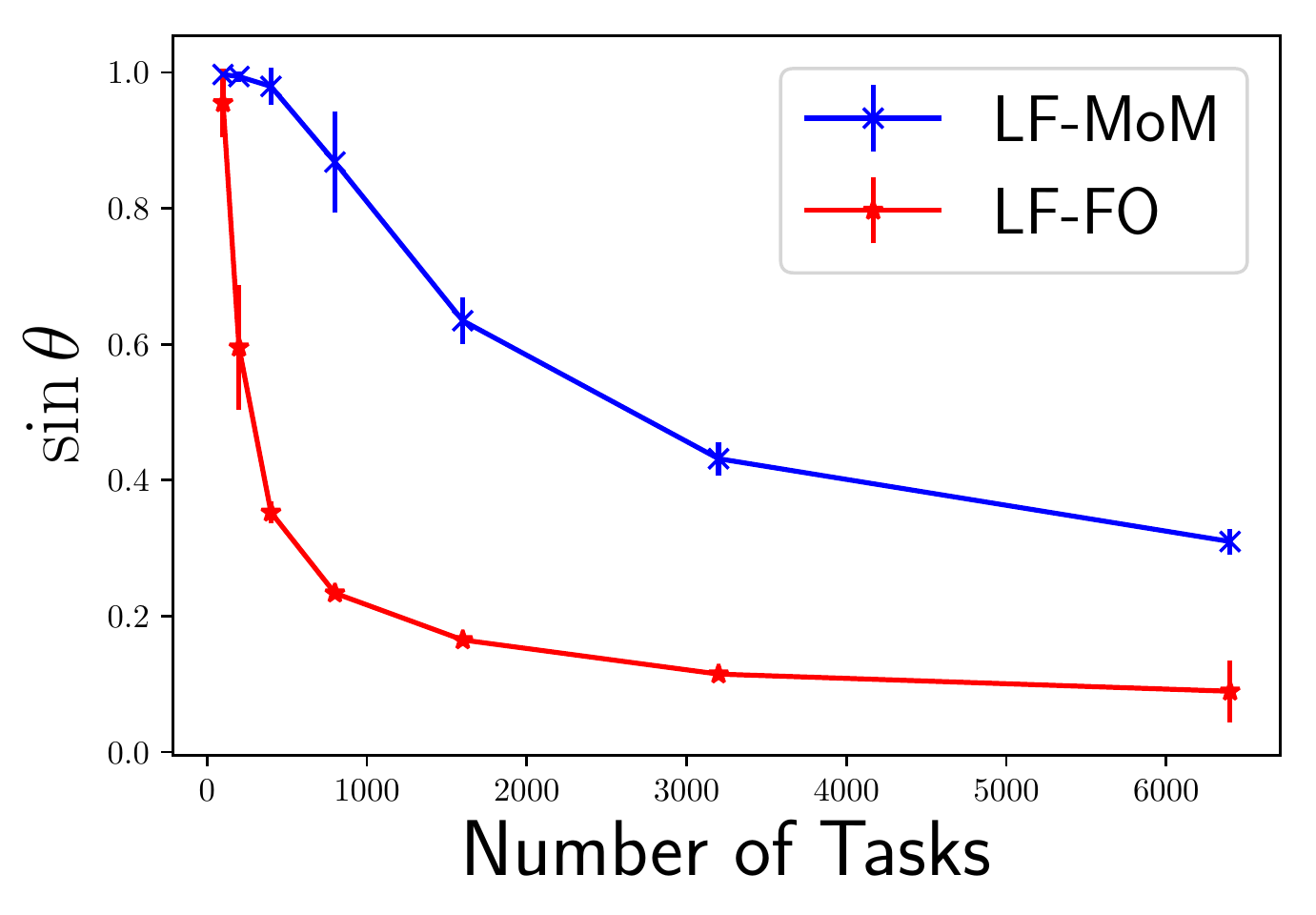}
\end{minipage}
\begin{minipage}[c]{.4\linewidth}
\includegraphics[width=\linewidth]{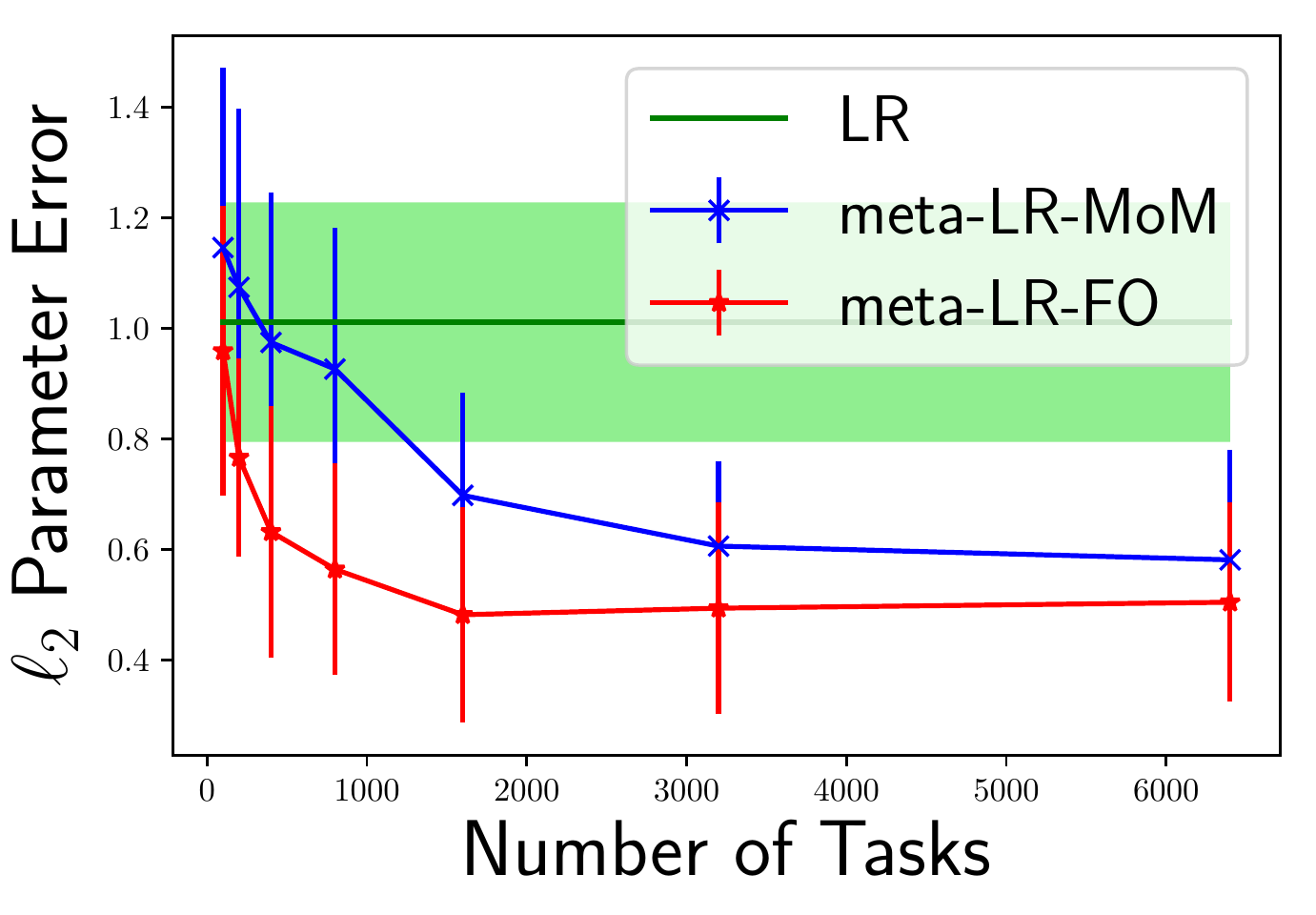}
\end{minipage}
\caption{Left: LF-FO vs. LF-MoM estimator with error measured in the subspace angle distance $\sin \theta(\Bone, \B)$. Right: meta-LR-FO and meta-LR-MoM vs. LR on new task with error measured on new task parameter. Here $d=100$, $r=5$, $n_t = 25$ while $n_2=25$ while the number of tasks is varied.
}
\label{fig:easy_task_t_change}
\end{figure}
As \cref{fig:easy_task_t_change} shows, in this regime both the method-of-moments estimator and the loss-based approach can learn a non-trivial estimate of the feature representation. The benefits of transferring this representation are also evident in the improved generalization performance seen by the meta-regression procedures on the new task. Interestingly, the loss-based approach learns an accurate feature representation $\Bone$ with significantly fewer samples then the method-of-moments estimator, in contrast to the previous experiment.

Similarly, if we consider an instance where $d=100$, $r=5$, $t=20$, and $n_2=50$ with varying numbers of training points $n_t$ per task, we see in \cref{fig:easy_task_n_train_change} that meta-learning of representations provides significant value in a new task. Note that these numerical experiments show that as the number of tasks is fixed, but $n_t$ increases, the generalization ability of the meta-learned regressions significantly improves as reflected in the bound \cref{eq:main_informal}. As in the previous experiment, the loss-based approach is more sample-efficient then the method-of-moments estimator. 
\begin{figure}[!hbt]
\centering
\begin{minipage}[c]{.4\linewidth}
\includegraphics[width=\linewidth]{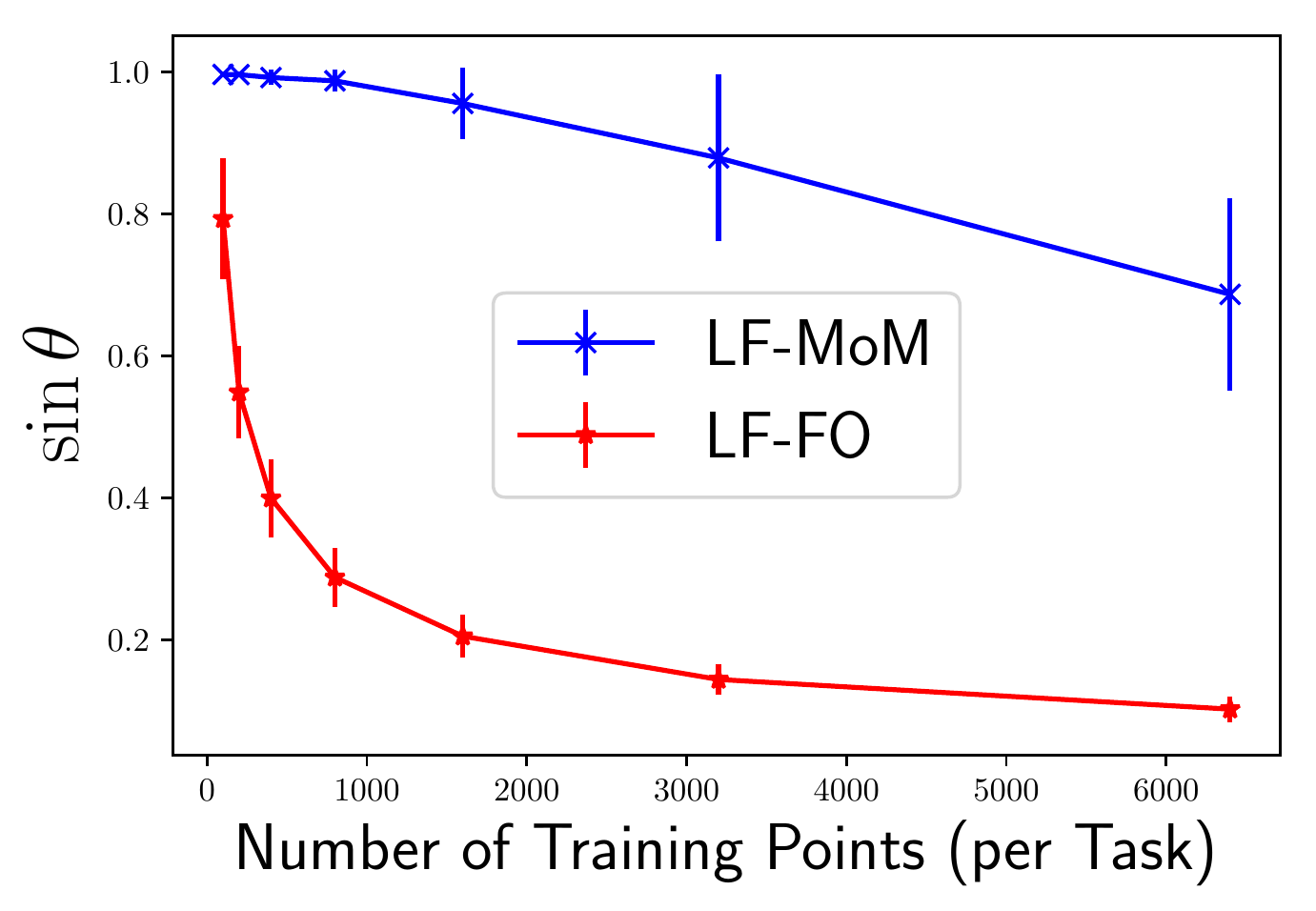}
\end{minipage}
\begin{minipage}[c]{.4\linewidth}
\includegraphics[width=\linewidth]{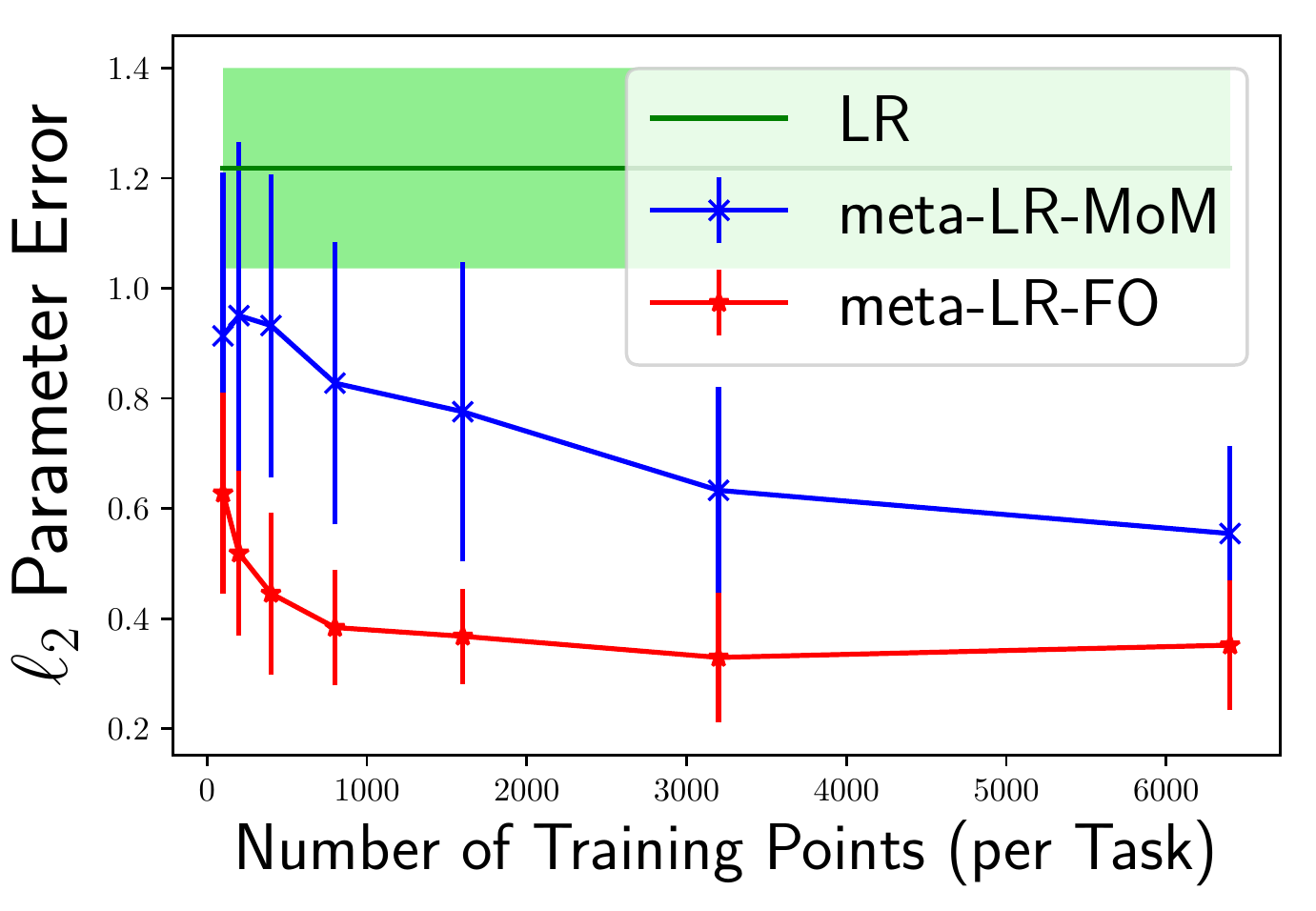}
\end{minipage}
\caption{Left: LF-FO vs. LF-MoM estimator with error measured in the subspace angle distance $\sin \theta(\Bone, \B)$. Right: meta-LR-FO and meta-LR-MoM vs. LR on new task with error measured on new task parameter. Here $d=100$, $r=5$, $t=20$, and $n_2=50$ while the number of training points per task ($n_t$) is varied.
}
\label{fig:easy_task_n_train_change}
\end{figure}

\section{Conclusions}
In this paper we show how a shared linear representation may be efficiently learned and transferred between multiple linear regression tasks. We provide both upper and lower bounds on the sample complexity of learning this representation and for the problem of learning-to-learn. We believe our bounds capture important qualitative phenomena observed in real meta-learning applications absent from previous theoretical treatments.

\section{Acknowledgements}
The authors thank Niladri Chatterji for helpful discussions.  This work was supported by the Army Research Office (ARO) under contract W911NF-17-1-0304 as
part of the collaboration between US DOD, UK MOD and UK Engineering and Physical
Research Council (EPSRC) under the Multidisciplinary University Research Initiative (MURI).

\bibliography{ref}
\bibliographystyle{plainnat}

\onecolumn

\appendix
\begin{center}{\LARGE \bf Appendices}\end{center}\vskip12pt
\textbf{Notation and Set-up}
We first establish several useful pieces of notation used throughout the Appendices. We will say that a mean-zero random variable $x$ is sub-gaussian, $x \sim \sG(\kappa)$, if $\mE[\exp(\lambda x))] \leq \exp(\frac{\kappa^2 \lambda^2}{2})$ for all $\lambda$. We will say that a mean-zero random variable $x$ is sub-exponential, $x \sim \sE(\nu, \alpha)$, if $\mE[\exp(\lambda x)] \leq \exp(\frac{\nu^2 \lambda^2}{2})$ for all $\abs{\lambda} \leq \frac{1}{\alpha}$. We will say that a mean-zero random vector is sub-gaussian, $\x \sim \sG(\kappa)$, if $\forall \v \in \mathbb{R}^p$, $\mE[\exp(\v^\top \x)] \leq \exp(\frac{\kappa^2 \Vert \v \Vert_2^2}{2})$. A standard Chernoff argument shows that if $x \sim \sE(\nu, \alpha)$ then $\Pr[\abs{x} \geq t] \leq 2 \exp(-\frac{1}{2} \min(\frac{t^2}{\nu^2}, \frac{t}{\alpha}))$. Throughout we will use $c,C$ to refer to universal constants that may change from line to line. 

\section{Proofs for \cref{sec:prelim}}
\label{app:prelim}
Here we provide a formal statement of \cref{thm:main_informal}. 

\begin{theorem}[Formal statement of \cref{thm:main_informal}]
	Suppose we are first given $n_1$ total samples from \cref{eq:model_main} which satisfy \cref{assump:design} and $\x_i \sim \cN(0, \I_d)$, with an equal number of samples from each task, which collectively satisfy \cref{assump:task}. Then, we are presented $n_2$ samples also from \cref{eq:model_main}, satisfying \cref{assump:design}, but from a $t+1$st task which satisfies $\norm{\balpha_{t+1}}^2 \leq O(1)$. If the $n_1$ samples are used in \cref{algo:LF_learn} to learn a feature representation $\Bone$, which is used in \cref{algo:LF_newtask} along with the $n_2$ samples to learn $\halpha$, and $n_1 \gtrsim \polylog(d, n_1) \frac{\bkappa dr}{\nu}$, $n_2 \gtrsim r \log n_2$, the excess prediction error on a new datapoint drawn from the covariate distribution, is,
	\begin{align*}
		\mE_{\xstar}[\langle \xstar, \Bone \halpha - \B \balpha_{t+1} \rangle^2] \leq \tlO \left(\frac{\bkappa dr}{\nu n_1}+\frac{r}{n_2} \right),
		 \label{eq:main}
	\end{align*}
	with probability at least $1-O(n_1^{-100}+n_2^{-100})$.
	\label{thm:main_formal}
\end{theorem}
\begin{proof}
	Note that $\mE_{\xstar}[\langle \xstar, \Bone \halpha - \B \balpha_{t+1} \rangle^2] = \norm{\Bone \halpha - \B \balpha_{t+1}}^2$. Combining \cref{thm:B_recovery}, \cref{thm:lr_transfer}  and applying a union bound then gives the result.
\end{proof}

Note that in order to achieve the formulation in \cref{thm:main_informal}, we make the simplifying assumption that the training tasks are well-conditioned in the sense that $\bar{\kappa} \leq \kappa \leq O(1)$ and $\nu \geq \Omega(\frac{1}{r}$)---which is consistent with the normalization in \cref{assump:task}. Such a setting is for example achieved (w.h.p.) if each $\balpha_t \sim \cN(0, \frac{1}{\sqrt{r}} \mathbf{\Sigma})$ where $\sigma_{1}(\mathbf{\Sigma})/\sigma_{r}(\mathbf{\Sigma}) \leq O(1)$.

\section{Proofs for \cref{sec:lf_mom}}
\label{app:mom}

Analyzing the performance of the method-of-moments estimator requires two steps. First, we show that the estimator $(1/n)\cdot \sum_{i=1}^n y_i^2 \x_i\x_i\trans$ converges to its mean in spectral norm, up to error fluctuations $\tlO(\sqrt{\frac{dr}{n}})$. Showing this requires adapting tools from the theory of matrix concentration.
Second, a standard application of the Davis-Kahan $\sin \theta$ theorem shows that top-r PCA applied to this noisy matrix, $(1/n)\cdot \sum_{i=1}^n y_i^2 \x_i\x_i\trans$, can extract a subspace $\hB$ close to the true column space of $\B$ up to a small error. Throughout this section we let $\barGamma = \sumton \barGamma_i$ with $\barGamma_i = \B \balpha_{t(i)} \balpha_{t(i)}^\top \B^\top$. We also let $\barLambda = \sumton \balpha_{t(i)} \balpha_{t(i)}^\top$ be the empirically observed task matrix. Note that under \cref{assump:task}, we have that $\barGamma$ and $\barLambda$ behave identically since $\B$ has orthonormal columns we have that $\tr(\barGamma) = \tr(\barLambda)$ and $\sigma_{r}(\barGamma) = \sigma_{r}(\barLambda)$. Furthermore throughout this section we use $\tkappa$ and $\tnu$ to refer to the average condition number and $r$-th singular value of the empirically observed task matrix $\barLambda$ -- since all our results hold in generality for this matrix. Note that under the uniform task observation model the task parameters of $\barLambda$ and the population task matrix $\frac{\A^\top \A}{t}$ are equal.

We first present our main theorem which shows our method-of-moments estimator can recover the true subspace $\B$ up to small error. 

\begin{proof}[Proof of \cref{thm:B_recovery}]
The proof follows by combining the Davis-Kahan sin $\theta$ theorem with our main concentration result for the matrix $\sumton y_i^2 \x_i^\top \x_i-\mE[\sumton y_i^2 \x_i^\top \x_i]$. First note that $\mE[\sumton y_i^2 \x_i^\top \x_i] = (2\barGamma + (1+\tr(\barGamma)) \I_d$ by \cref{lem:estimator_mean} and define $\sumton y_i^2 \x_i^\top \x_i-\mE[\sumton y_i^2 \x_i^\top \x_i] = \E$. Note that under the conditions of the result for $n \geq c d$ we have that $\norm{\E} \leq \tlO(\sqrt{\frac{d \tkappa r \tnu}{n}})$ by \cref{thm:estimator_conc} for large-enough $c$ due to the SNR normalization; so again by taking sufficiently large $c$ such that $n \geq c \cdot \polylog(d, n) \tkappa r d/\tnu$ we can ensure that $\norm{\E} \leq \delta \leq 2 \tnu$ for as small $\delta$ as we choose with the requisite probability.
	 Since $\norm{\bE} \leq \delta$ we have that $\sigma_{r+1}(y_i^2 \x_i\x_i\trans) - \sigma_{r+1}((2\barGamma + (1+\tr(\barGamma)) \I_d) \leq \delta$ and since $\barGamma$ is rank $r$,  $\sigma_{r}((2\barGamma + (1+\tr(\barGamma)) \I_d)-\sigma_{r+1}((2\barGamma + (1+\tr(\barGamma)) \I_d) = 2\sigma_{r}(\barGamma)$. Hence, applying the Davis-Kahan $\sin \theta$ theorem shows that, 
	\begin{align*}
		\norm{\Btwo^\top \B} \leq \frac{\norm{\Btwo^\top \E \B}}{2 \sigma_{r}(\barLambda) - \delta} \leq \frac{\norm{\E}}{2 \sigma_{r}(\barLambda) - \delta} \leq \frac{\norm{\E}}{\tnu} \leq \tlO \left(\sqrt{\frac{1}{\tnu} \frac{d \tkappa r}{n}} \right),
	\end{align*}
	where the final inequalities follows by taking $c$ large enough to ensure $\delta \leq \tnu$ and \cref{thm:estimator_conc}.

\end{proof}

We now present our main result which proves the concentration of the estimator,
\begin{theorem}
	Suppose the $n$ data samples $(\x_i, y_i)$ are generated from the model in \eqref{eq:model} and that \cref{assump:design,,assump:task} hold with $\x_i \sim \cN(0, 1)$ i.i.d. Then if $n \gtrsim c$ for sufficiently large $c$, 
	\begin{align*}
		& \norm{\sumton y_i^2 \x_i \x_i^\top - (2\barGamma + (1+\tr(\barGamma)) \I_d))} \leq \\
		& \log^3 n \cdot \log^3 d \cdot O \left(\sqrt{\frac{d \tkappa r \tnu}{n}} + \frac{d}{n} \right),
	\end{align*}
	with probability at least $1-O(n^{-100})$.
	\label{thm:estimator_conc}
\end{theorem}
\begin{proof}
	Note that the mean of $\sumton y_i^2 \x_i \x_i^\top$ is $2\barGamma + (1+\tr(\barGamma)) \I_d)$ by \cref{lem:estimator_mean}. Then using the fact that $y_i = \x_i^\top \B \balpha_{t(i)}+\epsilon_i$, we can write down the error decomposition for the estimator into signal and noise terms,
\begin{align*}
	 & \sumton y_i^2 \x_i\x_i\trans - (2\barGamma + (1+\tr(\barGamma)) \I_d = \sumton (\x_i \B \balpha_{t(i)})^2 \x_i \x_i^\top - (2\barGamma + \tr(\barGamma)) \I_d + \\
	 & \sumton 2 \epsilon_i \x_i \B \balpha_{t(i)} \x_i \x_i^\top + \sumton \epsilon_i^2 \x_i\x_i\trans - \I_d.
\end{align*}
We proceed to control the fluctuations of each term in spectral norm individually using tools from matrix concentration.
Applying \cref{lem:first_noise}, \cref{lem:second_noise}, \cref{lem:third_noise}, the triangle inequality and a union bound shows the desired quantity is upper bounded as,
\begin{align*}
& \log^3n  \cdot \log^3d \cdot O \left(\sqrt{\frac{d \max(1, \tr(\barGamma), \tr(\barGamma) \max_i \norm{\bbeta_i}^2)}{n}} + \frac{d \max(1, \max_i \norm{\bbeta_i}, \max_i \norm{\bbeta_i}^2)}{n} \right).
& 
\end{align*}
Finally, using \cref{assump:task} and the fact that $\tr{\barGamma} = \tr{\barLambda}$ and the fact $\norm{\bbeta_i}=\norm{\balpha_i}$, simplifies the result to the theorem statement. Note that since $\norm{\balpha_i} = \Theta(1)$ for all $i$ we have that, $\tr(\barGamma) = \Theta(1)$ so the SNR normalization guarantees the leading noise term satisfies $1 \leq O( \tr(\barGamma))$.
\end{proof}

We begin by computing the mean of the estimator.
\begin{lemma}
	\label{lem:estimator_mean}
	Suppose the $n$ data samples $(\x_i, y_i)$ are generated from the model in \eqref{eq:model_main} and that \cref{assump:design,,assump:task} hold. Then,
	\begin{align*}
		\mE[\sumton y_i^2 \x_i \x_i^\top] = 2\barGamma + (1+\tr(\barGamma)) \I_d
	\end{align*}
	where $\barGamma= \sumton \barGamma_i$ with $\mGamma_i = \B \balpha_{t(i)} \balpha_{t(i)}^\top \B^\top$.
\end{lemma}
\begin{proof}
	Since $\epsilon_i$ is mean-zero, using the definition of $y_i$ we immediately obtain,
	\begin{align*}
		\mE[\sumton y_i^2 \x_i \x_i^\top] = \I_d + \mE[\sumton \x_i^\top \mGamma_i \x_i \x_i \x_i^\top] = \I_d + \mE[\x^\top \barGamma \x \x \x^\top],
	\end{align*}
	for $\x \sim \cN(0, \I_d)$. Using the eigendecomposition of $\barGamma$ we have that $\mE[\x^\top \barLambda \x \x \x^\top] = \sumtor \sigma_i \mE[(\x^\top \v_i)^2 \x\x^\top]$. Due to the isotropy of the Gaussian distribution, it suffices to compute $\mE[(\x^\top \e_1)^2 \x\x^\top]$ and rotate the result back to $\v_i$. In particular we have that, 
	\begin{align*}
		& (\mE[(\x^\top \e_1)^2 \x\x^\top])_{ij} = \begin{cases}
			0 \quad i \neq j \\
			1 \quad i=j \neq 1 \\
			3  \quad i=j=1 \\
		\end{cases} \implies 
		\mE[(\x^\top \e_1)^2 \x\x^\top] = 2 \e_1 \e_1 + \I_d \implies \\
		& \implies \mE[(\x^\top \v_i)^2 \x\x^\top] = 2 \v_i \v_i + \I_d \implies \E[\x^\top \barGamma \x \x \x^\top] = 2 \barGamma + \tr(\barGamma) \I_d,
	\end{align*}
	from which the conclusion follows.
\end{proof}

We start by controlling the fluctuations of the final noise term (which has identity mean). 
\begin{lemma}
	\label{lem:first_noise}
Suppose the $n$ data samples $(\x_i, y_i)$ are generated from the model in \eqref{eq:model_main} and that \cref{assump:design,,assump:task} hold. Then for $n \geq c$ for sufficiently large $c$,
	\begin{align*}
		\norm{\sumton \epsilon_i^2 \x_i \x_i^\top-\I_d} \leq O \left(\log^2 n \left(\sqrt{\frac{d}{n}} + \frac{d}{n} \right) \right).
	\end{align*}
	with probability at least $1-O(n^{-100})$.
\end{lemma}
\begin{proof}
	We first decompose the expression as,
	\begin{align*}
		\norm{\sumton \epsilon_i^2 \x_i \x_i^\top-\I_d} \leq \norm{\sumton \epsilon_i^2 \x_i \x_i^\top - \sumton \epsilon_i^2 \I_d} + \norm{\sumton \epsilon_i^2 \I_d - \I_d}
	\end{align*}
	
	We begin by controlling the second term. By a sub-exponential tail bound we have that $\Pr[\abs{\sumton \epsilon_i^2-1}\geq t] \leq 2 \exp(-C n\min(t^2/8^2, t/8))$, since $\epsilon_i^2$ is $\sE(8,8)$ by \cref{lem:prod_sg}. Letting $t=c \sqrt{\log(\frac{1}{\delta})/n}$ for sufficiently large $c$, and assuming $n \gtrsim \log(\frac{1}{\delta})$, implies $\abs{\sumton \epsilon_i^2-1} \leq  O(\sqrt{\frac{\log(1/\delta)}{n}})$ with probability at least $1-2\delta$. Hence   $\norm{\sumton \epsilon_i^2 \I_d - \I_d} \leq O(\sqrt{\frac{\log(1/\delta)}{n}})$ on this event. 
	
	Now we apply \cref{lem:scaled_sgvec} with $a_i = \epsilon_i$ to control the first term.
	Using the properties of sub-Gaussian maxima we can conclude that $\Pr[\max_i \abs{\epsilon}_i \geq t] \leq 2n \exp(-t^2/2)$; taking $t=4 \sqrt{\log n}+c\sqrt{\log(1/\delta)}$ for sufficiently large $c$ implies that $\max_{i} \abs{\epsilon_i} \leq O(\sqrt{\log n})+O(\sqrt{\log(1/\delta)})$ with probability at least $1-\delta$. In the setting of \cref{lem:scaled_sgvec}, conditionally on $\epsilon_i$, $K = \max_{i} |\epsilon_i|$ and $\mSigma=\I_d$ so taking $t = c\sqrt{\log(1/\delta)}$ for sufficiently large $c$ implies that $\norm{\sumton \epsilon_i^2 \x_i \x_i^\top - \sumton \epsilon_i^2 \I_d} \leq K \cdot O(\sqrt{d/n}+\sqrt{\log(1/\delta)/n} + \frac{d}{n} + \frac{\log(1/\delta)}{n})$ with probability at least $1-2 \delta$ conditionally on $\epsilon_i$. Conditioning on the event that $\max_{i} \abs{\epsilon_i} \leq O(\sqrt{\log n})+O(\sqrt{\log(1/\delta)})$ to conclude the argument finally shows that, 
	\begin{align*}
		& \norm{\sumton \epsilon_i^2 \x_i \x_i^\top-\I_d} \leq O(\log n+\log(1/\delta)) \cdot O(\sqrt{\frac{d}{n}}+\sqrt{\frac{\log(1/\delta)}{n}} + \frac{d}{n} + \frac{\log(1/\delta)}{n}) + O(\sqrt{\frac{\log(1/\delta)}{n}}), 
	\end{align*}
	with probability at least $1-5 \delta$.
	Selecting $\delta = n^{-100}$ implies that $\norm{\sumton \epsilon_i^2 \x_i \x_i^\top-\I_d} \leq O \left(\log^2 n \left(\sqrt{\frac{d}{n}} + \frac{d}{n} \right) \right)$,
	with probability at least $1-O(n^{-100})$.
\end{proof}

We now proceed to controlling the fluctuations of the second noise term (which is mean-zero). Our main technical tool is \cref{lem:trunc_mb}.

\begin{lemma}
	\label{lem:second_noise}
	Suppose the $n$ data samples $(\x_i, y_i)$ are generated from the model in \eqref{eq:model_main} and that \cref{assump:design,,assump:task} hold. Then,
	\begin{align*}
		\norm{\sumton 2 \epsilon_i \x_i^\top \B \balpha_{t(i)} \x_i \x_i^\top} \leq O((\log n + \log d) \left(\sqrt{\frac{d \tr(\barGamma)}{n}} +  \frac{d \max_i \norm{\bbeta_i} (\log^2(n) + \log^2(d))}{n} \right).
	\end{align*}
	with probability at least $1-O((nd)^{-100})$.
\end{lemma}
\begin{proof}
	To apply the truncated version of the matrix Bernstein inequality (in the form of \cref{lem:trunc_mb}) we need to set an appropriate truncation level $R$, for which need control on the norms of $Z_i = \norm{2 \epsilon_i \x_i^\top \B \balpha_{t(i)} \x_i \x_i^\top} = 2 \abs{\epsilon_i} \norm{\x_i}^2 \abs{\x_i^\top \B \balpha_{t(i)}}$. Using sub-gaussian, and sub-exponential tail bounds we have that $\abs{\epsilon_i} \leq O(1+ \sqrt{\log(1/\delta)})$, $\norm{\x_i}^2 \leq O(d+\max(\sqrt{d \log(1/\delta)}, \log(1/\delta))) = O(d+\sqrt{d} \log(1/\delta))$, and $\abs{\x_i^\top \B \balpha_{t(i)}} \leq O(\norm{\bbeta_i}(1+\sqrt{\log(1/\delta)})$ each with probability at least $1-\delta$. Accordingly with probability at least $1-3\delta$ we have that $\norm{Z_i} \leq O(\norm{\bbeta_i} d(1+\log^2(1/\delta)))$. We can rearrange this statement to conclude that $\Pr[\norm{Z_i} \geq c_1 \norm{\bbeta_i} d + t] \leq 3 \exp(-c_2 (\frac{t}{\norm{\bbeta_i} d})^{1/2})$ for some $c_1, c_2$. Define a truncation level $R = c_1 \max_i \norm{\bbeta_i} d + K \max_i \norm{\bbeta_i} d$ for some $K$ to be chosen later. We can also use the aforementioned tail bound to control $\norm{\mE[Z_i]-\mE[Z_i']} \leq \mE[Z_i \Ind[\norm{Z_i} \geq \alpha]] \leq \int_{K \norm{\bbeta_i} d}^{\infty} 3 \exp(-c_2 (\frac{t}{\norm{\bbeta_i} d})^{1/2}) \leq O((1+\sqrt{K}) \exp(-c \sqrt{K}) \max_i \norm{\bbeta_i} d) = \Delta$.

	Next we must compute an upper bound for the matrix variance term $$\norm{\sum_{i=1}^n \mE[\epsilon_i^2 (\balpha_i^\top \B^\top \x_i)^2 \norm{\x_i}^2 \x_i \x_i^\top]} = \norm{\sum_{i=1}^n \mE[(\balpha_i^\top \B^\top \x)^2 \norm{\x}^2 \x \x^\top]} = n \norm{\mE[\x^\top \barGamma \x \norm{\x}^2 \x \x^\top]} = n \norm{\sumtor \sigma_i \mE[(\v_i^\top \x)^2 \norm{\x}^2 \x \x^\top]}$$, taking an expectation over $\epsilon_i$ in the first equality, and diagonalizing $\barGamma$.
	As before due to isotropy of the Gaussian it suffices to compute the expectation with $\v_i = \e_1$ and rotate the result back to $\v_i$. Before computing the term we note that for a standard normal gaussian random variable $g \sim \cN(0,1)$ we have that $\mE[g^6]=15$, $\mE[g^4]=3$, $\mE[g^2]=1$. Then by simple combinatorics we find that,
	\begin{align*}
		& (\mE[x_1^2 (\sum_{a=1}^n x_a^2) \x \x^\top])_{ij} = \begin{cases}
			0 \quad i \neq j \neq 1 \\
			0 \quad i = 1 \neq j \\
			2 \cdot 3 \cdot 1 + (d-2) \cdot 1 \quad i=j \neq 1 \\
		    15 + 3(d-1) \quad i=j=1 \\
		\end{cases} \implies \\
		& \mE[(\x^\top \e_1)^2 \norm{\x}^2 \x \x^\top] = (2d+8) \e_1 \e_1^\top + (d+4) \I_d \\
		& \implies \norm{\sum_{i=1}^n \mE[\epsilon_i^2 (\balpha_i^\top \B^\top \x_i)^2 \norm{\x_i}^2 \x_i \x_i^\top]} \leq n(d+4) \norm{\barGamma + \tr(\barGamma) \I_d} \leq 10 nd (\tr(\barGamma)) = \sigma^2.
	\end{align*}

Finally, we can assemble the previous two computations to conclude the result with appropriate choices of $R$ (parametrized through $K$) and $t$ by combining with \cref{lem:trunc_mb}. Before beginning recall by definition we have that $\tr(\barGamma) = \frac{1}{n}  \max_i \norm{\bbeta_i}^2$. Let us choose $\sqrt{K} = \frac{c_3}{c} (\log(n) + \log(d))$ for some sufficiently large $c_3$. In this case, we can choose $c_3$ such that $\Delta \leq O((\log n + \log d) \frac{\max_i \norm{\bbeta_i} d}{n^{10} d^{10}}) \leq O(\frac{\sqrt{\tr(\barGamma)}}{n^8 d^8})$, since $\sqrt{\tr(\barGamma)} \geq \frac{1}{\sqrt{n}} \max_i \norm{\bbeta_i}$. Similarly, our choice of truncation level becomes $R = O( (\log^2(n) + \log^2(d)) \max_i \norm{\bbeta_i} d)$. At this point we now choose $t = c_4(\log n + \log d) \max(\sigma/\sqrt{n}, R/n)$ for sufficiently large $c_4$. For large enough $c_4$ we can guarantee that $t \geq 2 \Delta \implies t - \Delta \geq \frac{t}{2}$.

Hence combining these results together and applying \cref{lem:trunc_mb} we can provide the following upper bound on the desired quantity,
\begin{align*}
	\Pr[\norm{\sumton 2 \epsilon_i \x_i^\top \B \balpha_{t(i)} \x_i \x_i^\top} \geq t] \leq O(d \exp(-c \cdot c_4 (\log n + \log d)) + O(n \sqrt{K} \exp(-c_3 (\log n + \log d)) \leq O((nd)^{-100}),
\end{align*}
by taking $c_3$ and $c_4$ sufficiently large, with $t = O((\log n + \log d) \left(\sqrt{\frac{d \tr(\barGamma)}{n}} +  \frac{d \max_i \norm{\bbeta_i} (\log^2(n) + \log^2(d))}{n} \right)$.
\end{proof}

Finally we turn to controlling the fluctuations of the primary signal term around its mean using a similar argument to the previous term.

\begin{lemma}
	\label{lem:third_noise}
	Suppose the $n$ data samples $(\x_i, y_i)$ are generated from the model in \eqref{eq:model_main} and that \cref{assump:design,,assump:task} hold. Then
	\begin{align*}
		& \norm{\sumton (\x_i \B \balpha_{t(i)})^2 \x_i \x_i^\top-(2\barGamma + \tr(\barGamma) \I_d)} \leq \\ & O((\log n + \log d) \left(\sqrt{\frac{d \tr(\barGamma) \max_i \norm{\bbeta_i}^2}{n}} +  \frac{d \max_i \norm{\bbeta_i}^2 (\log^2(n) + \log^2(d))}{n} \right),
	\end{align*}
	with probability at least $1-O((nd)^{-100})$.
\end{lemma}
\begin{proof}
	The proof is similar to the proof of \cref{lem:second_noise} and uses \cref{lem:trunc_mb}. We begin by controlling the norms of $Z_i = (\x_i^\top \B \balpha_{t(i)})^2 \x_i \x_i^\top$. $\norm{Z_i} = \norm{\x_i}^2 (\x_i^\top \B \balpha_{t(i)})^2$. Using Gaussian and sub-exponential tail bounds we have that, $\norm{\x_i}^2 \leq O(d+\sqrt{d} \log(1/\delta))$ and $(\x_i^\top \B \balpha_{t(i)})^2 \leq O(\norm{\bbeta_i}^2(1+\log(1/\delta)))$ each with probability at least $1-\delta$. Hence with probability at least $1-2\delta$ we find that $\norm{Z_i} \leq O(d \norm{\bbeta_i}^2 (1+ \log^2(1/\delta))$.

	We can rearrange this statement to conclude that $\Pr[\norm{Z_i} \geq c_1 \norm{\bbeta_i}^2 d + t] \leq 2 \exp(-c_2 (\frac{t}{\norm{\bbeta_i}^2 d})^{1/2})$ for some $c_1, c_2$. Define a truncation level $R = c_1 \max_i \norm{\bbeta_i}^2 d + K \max_i \norm{\bbeta_i}^2 d$ for some $K$ to be chosen later. We can use the aforementioned tail bound to control $\norm{\mE[Z_i]-\mE[Z_i']} \leq \mE[Z_i \Ind[\norm{Z_i} \geq \alpha]] \leq \int_{K \norm{\bbeta_i}^2 d}^{\infty} 2 \exp(-c_2 (\frac{t}{\norm{\bbeta_i}^2 d})^{1/2}) \leq O((1+\sqrt{K}) \exp(-c \sqrt{K}) \max_i \norm{\bbeta_i}^2 d) = \Delta$.

	Next we must compute an upper bound the matrix variance term $$\norm{\sum_{i=1}^n \mE[(\balpha_i^\top \B^\top \x_i)^4 \norm{\x_i}^2 \x_i \x_i^\top]} = \norm{\sum_{i=1}^n \mE[(\bbeta_i^\top \x)^4 \norm{\x}^2 \x \x^\top]}.$$
	As before due to isotropy of the Gaussian it suffices to compute each expectation assuming $\bbeta_i \propto \e_1$ and rotate the result back to $\bbeta_i$. Before computing the term we note that for a standard normal gaussian random variable $g \sim \cN(0,1)$ we have that $\mE[g^8]=105$, $\mE[g^6]=15$, $\mE[g^4]=3$, $\mE[g^2]=1$. Then by simple combinatorics we find that,
	\begin{align*}
		& (\mE[(\x^\top \e_1)^4 (\sum_{a=1}^n x_a^2) \x \x^\top])_{ij} = \begin{cases}
			0 \quad i \neq j \neq 1 \\
			0 \quad i = 1 \neq j \\
			15+3 \cdot 3 + (d-2) \cdot 3 \quad i=j \neq 1 \\
		    105 + 15(d-1) \quad i=j=1 \\
		\end{cases} \implies \\
		& \mE[(\x^\top \bbeta_i)^4 \norm{\x}^2 \x \x^\top] = (2d+75) \e_1 \e_1^\top + (3d+15) \I_d \\
		& \implies \norm{\sum_{i=1}^n \mE[(\balpha_i^\top \B^\top \x_i)^4 \norm{\x_i}^2 \x_i \x_i^\top]} \leq O(d) \norm{\sum_{i}^n \norm{\bbeta_i}^4(\bbeta_i \bbeta_i^\top + \I_d)} \leq O(d \sum_{i=1}^n \norm{\bbeta_i}_2^4) \leq \\
		& O(dn \max_i \norm{\bbeta_i}^2 \tr(\barGamma)) = \sigma^2.
	\end{align*}
Finally, we can assemble the previous two computations to conclude the result with appropriate choices of $R$ (parametrized through $K$) and $t$ by combining with \cref{lem:trunc_mb}. Before beginning recall by definition we have that $\tr{\barGamma} \geq \frac{1}{n}  \max_i \norm{\bbeta_i}^2$. Let us choose $\sqrt{K} = \frac{c_3}{c} (\log(n) + \log(d))$ for some sufficiently large $c_3$. In this case, we can choose $c_3$ such that $\Delta \leq O((\log n + \log d) \frac{\max_i \norm{\bbeta_i}^2 d}{n^{10} d^{10}}) \leq O(\frac{\sqrt{\tr(\barGamma)} \max_i \norm{\bbeta_i}}{n^7 d^7})$, since $\sqrt{\tr(\barGamma)} \geq \frac{1}{\sqrt{n}} \max_i \norm{\bbeta_i}$. Similarly, our choice of truncation level becomes $R = O( (\log^2(n) + \log^2(d)) \max_i \norm{\bbeta_i}^2 d)$. At this point we now choose $t = c_4(\log n + \log d) \max(\sigma/\sqrt{n}, R/n)$ for sufficiently large $c_4$. For large enough $c_4$ we can guarantee that $t \geq 2 \Delta \implies t - \Delta \geq \frac{t}{2}$.

Hence combining these results together and applying \cref{lem:trunc_mb} we can provide the following upper bound on the desired quantity:
\begin{align*}
	\lefteqn{ \Pr[\norm{\sumton (\x_i \B \balpha_{t(i)})^2 \x_i \x_i^\top-(2\barGamma + \tr(\barGamma) \I_d)} \geq t]} \\ 
	& \leq O(d \exp(-c \cdot c_4 (\log n + \log d)) + O(n \sqrt{K} \exp(-c_3 (\log n + \log d)) \\
	& \leq 
	 O((nd)^{-100}),
\end{align*}
by taking $c_3$ and $c_4$ sufficiently large, with \begin{align*}
t = O((\log n + \log d) \left(\sqrt{\frac{d \tr(\barGamma) \max_i \norm{\bbeta_i}^2}{n}} +  \frac{d \max_i \norm{\bbeta_i}^2 (\log^2(n) + \log^2(d))}{n} \right).
\end{align*}.
\end{proof}


\section{Proofs for \cref{sec:lf_gd}}
\label{sec:landscape}
In our landscape analysis we consider a setting with $t$ tasks and we observe a datapoint from each of the $t$ tasks uniformly at random at each iteration. Formally, 
we define the matrix we are trying to recover as 
\begin{align}
\M_{\star} = (\balpha_1, \hdots, \balpha_t)^\top \B^\top \underbrace{=}_{\text{SVD}} \X^{\star} \D^{\star} (\Y{\star})^\top \in \mR^{t \times d},
\end{align}
with $\U^{\star}=\X^{\star} (\D^{\star})^{1/2}$, and $(\D^{\star})^{1/2} (\Y^{\star})^\top = (\V^{\star})^\top$, from which we obtain the observations:
\begin{align}
	y_i = \langle \e_{t(i)} \x_i^\top, \M_{\star} \rangle + \sigma \cdot \epsilon_i,
\end{align}
where we sample tasks uniformly $t(i) \sim \{1, \hdots, t\}$ and $\x_i$ is a sub-gaussian random vector. Note that $\M^{\star}$ is a rank-$r$ matrix, $\U^{\star} \in \mR^{t \times r}$, and $\V^{\star} \in \mR^{d \times r}$.  In this section, we denote $\tld = \max\{t, d\}$ and let $\sigstarl, \sigstarr$ be the $1$-st and $r$-th eigenvalue of matrix $\M^\star$. We denote $\cn^\star = \sigstarl/\sigstarr$ as its condition number.
Note that as $\B$ is an orthonormal matrix we have that $\M^{\star} (\M^{\star})^\top = t \cdot \A^\top \A/t$ from which it follows that $(\sigstarl)^2=t \cdot \sigma_1(\A^\top \A/t) \leq t \bkappa \nu \leq O(t)$ by the normalization on $\norm{\balpha_i}$. Similarly $(\sigstarr)^2 = t \sigma_r(\A^\top \A/t) \geq t \nu$. So it follows that $\sigstarl \leq \sqrt{t \kappa \nu}$, $\sigstarr \geq \sqrt{t \nu}$ and $\cn^\star \leq \sqrt{\kappa}$. We use this to simplify the preconditions and the statement of the incoherence ball in the main although we work in full generality throughout the Appendix.

We now present the proof of our main result.

\begin{proof}[Proof of \cref{thm:main_landscape_convert}]
	Under the conditions of the theorem note that by \cref{thm:main_landscape} we have that, 	
\begin{equation*}
	\norm{\U\V^\top - \M^\star}_F\leq O\left(\sigma\sqrt{t\frac{\max\{t, d\} r\log n}{n}}\right),
\end{equation*}
for $n \ge \polylog(n,d,t) C \mu^2 r^4 \max\{t, d\}  (\cn^\star)^4$. 
 First recall by \cref{lem:diverse_to_incoh} the incoherence parameter can in fact be shown to be $\mu = O(\bkappa)$ under our assumptions which gives the precondition on the sample complexity due to the task diversity assumption and normalization.
	To finally convert this bound to a guarantee on the subspace angle we directly apply \cref{lem:frob_to_angle} once again noting the task diversity assumption. Lastly note that as $\B$ is orthonormal we have that $\sigstarl \leq \sqrt{t \kappa \nu}$, $\sigstarr \geq \sqrt{t \nu}$ and $\cn^\star \leq \sqrt{\kappa}$ as previously argued and $\sigma=1$ under the conditions of the result.
\end{proof}

\subsection{Geometric Arguments for Landscape Analysis}
Our arguments here are generally applicable to various matrix sensing/completion problems so we define some generic notation:
\begin{equation}\label{eq:objective}
f(\U, \V) = \frac{2}{n}\sum_{i=1}^n (\la \A_i, \U \V\trans \ra - \sqrt{t} y_i)^2 + \frac{1}{2}\fnorm{\U\trans\U - \V\trans\V}^2,
\end{equation}
where $\A_i = \sqrt{t} \e_{t(i)} \x_i\trans$. We work under the following constraint set for large constant $C_0$:
\begin{equation}\label{eq:constraint}
\cW = \{~ (\U, \V) ~|~ \max_{i\in[t]} \norm{\e_i\trans \U}^2 \le \frac{C_0 \mu r \sigstarl}{t}, \quad 
\norm{\U}^2 \le C_0 \sigstarl, \quad \norm{\V}^2 \le C_0 \sigstarl ~\}.
\end{equation}

We renormalize the statistical model for convenience simply for the purposes of the proof throughout \cref{sec:landscape} the remainder of  as:
\begin{equation}\label{eq:model}
y_i = \la \A_i, \M^\star\ra + n_i,
\end{equation}
where $n_i \sim \sqrt{t} \sigma \cdot \epsilon_i$ and where $\epsilon_i$ is a sub-gaussian random vector with parameter 1 (note this is because we have scaled $\A_i$ up by a factor of $\sqrt{t}$). $\M^\star$ is rank $r$, and we let $\X$ be the left singular vector of $\M^\star$, and assume $\X$ is $\mu$-incoherent;\footnote{Note that for our particular problem this is not an additional assumption since by \cref{lem:diverse_to_incoh} our task assumptions imply this.} i.e., $\max_{i}\norm{\e_i\trans \X}^2 \le \mu r/t$. 

We now reformulate the objective (denoting $\M = \U\V\trans$) as
\begin{equation}\label{eq:objective_reformed}
f(\U, \V) = 2 (\M - \M^\star):\H_0:(\M - \M^\star) + \frac{1}{2} \fnorm{\U\trans \U - \V\trans \V}^2 + Q(\U),
\end{equation}
where $\M:\H_0:\M = \frac{1}{n}\sum_{i=1}^n \la \A_i, \M \ra^2$ and $\E [\M:\H_0:\M] = \fnorm{\M}^2$ and $Q$ is a regularization term:
\begin{equation}\label{eq:regularization}
Q(\U, \V) = \frac{2}{n} \sum_{i=1}^n[(\la \M-\M^\star,\A_i\ra - n_{i})^2 - (\la \M-\M^\star,\A_i\ra)^2].
\end{equation}

In this section, we denote $\tld = \max\{t, d\}$ and let $\sigstarl, \sigstarr$ be the $1$-st and $r$-th eigenvalue of matrix $\M^\star$. We denote $\cn^\star = \sigstarl/\sigstarr$ as its condition number.

The high-level idea of the analysis uses ideas from \citet{ge2017no}. The overall strategy is to argue that if we are currently not located at local minimum in the landscape we can certify this by inspecting the gradient or Hessian of $f(\U, \V)$ to exhibit a direction of local improvement $\Delta$ to decrease the function value of $f$. Intuitively this direction brings us close to the true underlying $(\U^\star, \V^\star)$.

We now establish some useful definitions and notation for the following analysis

\subsubsection{Definitions and Notation}

\begin{definition}
\label{def:asymmetricquantities}
Suppose $\M^\star$ is the optimal solution with SVD is $\X^\star\D^\star \Y^\star{}\trans$. Let $\U^\star = \X^\star (\D^\star)^{\frac{1}{2}}$, $\V^\star = \Y^\star (\D^\star)^{\frac{1}{2}}$. Let $\M = \U\V^\top$ be the current point in the landscape. We reduce the problem of studying an asymmetric matrix objective to the symmetric case using the following notational transformations:
\begin{equation}
\W = \begin{pmatrix}\U \\ \V\end{pmatrix}, 
\W^\star = \begin{pmatrix}\U^\star \\ \V^\star\end{pmatrix}, 
\N = \W\W\trans, 
\N^\star = \W^\star \W^\star {}\trans\label{eq:defwn}
\end{equation}
\end{definition}

We will also transform the Hessian operators to operate on $(t+d)\times r$ matrices. In particular, define the Hessians $\H_1, \G$ such that for all $\W$ we have:
\begin{align*}
 \N:\H_1:\N &= \M:\H_0:\M \\
 \N:\G:\N  &= \norm{\U\trans \U - \V\trans \V}_F^2.
\end{align*}

Now, letting $Q(\W) = Q(\U,\V)$, we can rewrite the objective function $f(\W)$ as
\begin{equation}
\frac{1}{2}\left[(\N - \N^\star):4\H_1:(\N - \N^\star) + \N:\G:\N \right]+ Q(\W) \label{eq:a}.
\end{equation}

We now introduce the definition of local alignment of two matrices.
\begin{definition}\label{def:delta} Given matrices $\W, \W^\star \in \mR^{d\times r}$, define their difference $\Delta = \W - \W^\star \R^\star$, where $\R^\star\in \mR^{r\times r}$ is chosen as $\mat{R}^\star = \argmin_{\mat{Z}^\top \mat{Z} = \mat{Z}\mat{Z}^\top = \I}\|\W-\W^\star\mat{Z}\|_F^2.$
\end{definition}

Note that this definition tries to ``align'' $\U$ and $\U^\star$ before taking their difference, and therefore is invariant under rotations. In particular, this definition has the nice property that as long as $\N = \W\W^\top$ is close to $\N^\star = \W^\star(\W^\star)^\top$ in Frobenius norm, the corresponding $\Delta$ between them is also small (see Lemma \ref{lem:bound}).

\subsubsection{Proofs for Landscape Analysis}
With these definitions in hand we can now proceed to the heart of the landscape analysis. Since $\W^\star$ has rotation invariance, in the following section we always choose $\W^\star$ so that it aligns with the corresponding $\W$ according to Definition \ref{def:delta}.

We first restate a useful result from \citet{ge2017no},
\begin{lemma}[{\citep[Lemma 16]{ge2017no}}]\label{lem:asymmetricmain}
For the objective \eqref{eq:a}, let $\Delta, \N, \N^\star$ be defined as in Definition~\ref{def:asymmetricquantities}, Definition~\ref{def:delta}. Then, for any $\W\in \R^{(t + d)\times r}$, we have
\begin{align}
\Delta : \hess f(\W) :\Delta 
\le &\Delta\Delta\trans: \H :\Delta\Delta\trans - 3(\N - \N^\star):\H:(\N - \N^\star)  \nn \\
&+ 4\la \grad f(\W), \Delta\ra + [\Delta : \hess Q(\W) : \Delta - 4 \la \grad Q(\W), \Delta \ra], \label{eq:main_asym}
\end{align}
where $\H = 4\H_1+\G$. Further, if $\H_0$ satisfies $\M:\H_0:\M \in (1\pm \delta)\|\M\|_F^2$ for some matrix $\M = \U\V^\top$, let $\W$ and $\N$ be defined as in \eqref{eq:defwn}, then $\N:\H:\N \in (1\pm 2\delta)\|\N\|_F^2$.
\end{lemma}

With this result we show a key result which shows that with enough samples all stationary points in the incoherence ball $\cW$ that are not close to $\W^{\star}$ have a direction of negative curvature.

\begin{lemma} \label{lem:step2_mc_asym}
If \cref{assump:design} holds, then when the number of samples $n \ge C \polylog(d,n,t) \mu^2 r^4 \max\{t, d\}  (\cn^\star)^4$ for a sufficiently large constant $C$, with probability at least $1-1/\poly(d)$, all stationary points $\W \in \text{int}(\cW)$ satisfy:
\begin{equation*}
\Delta\Delta\trans:\H:\Delta\Delta\trans - 3(\N - \N^\star):\H:(\N - \N^\star) \le -0.1  \fnorm{\N - \N^\star}^2.
\end{equation*}
\end{lemma}
\begin{proof}
We divide the proof into two cases according to the norm of $\Delta$ and use different concentration inequalities in each case.
In this proof, we denote $\Delta = (\Delta_\U\trans, \Delta_\V\trans)\trans$, clearly, we have $\fnorm{\Delta_\U} \le \fnorm{\Delta}$ and  $\fnorm{\Delta_\V} \le \fnorm{\Delta}$.

\noindent\textbf{Case 1:} $\fnorm{\Delta}^2 \le \sigstarr/1000$. In this case, $\fnorm{\Delta_\U}^2  \le \fnorm{\Delta}^2 \le\sigstarr/1000$ and $\fnorm{\Delta_\V}^2  \le \fnorm{\Delta}^2 \le\sigstarr/1000$. By \eqref{eq:DD_concen}, we have
\begin{align*}
\Delta\Delta\trans: \H:\Delta\Delta\trans \le \fnorm{\Delta \Delta\trans}^2 + 0.004\sigstarr \fnorm{\Delta_\V}^2 \le 0.005 \sigstarr  \fnorm{\Delta}^2
\end{align*}
On the other hand, denote $\mS = \W^\star\Delta\trans + \Delta(\W^\star)\trans$, by \eqref{eq:DWstar_concen} and Lemma \ref{lem:asymmetricmain}, we know:
\begin{equation*}
\mS : \H: \mS \ge 0.999\fnorm{\mS}^2.
\end{equation*}
Since we choose $\W^\star$ to align with the corresponding $\W$ according to Definition \ref{def:delta}, by \cref{lem:bound}.
\begin{align*}
\fnorm{\mS}^2 = 2(\fnorm{\Delta\trans \W^\star}^2+ \fnorm{\Delta (\W^\star)\trans}^2)
\ge 2\fnorm{\Delta (\W^\star)\trans}^2 \ge 2\sigstarr \fnorm{\Delta}^2.
\end{align*}
This gives:
\begin{align*}
\lefteqn{\Delta\Delta\trans:\H:\Delta\Delta\trans - 3(\N - \N^\star):\H:(\N - \N^\star)} \\
& = \Delta\Delta\trans:\H:\Delta\Delta\trans - 3(\mS + \Delta\Delta\trans):\H:(\mS+\Delta\Delta\trans) \\
& \le -6 \mS:\H:\Delta\Delta\trans - 3\mS:\H:\mS\\
& \le - \mS:\H:\mS - 2\sqrt{\mS:\H:\mS}(\sqrt{\mS:\H:\mS} - 3\sqrt{\Delta\Delta\trans :\H:\Delta\Delta\trans})\\
& \le  -0.999\fnorm{\mS}^2 - 2\sqrt{\mS:\H:\mS} \cdot \sqrt{\sigstarr} \cdot(\fnorm{\Delta} - 0.3\fnorm{\Delta})
\le -0.999\fnorm{\mS}^2.
\end{align*}
Finally, we know $\N - \N^\star = \mS + \Delta \Delta\trans$, and $\fnorm{\mS}^2 \ge 2\sigstarr \fnorm{\Delta}^2 \ge 500 \fnorm{\Delta}^4 = 500 \fnorm{\Delta \Delta\trans}^2$. Therefore:
\begin{equation*}
\fnorm{\N - \N^\star} \le \fnorm{\mS} + \fnorm{\Delta \Delta\trans} \le 2\fnorm{\mS}.
\end{equation*}
This gives:
\begin{equation*}
\Delta\Delta\trans:\H:\Delta\Delta\trans - 3(\N - \N^\star):\H:(\N - \N^\star) \le -0.999\fnorm{\mS}^2 \le -0.1 \fnorm{\N - \N^\star}^2.
\end{equation*}
~

\noindent\textbf{Case 2:} $\fnorm{\Delta}^2 \ge \sigstarr/1000$, by \eqref{eq:M_concen}, we have:
\begin{align*}
\frac{1}{n}\sum_{i=1}^n\la \A_i, \M - \M^\star\ra^2 
\ge& \fnorm{\M-\M^\star}^2 - (\sigstarr)^2 /10^{6} \ge \fnorm{\M-\M^\star}^2  - 0.001\sigstarr \fnorm{\Delta}^2.
\end{align*}
This implies:
\begin{equation*}
(\N - \N^\star):\H:(\N - \N^\star) \ge \fnorm{\N-\N^\star}^2 - 0.004\sigstarr \fnorm{\Delta}^2.
\end{equation*}
Then by \eqref{eq:DD_concen}, we have:
\begin{align*}
&\Delta\Delta\trans:\H:\Delta\Delta\trans - 3(\N - \N^\star):\H:(\N - \N^\star) \\
\le& \fnorm{\Delta\Delta\trans}^2 +0.004\sigstarr\fnorm{\Delta}^2 - 3(\fnorm{\N-\N^\star}^2 -0.004\sigstarr\fnorm{\Delta}^2) \\
\le& -\fnorm{\N-\N^\star}^2 + 0.016\sigstarr\fnorm{\Delta}^2
\le -0.1\fnorm{\N-\N^\star}^2,
\end{align*}
where the last step follows by applying \cref{lem:bound}. This finishes the proof.
\end{proof}
With this key structural lemma in hand, we now present the main technical result for the section which characterizes the effect of the additive noise $n_i$ on the landscape.

\begin{theorem}\label{thm:main_landscape}
If \cref{assump:design} holds, when the number of samples $n \ge C \polylog(n,d,t) \mu^2 r^4 \max\{t, d\}  (\cn^\star)^4$ for sufficiently large constant $C$, with probability at least $1-1/\poly(d)$, we have that any local minimum $(\U, \V) \in \text{int}(\cW)$ of the objective \eqref{eq:objective} satisfies:
\begin{equation*}
\norm{\U\V^\top - \M^\star}_F\leq O\left(\sigma\sqrt{\frac{t\max\{t, d\} r\log n}{n}}\right).
\end{equation*}
\end{theorem}

\begin{proof} By \cref{lem:step2_mc_asym}, we know
$$
\Delta\Delta\trans:\H:\Delta\Delta\trans - 3(\N - \N^\star):\H:(\N - \N^\star) \le -0.1 \|\N-\N^\star\|_F^2.
$$
In order to use \cref{lem:asymmetricmain}, we bound the contribution from the noise term $Q$. Recall \eqref{eq:regularization}:
\begin{align*}
Q(\W)=&-\frac{4}{n}\sum_{i=1}^n (\la \M-\M^\star,\A_i\ra n_{i}) +\frac{2}{n}\sum_{i=1}^n(n_i)^2\\
\la \grad Q(\W), \Delta \ra=&-\frac{4}{n}\sum_{i=1}^n(\la \U\Delta_\V\trans + \Delta_\U\V\trans,\A_i\ra n_i)\\
\Delta : \hess Q(\W) :\Delta=&-\frac{8}{n}\sum_{i=1}^n (\la \Delta_\U\Delta_\V\trans,\A_i\ra n_i).
\end{align*}

Let $\B_i$ be the $(d_1+d_2)\times (d_1+d_2)$ matrix whose diagonal blocks are 0, and off diagonal blocks are equal to $\A_i$ and $\A_i^\top$ respectively. Then we have
\begin{align*}
\lefteqn{[\Delta : \hess Q(\W) : \Delta - 4 \la \grad Q(\W), \Delta \ra]}\\
&= -\frac{8}{n}\sum_{i=1}^n (\la \Delta_\U\Delta_\V\trans,\A_i\ra n_i) + \frac{16}{n}\sum_{i=1}^n(\la \U\Delta_\V\trans + \Delta_\U\V\trans,\A_i\ra n_i)\\
&=\frac{24}{n}\sum_{i=1}^n (\la \Delta_\U\Delta_\V\trans,\A_i\ra n_i) + \frac{16}{n}\sum_{i=1}^n(\la \U^\star \Delta_\V\trans + \Delta_\U(\V^\star)\trans,\A_i\ra n_i)
\end{align*}
Now we can use \cref{lem:mc_noise} again to bound the noise terms: 
\begin{align*}
|\frac{24}{n}\sum_{i=1}^n (\la \Delta_\U\Delta_\V\trans,\A_i\ra n_i)|& \le O\left(\sigma\sqrt{\frac{t\max\{t, d\} r\log n}{n}}\right)\sqrt{\fnorm{\Delta_\U\Delta_\V\trans}^2 + 0.001\sigstarr \fnorm{\Delta_\V}^2}\\
|\frac{16}{n}\sum_{i=1}^n(\la \U^\star \Delta_\V\trans + \Delta_\U(\V^\star)\trans,\A_i\ra n_i)| & \leq O\left(\sigma\sqrt{\frac{t\max\{t, d\} r\log n}{n}}\right)\norm{\U^\star\Delta_\V\trans +\Delta_\U(\V^{\star})\trans}_F.
\end{align*}
On the one hand, by \cref{lem:bound}, we have:
\begin{equation*}
\fnorm{\Delta_\U\Delta_\V\trans}^2 + 0.001\sigstarr \fnorm{\Delta_\V}^2
\le \fnorm{\Delta\Delta\trans}^2 + 0.001\sigstarr \fnorm{\Delta}^2 \le 3\fnorm{\N - \N^\star}^2.
\end{equation*}
On the other hand, again by \cref{lem:bound}, we have:
\begin{equation*}
\norm{\U^\star\Delta_\V\trans +\Delta_\U(\V^{\star})\trans}^2_F \le \norm{\W^\star \Delta\trans  +\Delta(\W^{\star})\trans}^2_F
= 2 [\fnorm{\W^\star \Delta\trans }^2 + \fnorm{\Delta\trans  \W^\star }^2] \le 10 \fnorm{\N - \N^\star}^2.
\end{equation*}
In sum, we have:
\begin{equation*}
[\Delta : \hess Q(\W) : \Delta - 4 \la \grad Q(\W), \Delta \ra] \le O\left(\sigma\sqrt{\frac{t\max\{t, d\} r\log n}{n}}\right)\fnorm{\N - \N^\star}.
\end{equation*}

Therefore, by Lemma \ref{lem:asymmetricmain}, the Hessian at $\Delta$ direction is equal to:
\begin{align*}
\Delta : \hess f(\W) :\Delta
\leq &-0.1\norm{\N - \N^\star}_F^2
+O(\sigma\sqrt{\frac{t\max\{t, d\} r\log n}{n}})\fnorm{\N - \N^\star}.
\end{align*}
When the point further satisfies the second-order optimality condition we have
\begin{align*}
\norm{\N-\N^\star}_F \leq O\left(\sigma\sqrt{\frac{t\max\{t, d\} r\log n}{n}}\right).
\end{align*}
In particular, $\M-\M^\star$ is a submatrix of $\N-\N^\star$, so $\norm{\M-\M^\star}_F\le O(\sigma\sqrt{\frac{t\max\{t, d\} r\log n}{n}})$.
\end{proof}

\subsection{Linear Algebra Lemmas} \label{sec:lemmas}

We collect together several useful linear algebra lemmas.

\begin{lemma}\label{lem:bound}
Given matrices $\W,\W^\star \in \mR^{d\times r}$, let $\N = \W\W^\top$ and $\N^\star = \W^\star(\W^\star)^\top$, and let $\Delta, \R^\star$ be defined as in Definition~\ref{def:delta}, and let $\tilde{\W}^\star = \W^\star \R^\star$then we have the followings properties:
\begin{enumerate}
\item $\W (\tilde{\W}^\star)\trans$ is a symmetric p.s.d. matrix;
\item $\|\Delta\Delta^\top\|_F^2 \le 2\|\N - \N^\star\|_F^2$;
\item $\sigstarr\|\Delta\|_F^2 \le \fnorm{\Delta(\tilde{\W}^\star)\trans}^2 \le \frac{1}{2(\sqrt{2}-1)}\|\N-\N^\star\|_F^2$.
\item $\fnorm{\Delta\trans \tilde{\W}^\star}^2 \le \|\N-\N^\star\|_F^2$
\end{enumerate}
\end{lemma} 

\begin{proof}
Statement 1 is in the proof of \citet[Lemma 6]{ge2017no}. Statement 2 is by \citet[Lemma 6]{ge2017no}.
Statement 3 \& 4 follow by Lemma \ref{lem:aux_deltalinear}.
\end{proof}

\begin{lemma} \label{lem:aux_deltalinear}
Let $\U$ and $\Y$ be $d \times r$ matrices such that $\U\trans\Y = \Y\trans \U$ is a p.s.d. matrix. Then,
\begin{align*}
\sigma_{\min}(\U\trans\U)\fnorm{\U - \Y}^2 \le \fnorm{(\U-\Y)\U\trans}^2 \le& \frac{1}{2(\sqrt{2}-1)}\fnorm{\U\U\trans - \Y\Y\trans}^2\\
\fnorm{(\U-\Y)\trans\U}^2 \le& \fnorm{\U\U\trans - \Y\Y\trans}.
\end{align*}
\end{lemma}
\begin{proof}
For the first statement, the left inequality is immediate, so we only need to prove right inequality.
To prove this, we let $\Delta = \U - \Y$, and expand:
\begin{align*}
\norm{\U\U\trans - \Y\Y\trans}_F^2  =& \norm{\U\Delta\trans + \Delta\U\trans - \Delta\Delta\trans}_F^2 \\
& = \tr(2\U\trans\U\Delta\trans\Delta + (\Delta\trans\Delta)^2 + 2(\U\trans\Delta)^2 - 4\U\trans\Delta\Delta\trans \Delta) \\
&= \tr((4 - 2\sqrt{2})\U\trans (\U - \Delta) \Delta\trans\Delta + (\Delta\trans\Delta -\sqrt{2}\U\trans\Delta )^2 + 2(\sqrt{2} - 1)\U\trans\U\Delta\trans\Delta ) \\
&\ge  \tr((4 - 2\sqrt{2})\U\trans\Y\Delta\trans\Delta + 2(\sqrt{2} - 1)\U\trans\U\Delta\trans\Delta) \ge 2(\sqrt{2} - 1)
\fnorm{\U\Delta\trans}^2.
\end{align*}
The last inequality follows since $\U\trans \Y$ is a p.s.d. matrix. For the second statement, again, we have:
\begin{align*}
\norm{\U\U\trans - \Y\Y\trans}_F^2  &= \norm{\U\Delta\trans + \Delta\U\trans - \Delta\Delta\trans}_F^2 \\
&= \tr(2\U\trans\U\Delta\trans\Delta + (\Delta\trans\Delta)^2 + 2(\U\trans\Delta)^2 - 4\U\trans\Delta\Delta\trans \Delta) \\
&=\tr(2\U\trans (\U - \Delta) \Delta\trans\Delta + (\Delta\trans\Delta -\U\trans\Delta )^2 + (\U\trans\Delta)^2 ) \\
&\ge  \tr(2\U\trans\Y\Delta\trans\Delta + (\U\trans\Delta)^2) \ge \fnorm{\U\trans \Delta}^2,
\end{align*}
where the last inequality follows since $\U\trans\Delta = \Delta\trans \U$.
\end{proof}

\subsection{Concentration Lemmas}

We need to show three concentration-style results for the landscape analysis. The first is an RIP condition for over matrices in the linear space $\cT = \{ \U_{\star} \X^\top + \Y \V_{\star}^\top | \X \in \mR^{t \times r}, \Y \in \mR^{d \times r} \}$ using matrix concentration. The second and third are coarse concentration results that exploit the rank $r$ structure of the underlying matrix $\M$ and are used in the two distinct regimes where the distance to optimality can be small or large. Also note that throughout we can assume a left-sided incoherence condition on the underlying matrix of the form $\max_{i \in [t]} \norm{\e_i^\top \U_{\star}}^2 \leq \frac{\mu r}{t}$ due to \cref{assump:task}.

We first present the RIP-style matrix concentration result which rests on an application of the matrix Bernstein inequality over a projected space. The proof has a similar flavor to results in \citet{recht2011simpler}. First we define a projection operator on the space of matrices as $\pt \Z = \pu \Z + \Z \pv - \pu \Z \pv$ where $\pu$ and $\pv$ are orthogonal projections onto the subspaces spanned by $U$ and $V$. While $\pu$ and $\pv$ are matrices, $\pt$ is a linear operator mapping matrices to matrices. Intuitively we wish to show that for all $\W \in \mR^{t \times d}$, that the observations matrices are approximately an isometry over the space of projected matrices w.h.p: $\sumton t(\langle \e_{t(i)} \x_i^\top, P_{\cT} \W \rangle^2 \approx \fnorm{P_{\cT} \W}^2 = \fnorm{\W}^2$. Explicitly, we define the action of the operator $\C_i = \A_i \A_i^\top$ where $\A_i = \sqrt{t} \x_i \e_j^\top$ as $\C_i(\M) = t \x_i \e_j^\top \langle \e_j \x_i^\top, \M \rangle$. 

We record a useful fact we will use in the sequel:
\begin{align*}
	\lefteqn{ \sqrt{t} \pt(\x_i \e_j^\top) = \pu \e_i \x_j^\top + \e_i (\pv \x_j)^\top - (\pu \e_i)(\pv \x_j)^\top \implies} \\
	& \fnorm{\pt (\x_i \e_j^\top}^2 = \langle \pt (\x_i \e_j^\top), \x_i \e_j \rangle = \norm{\pu \e_i}^2 \norm{\x_i}^2 + \norm{\e_i}^2 \norm{\pv \x_i}^2 - \norm{\pu \e_i}^2 \norm{\pv \x_j}^2 \leq \\
	& \norm{\pu \e_i}^2 \norm{\x_i}^2 + \norm{\pv \x_i}^2,
\end{align*}
where the last inequality holds almost surely.

We now present the proof of the RIP-style concentration result.
\begin{lemma}\label{lem:rip_conc}
	Under \cref{assump:design,,assump:task} and the uniform task sampling model above,
	\begin{align*}
	\norm{\frac{1}{n}\sum_{i=1}^n P_{\cT} \A_i \A_i\trans P_{\cT} - P_{\cT}} \leq (\log(ndt)) \cdot O \left(\sqrt{\frac{\mu d r^2 + tr^2}{n}} + \frac{(\mu d r + r t) \log(tdn)}{n} \right),
	\end{align*}
	with probability at least $1-O(n^{-100})$, where $\mu = O(\bkappa)$.
\end{lemma}
\begin{proof}
Note that under the task assumption, \cref{lem:diverse_to_incoh} diversity implies incoherence of the matrix $\U^{\star}$ with incoherence parameter $\mu = O(\bkappa)$. 
First, note $\mE[\C_i(\M)] = \M$ so $\mE[\sumton P_{\cT} \A_i P_{\cT} - P_{\cT}] = 0$.
To apply the truncated version of the matrix Bernstein inequality from \cref{lem:trunc_mb} we first compute a bound on the norms of each $\C_i$ to set the truncation level $R$. Note that $\norm{\pt \A_i \A_i \pt} = \fnorm{\pt(\x_i \e_j^\top)}^2 \leq t \cdot O((\frac{\mu r}{t} \norm{\x_i}^2 + \norm{\pv \x_i}^2)$) using the fact the operator $\A_i$ is rank-one along with the \cref{lem:diverse_to_incoh} which shows task diversity implies incoherence with incoherence parameter $\bkappa$. Now exploiting \cref{lem:lensg} we have that $\norm{\x_i}^2 \leq O(d+\max(\sqrt{d \log(1/\delta)}, \log(1/\delta))) = O(d+\sqrt{d} \log(1/\delta))$ and $\norm{P_V \x_i}^2 \leq O(r + \sqrt{r} \log(1/\delta))$ with probability at least $1-2\delta$ using sub-exponential tail bounds and a union bound\footnote{Note that by definition the orthogonal projection of a $d$-dimensional subgaussian random vector onto an $r$-dimensional subspace is an $r$-dimensional subgaussian random vector.}. Hence $\norm{P_{\cT} \A_i P_{\cT}}  \leq O(\mu r d +   \mu r \sqrt{d}\log(1/\delta)) + tr + t \sqrt{r} \log(1/\delta)) = O(\mu r d + tr + (\mu r \sqrt{d} +t \sqrt{r}) \log(1/\delta))$.

We can rearrange this statement to conclude that $\Pr[\norm{P_{\cT} \A_i \A_i^\top P_{\cT}}  \geq c_1 (\mu r d + tr) + x] \leq  \exp(-c_2 (\frac{x}{\mu r \sqrt{d} + t \sqrt{r}}))$ for some $c_1, c_2$. Define a truncation level $R = c_1 (\mu r d +rt) + K (\mu r \sqrt{d} + t \sqrt{r})$ for some $K$ to be chosen later. We can use the aforementioned tail bound to control $\norm{\mE[Z_i]-\mE[Z_i']} \leq \mE[Z_i \Ind[\norm{Z_i} \geq R]] \leq \int_{K (\mu r \sqrt{d} + t \sqrt{r})}^{\infty} \exp(-c_2 (\frac{x}{\mu r \sqrt{d} + t \sqrt{r}})) \leq O(\exp(-c K)  (\mu r \sqrt{d} + t \sqrt{r})) = \Delta$.

Now we consider the task of bounding the matrix variance term. The calculation is somewhat tedious but straightforward under our assumptions. We make use of the standard result that for two matrices $\X$ and $\Y$ that $\norm{\X-\Y} \leq \max(\norm{\X}, \norm{\Y})$.

It suffices to bound the operator norm $\norm{\E[\fnorm{\pt \A_i}^2 \pt \A_i (\pt \A_i)^\top}$. Using the calculation from the prequel and carefully cancelling terms we can see that, 
\begin{align*}
	& \norm{\mE[\fnorm{\pt \A_i}^2 \pt \A_i (\pt \A_i)^\top]} \leq  t^2 \Vert \mE[(\norm{\pu \e_i}^2 \norm{\x_i}^2 + \norm{\pv \x_i}^2 - \norm{\pu \e_i}^2 \norm{\pv \x_j}^2) \cdot \\
	& (\norm{\pu \e_i}^2 \x_i \x_i^\top + \norm{\pv \x}^2 \e_i \e_i^\top - \norm{\pv \x_i}^2 \pu \e_i (\pu \e_i)^\top] \Vert \\
	& \leq t^2 O(\norm{\mE[(\norm{\pu \e_i}^2 \norm{\x_i}^2 \norm{\pu \e_i}^2 \x_i \x_i^\top + (\norm{\pu \e_i}^2 \norm{\x_i}^2 \norm{\pv \x_i}^2 \e_i \e_i^\top}] + \norm{\pv \x_i}^4 \e_i \e_i^\top).
\end{align*}
 We show how to calculate these leading terms as the subleading terms can be shown to be lower-order by identical calculations. First note using the fact that $\mE[\norm{\pu \e_i}^2] \leq \frac{r}{t} \leq 1$, since $t \geq r$ by the task diversity assumption. Then $t^2 \cdot \norm{\mE[(\norm{\pu \e_i}^2 \norm{\x_i}^2 \norm{\pu \e_i}^2 \x_i \x_i^\top} \leq \norm{\mu r t \mE[\norm{\x_i}^2 \x_i \x_i^\top] \mE[\norm{\pu \e_i}^2]} \leq O(\mu r^2 d)$ appealing to the fact $\mE[\norm{\x}^2 \x \x^\top] \preceq O(\I_d)$ by \cref{lem:l4l2norm}. 
 
 Similarly, we have that, $t^2 \cdot \norm{\mE[\norm{\pu \e_i}^2 \norm{\x_i}^2 \norm{\pv \x_i}^2 \e_i \e_i^\top]} \leq \mu r \mE[\norm{\x_i}^2 \norm{\pv \x_i}^2] \leq \mu r^2 d$ using incoherence and by \cref{lem:l4l2norm}. Finally, we have that $t^2 \cdot O(\norm{\mE[ \norm{\pv \x_i}^4 \e_i \e_i^\top}) \leq O(t r^2)$. Hence we have that $\sigma^2 = n \cdot O(\mu r^2 d + t r^2)$.

Finally, we can assemble the previous two computations to conclude the result with appropriate choices of $R$ (parametrized through $K$) and $x$ by combining with \cref{lem:trunc_mb}. Let us choose $K = \frac{c_3}{c} (\log(n) + \log(d) + \log(t))$ for some sufficiently large $c_3$. In this case, we can choose $c_3$ such that $\Delta \leq O(\frac{ \mu r \sqrt{d} + t \sqrt{r}}{n^{10} d^{10} t^{10}}) \leq O(\frac{\mu}{n^{10} d^{8}})$. Similarly, our choice of truncation level becomes $R = O( \mu r d + tr + (\log n + \log d + \log t) (\mu r \sqrt{d} + t \sqrt{r})$. At this point we now choose $x = c_4 (\log n + \log d + \log t) \max(\sigma/\sqrt{n}, R/n)$ for sufficiently large $c_4$. For large enough $c_4$ we can guarantee that $x \geq 2 \Delta \implies x - \Delta \geq \frac{x}{2}$.

Hence combining these results together and applying \cref{lem:trunc_mb} we can provide the following upper bound on the desired quantity:
\begin{align*}
	& \Pr[\norm{\sumton P_{\cT} \A_i P_{\cT} - P_{\cT}} \geq x] \leq \\
	& O(d \exp(-c \cdot c_4 (\log n + \log d + \log t)) + O(n \sqrt{K} \exp(-c_3 (\log n + \log d + \log t)) \leq O((ndt)^{-100})
\end{align*}
by taking $c_3$ and $c_4$ sufficiently large, with $x = O((\log (ndt)) \left(\sqrt{\frac{\mu d r^2 + tr^2}{n}} +  \frac{(\mu r d + tr) + (\mu r \sqrt{d} + t \sqrt{r}) (\log(ndt)))}{n} \right)$.
\end{proof}

\begin{lemma}\label{lem:coarse_conc}
	Let the covariates $\x_i$ satisfy the design conditions in \cref{assump:design} in the uniform task sampling model. Then for all matrices $\M$ matrices that are of rank $2r$, we have uniformly that,
	\begin{align*}
	 \abs{ \frac{1}{n}\sum_{i=1}^n\la \A_i, \M\ra^2 - \fnorm{\M}^2} \leq O \left(\sqrt{\frac{\max(t, d)r}{n}} \cdot \sqrt{t} \max_i \norm{\e_i^\top \M} \fnorm{\M}  +\frac{\max(t, d)r}{n} \cdot t \max_i \norm{\e_i^\top \M}^2  \right).
	 \end{align*}
	 with probability at least $1-(3000r )^{-10\max(t,d)r}$.
\end{lemma}
\begin{proof}
	Note that by rescaling it suffices to restrict attention to matrices $\M$ that are of rank $2r$ and have Frobenius norm $1$ (a set which we denote $\Gamma$). Applying \cref{lem:y_subexp}, we have that,
	 \begin{align*}
	 	\abs{\sumton t(\e_{t(i)}^\top \M \x_i)^2 - \norm{\M}_F^2} \leq O \left(\frac{1}{\sqrt{n}} \sqrt{\log(\frac{1}{\delta}}) + \frac{1}{n} \log(\frac{1}{\delta}) \right),
	 \end{align*}
	 for any fixed $\M \in \Gamma$ with probability at least $1-\delta$.
	 Now using \ref{prop:rankrcover} with $\epsilon = \frac{1}{1000}$ have that the set $\Gamma$ admits a cover $K$ of size at most $|K| = (3000r)^{(t+d+1)r}$. Now by choosing $\delta = (3000)^{-c (t+d+1)r}$ for a sufficiently large constant $c$ we can ensure that, 
	 \begin{align*}
	 	\abs{\sumton t(\e_{t(i)}^\top \M_j \x_i)^2 - \norm{\M_j}_F^2} \leq O \left(\frac{1}{\sqrt{n}} \sqrt{(\max(t, d) r} + \frac{1}{n} \max(t, d)r \right) \quad \forall \M_j \in K,
	 \end{align*}
	 with probability at least $1-(3000r)^{-10\max(t,d)r}$ using a union bound. Now a straightforward Lipschitz continuity argument shows that since any $\M \in \Gamma$ can be written as $\M = \M_i + \epsilon a_i$ for $\M_i \in K$ and another $a_i \in \Gamma$, then \begin{align*}
	     \sup_{\M \in \Gamma} \abs{\sumton t(\e_{t(i)}^\top \M \x_i)^2 - \norm{\M}_F^2}  \leq 2 (\sup_{\M_j \in K} \abs{\sumton t(\e_{t(i)}^\top \M_j \x_i)^2 - \norm{\M_j}_F^2}),
	 \end{align*}
	 and hence the conclusion follows. Rescaling the result by $\fnorm
	 {\M}^2$ finishes the result.
\end{proof}

\begin{lemma}
	Let the covariates $\x_i$ satisfy the design condition in \cref{assump:design} in the uniform task sampling model. Then if $Y_i = t (\e_{t(i)}^\top \A \x_i)^2 - \fnorm{\A}^2$, $Y_i$ is a sub-exponential random variable, and
	 \begin{align*}
	 	\abs{\sumton t(\e_{t(i)}^\top \M \x_i)^2 - \norm{\M}_F^2} \leq O \left( \frac{\sqrt{t} \max_i \norm{\e_i^\top \M}_2 \fnorm{\M}}{\sqrt{n}} \sqrt{\log(\frac{1}{\delta})} + \frac{t \max_i \norm{\e_i^\top \M}^2}{n} \log(\frac{1}{\delta}) \right),
	 \end{align*}
	 for any fixed $\M$ with probability at least $1-\delta$.
	\label{lem:y_subexp}
\end{lemma}
\begin{proof}
	First note that under our assumptions $Y_i$, $\mE[t(\e_{j}^\top \A \x_i)^2] = \norm{\A}_F^2$. To establish the result, we show the Bernstein condition holds with appropriate parameters. To do so, we bound for $k \geq 1$,
	\begin{align*}
		& \abs{\mE[Y_i^k]} \leq t^{k} 2^{2k} \mE[(\e_j^\top \M_j \x_i)^{2k}] = t^{k} 2^{2k} \cdot \mE[\norm{\e_j^\top \M_j}^{2k}] C^{2k} k! \leq (C')^{4k} k! \cdot \mE[(t\norm{\e_j^\top \M}^2)^{k-1} \cdot (t \norm{\e_j^\top \M}^2)] \\
		& \leq (C'')^k k! (t \nu^2)^{k-2} (t\nu^2 \fnorm{\A}^2) = \frac{1}{2} k! (\underbrace{C'' t \nu^2}_{b})^{k-2} \cdot(\underbrace{C''^2 t\nu^2 \fnorm{\M}^2}_{\sigma^2}),
	\end{align*}
	by introducing an independent copy of $Y$, using Jensen's inequality, and the inequality $(\frac{a+b}{2})^k \leq 2^{k-1} (a^k  + b^k)$ in the first inequality, and the sub-gaussian moment bound $\mE[Z^{2k}] \leq 2k \Gamma(k) C^{2k} \leq k! C^{2k}$ for universal constant $C$ which holds under our design assumptions. Hence directly applying the Bernstein inequality (see \citet[Proposition 2.9]{wainwright2019high} shows that,
	\begin{align*}
		\mE[e^{\lambda \cdot Y_i}] \leq e^{\lambda^2 (\sqrt{2} \sigma)^2/2} \quad \forall \abs{\lambda} \leq \frac{1}{2b}.
	\end{align*}
	Hence, using a standard sub-exponential tail bound we conclude that, 
	 \begin{align*}
	 	& \abs{\sumton t(\e_{t(i)}^\top \M \x_i)^2 - \norm{\M}_F^2} \leq O \left(\frac{\sigma}{\sqrt{n}} \sqrt{\log(\frac{1}{\delta}}) + \frac{b}{n} \log(\frac{1}{\delta}) \right) = \\
	 	& O \left( \frac{\sqrt{t} \max_i \norm{\e_i^\top \M}_2 \fnorm{\M}}{\sqrt{n}} \sqrt{\log(\frac{1}{\delta})} + \frac{t \max_i \norm{\e_i^\top \M}^2}{n} \log(\frac{1}{\delta}) \right),
	 \end{align*}
	 for any fixed $\A \in \Gamma$ with probability at least $1-\delta$.
\end{proof}

We now restate a simple covering lemma for rank-$O(r)$ matrices from \citet{candes2010tight}.
\begin{lemma}[Lemma 3.1 from \citet{candes2010tight}]
	Let $\Gamma$ be the set of matrices $\M \in \mR^{t \times d}$ that are of rank at most $r$ and have Frobenius norm equal to $1$. Then for any $\epsilon < 1$, there exists an $\epsilon$-net covering of $\Gamma$ in the Frobenius norm, $S$, which has cardinality at most $(\frac{9}{\epsilon})^{(t+d+1)r}$.
	\label{prop:rankrcover}
\end{lemma}

We now state a central lemma which combines the previous concentration arguments into a single condition we use in the landscape analysis.

\begin{lemma}\label{lem:uniform_concentration}
Let \cref{assump:design,,assump:task} hold in the uniform task sampling model. 
When number of samples is greater than $n \ge C \polylog(d,n,t) \mu^2 r^4 \max\{t, d\}  (\cn^\star)^4$ with large-enough constant $C$, with at least $1-1/\poly(d)$ probability, we have following holds for all $(\U, \V) \in \cW$ simultanously:
\begin{align}
\frac{1}{n}\sum_{i=1}^n\la \U^\star \Delta_\V\trans + \Delta_\U(\V^\star)\trans,\A_i\ra^2  \in & (1\pm 0.001) \fnorm{\U^\star \Delta_\V\trans + \Delta_\U(\V^\star)\trans}^2\label{eq:DWstar_concen}\\
\frac{1}{n}\sum_{i=1}^n\la \A_i, \Delta_\U\Delta_\V\trans\ra^2 \le & \fnorm{\Delta_\U\Delta_\V \trans}^2 +  0.001 \sigstarr \fnorm{\Delta_\V}^2  \label{eq:DD_concen}\\
\frac{1}{n}\sum_{i=1}^n\la \A_i, \M - \M^\star\ra^2 \ge& \fnorm{\M-\M^\star}^2 - (\sigstarr)^2 /10^{6},\label{eq:M_concen}
\end{align}
where $\M = \U\V\trans$ and $\Delta_\U, \Delta_\V$ are defined as in \cref{def:delta}. Here $\mu = O(\bkappa)$.
\end{lemma}

\begin{proof}
This result follows immediately by applying Lemma \ref{lem:rip_conc} to the first statement and Lemma \ref{lem:coarse_conc} to the following two statements using the definition of the incoherence ball $\cW$.
\end{proof}

\begin{lemma}\label{lem:mc_noise}
Suppose the set of matrices $\A_1,\A_2,...,\A_n$ satisfy the event in \cref{lem:uniform_concentration}, let $n_1,n_2,...,n_m$ be i.i.d. sub-gaussian random variables with variance parameter $t\sigma^2$, then with high probability for any $(\U, \V) \in \cW$, we have
\begin{align*}
|\frac{1}{n}\sum_{i=1}^n (\la \Delta_\U\Delta_\V\trans,\A_i\ra n_i)|& \le O(\sigma\sqrt{\frac{t\max\{t, d\} r\log n}{n}})\sqrt{\fnorm{\Delta_\U\Delta_\V\trans}^2 + 0.001\sigstarr \fnorm{\Delta_\V}^2}\\
|\frac{1}{n}\sum_{i=1}^n(\la \U^\star \Delta_\V\trans + \Delta_\U(\V^\star)\trans,\A_i\ra n_i)| & \leq O(\sigma\sqrt{\frac{t\max\{t, d\} r\log n}{n}})\norm{\U^\star\Delta_\V\trans +\Delta_\U(\V^{\star})\trans}_F
\end{align*}
for $n \gtrsim \polylog(d)$.
\end{lemma}

\begin{proof}
Note since the left hand side of the expressions are linear in the matrices we can normalize to those of Frobenius norm 1. The proof of both statements is identical so we simply prove the second. 

Define $\delta = \fnorm{\U^\star \Delta_\V\trans + \Delta_\U(\V^\star)\trans}$ and $\M=\U^\star \Delta_\V\trans + \Delta_\U(\V^\star)\trans$ for convenience, which can be thought of as arbitrary rank-$r$ matrices. Then let $S$ be an $\epsilon$-net for all rank-$r$ matrices with Frobenius norm 1; by \cref{prop:rankrcover} we have that $\log |S| \leq O(\max(t,d) r \log(\frac{1}{\epsilon}))$. We set $\epsilon=\frac{1}{n^3}$ so $\log(\frac{1}{\epsilon}) = O(\log n)$. Now for any matrix $\M \in S$ we have that $\frac{1}{n} \langle \A_i, \M \rangle$ is a sub-gaussian random variable with variance parameter at most $t \sigma^2 \frac{\delta^2}{n}$. Thus, using a sub-gaussian tail bound along with a union bound over the net shows that uniformly over the $\M \in S$,
\begin{align*}
	|\frac{1}{n}\sum_{i=1}^n \langle \M,\A_i\ra n_i)| & \leq O\left(\sigma \delta \sqrt{\frac{t\max\{t, d\} r\log n}{n}}\right),
\end{align*}
with probability at least $1-\frac{1}{\poly(d)}$. We now show how to lift to the set of all $\M$. Note that with probability at least $1-e^{-\Omega(n)}$ that $\norm{\n} = O(\sqrt{t} \sigma \sqrt{n})$ by a sub-gaussian tail bound (see for example \cref{lem:lensg}). Let $\M$ be an arbitrary element, and $\M'$ its closest element in the cover; then we have that $\z_i = \langle \A_i, \M-\M' \rangle \leq \frac{\delta}{n^2}$ using the precondition on $\A_i$. 
Combining and using a union bound then shows that,
\begin{align*}
	& \abs{\sumton n_i \langle \A_i, \M \rangle } \leq  \abs{\sumton n_i \langle \A_i, \M' \rangle } + \abs{\sumton n_i \langle \A_i, \M-\M' \rangle } \leq \\
	& O\left(\sigma \delta \sqrt{\frac{t\max\{t, d\} r\log n}{n}}\right) + \frac{\sqrt{t} \sigma \delta}{\sqrt{n}}\leq O\left(\sigma \delta \sqrt{\frac{t \max(t, d) r \log n}{n}}\right).
\end{align*}
Rescaling and recalling the definition of $\delta$ gives the result. 
\end{proof}

\subsection{Task Diversity for the Landscape Analysis}
Here we collect several useful results for interpreting the results of the landscape analysis. Throughout this section we use the notation $\U \in \mR^{t \times r}$ and $\V \in \mR^{d \times r}$.

The first result allow us to convert a guarantee on error in Frobenius norm to a guarantee in angular distance, assuming an appropriate diversity condition on $\U$.

\begin{lemma}
Suppose $\V$ and $\hat{\V}$ are orthonormal projection matrices, that is $\V\trans \V = \I_r$, and $\hat{\V}\trans \hat{\V} = \I_r$. Then, for any $\epsilon>0$, if $\fnorm{\hat{\U}\hat{\V}\trans - \U\V\trans}^2 \le \epsilon$ for some $\hat{\U}$ and $\U$, then:
\begin{equation*}
\mathrm{dist}^2(\V, \hat{\V}) \le \frac{\epsilon}{\nu t},
\end{equation*}
where $\nu = \sigma_r(\U\trans \U)/t$.
\label{lem:frob_to_angle}
\end{lemma}

Here the distance function is the sine function of the principal angle; i.e.
\begin{equation*}
\mathrm{dist}(\V, \hat{\V}) := \norm{\V\trans \hat{\V}_\perp},
\end{equation*}
and $\nu = \sigma_r(\U\trans \U)/t$ represents an analog of the task diversity matrix.

\begin{proof}
Define the function $f(\tilde{\U}) = \fnorm{\tilde{\U}\hat{\V}\trans - \U\V\trans}^2$.
The precondition of the theorem states that there exists $\hat{\U}$ so that $\fnorm{\hat{\U}\hat{\V}\trans - \U\V\trans}^2 \le \epsilon$. This clearly implies the following:
\begin{equation} \label{eq:minimal_solution}
\min_{\tilde{\U}} f(\tilde{\U}) \le \epsilon.
\end{equation}
Setting the gradient $\dd f/ \dd \tilde{\U} = 0$, we have the minimizer $\tilde{\U}^\star$ satisfies:
\begin{equation*}
(\tilde{\U}^\star \hat{\V}\trans - \U \V\trans)  \hat{\V} = 0,
\end{equation*}
which gives:
\begin{equation*}
\tilde{\U}^\star = \U \V\trans \hat{\V}.
\end{equation*}
Plugging this back to Eq.~\eqref{eq:minimal_solution} gives:
\begin{equation*}
\fnorm{\U\V\trans(\hat{\V}\hat{\V}\trans - \I)}^2 \le \epsilon.
\end{equation*}
Finally, we have:
\begin{align*}
\fnorm{\U\V\trans(\hat{\V}\hat{\V}\trans - \I)}^2 = &\fnorm{\U\V\trans\hat{\V}_\perp\hat{\V}_\perp\trans}^2
= \fnorm{\U\V\trans\hat{\V}_\perp}^2 = \tr(\U\trans\U\V\trans\hat{\V}_\perp\hat{\V}_\perp\trans \V)\\
\ge & \sigma_{r}(\U\trans\U) \fnorm{\V\trans\hat{\V}_\perp}^2
\ge \sigma_{r}(\U\trans\U) \norm{\V\trans\hat{\V}_\perp}^2.
\end{align*}
The second last inequality follow since  for any p.s.d.\ matrices $\A$ and $\B$, we have $\tr(\A\B) \ge \sigma_{\min}(\A) \tr(\B)$.
This concludes the proof.
\end{proof}

For the following let $\A = (\balpha_1, \cdots, \balpha_t)\trans \in \mR^{t\times r}$ and denote the SVD of $\A = \U \mSigma \V\trans$. Next we remark that our assumptions on task diversity and normalization implicit in the matrix $\A$ are sufficient to actually imply an incoherence condition on $\U$ (which is used in the matrix sensing/completion style analysis). 
\begin{lemma}
If $\mu = \frac{1}{r \sigma_r(\A^\top \A/t)}$ and $\max_{i\in[t]} \norm{\balpha_i}^2 \le C$, then we have:
\begin{equation*}
\max_{i\in[t]} \norm{\e_i\trans \U}^2 \le \frac{C\mu r}{ t}.
\end{equation*}
\label{lem:diverse_to_incoh}
\end{lemma}
\begin{proof}
Since $\max_{i\in[t]} \norm{\balpha_i}^2 \le C $, we have, for any $i\in [t]$
\begin{equation*}
C \ge \norm{\balpha_i}^2 =\norm{\e_i\trans \A}^2 = \norm{\e_i\trans \U\mSigma}^2
\ge \norm{\e_i\trans \U}^2 \sigma_{\min}^2(\mSigma) = \norm{\e_i\trans \U}^2 \sigma_{r}(\A\trans \A)
= (t/\mu r) \norm{\e_i\trans \U}^2,
\end{equation*}
which finishes the proof. 
\end{proof}
Note in the context of \cref{assump:task} the incoherence parameter corresponds to the parameter $\bkappa \leq \kappa$ since under our normalization $\tr(\A^\top \A/t) = \Theta(1)$. Further to quickly verify the incoherence ball contains the true parameters it is important to recall the scale difference $\M^{\star}$ and $\A^\top \A/t$ by a factor of $\sqrt{t}$.
\section{Proofs for \cref{sec:new_task}}
Assuming we have obtained an estimate of the column space or feature set $\Bone$ for the initial set of tasks, such that $\norm{\Bone^\top \B} \leq \delta$, we now analyze the performance of the plug-in estimator (which explicitly uses the estimate $\Bone$ in lieu of the unknown $\B$) on a new task. Recall we define the estimator for the new tasks by a projected linear regression estimator: $\halpha =\argmin_{\balpha} \norm{\y-\X \Bone \balpha}^2 \implies \halpha = (\Bone^\top \X^\top \X \Bone)^{-1} \Bone^\top \X^\top \y$.

Analyzing the performance of this estimator requires first showing that the low-dimensional empirical covariance and empirical correlation concentrate in $\tlO(r)$ samples and performing an error decomposition to compute the bias resulting from using $\Bone$ in lieu of $\B$ as the feature representation. We measure the performance the estimator with respect to its estimation error with respect to the underlying parameter $\B \balpha_0$; in particular, we use $\norm{\Bone \halpha-\B \balpha_0}^2$. Note that our analysis can accommodate covariates $\x_i$ generated from non-isotropic \textit{non-Gaussian} distributions. In fact the only condition we require on the design is that the covariates are sub-gaussian random vectors in the following sense.
\begin{assumption}
	Each covariate vector $\x_i$ is mean-zero, satisfies $\mE[\x \x^\top]=\mSigma$ such that $\sigma_{\max}(\mSigma) \leq \Cmax$ and $\sigma_{\min}(\mSigma) \geq \Cmin > 0$ and is $\mSigma$-sub-gaussian, in the sense that $
     \mE[\exp(\v^\top \x_i)] \leq \exp \left( \frac{\Vert \mSigma^{1/2} \v \Vert^2}{2} \right)$. Moreover, the additive noise $\epsilon_i$ is i.i.d. sub-gaussian with variance parameter $1$ and is independent of $\x_i$. 
	\label{assump:sg_design}
\end{assumption}
	In the context of the previous assumption we also define the conditioning number as $\cond = \Cmax/\Cmin$. Note that \cref{assump:design} immediately implies \cref{assump:sg_design}.

	Throughout this section we will let $\Bone$ and $\Btwo$ be orthonormal projection matrices spanning orthogonal subspaces which are rank $r$ and rank $d-r$ respectively---so that $\ran(\Bone) \oplus \ran(\Btwo) = \mR^d$.

\begin{proof}[Proof of \cref{thm:lr_transfer}]
	To begin we use the definition of 
	\begin{align*}
	\halpha = (\Bone^\top \X^\top \X \Bone)^{-1} \Bone \X^\top \y = (\Bone^\top \X^\top \X \Bone)^{-1} \Bone \X^\top \X \B \balpha_0 + (\Bone^\top \X^\top \X \Bone)^{-1} \Bone \X^\top \bepsilon
	\end{align*}
	to decompose the error as,
	\begin{align*}
		& (\Bone\halpha-\B \balpha_0) = \Bone (\Bone^\top \X^\top \X \Bone)^{-1} \Bone \X^\top \X \B \balpha_0 - \B \balpha_0 + \Bone(\Bone^\top \X^\top \X \Bone)^{-1} \Bone^\top \X^\top \bepsilon.
	\end{align*}
	Now squaring both sides of the equation gives,
	so \begin{align*}
	& \norm{\Bone\halpha-\B \balpha_0}^2 \leq 2(\norm{\Bone (\Bone^\top \X^\top \X \Bone)^{-1} \Bone \X^\top \X \B \balpha_0 - \B \balpha_0}^2+\norm{\Bone(\Bone^\top \X^\top \X \Bone)^{-1} \Bone^\top \X^\top \bepsilon}^2).
	\end{align*}
	The first bias term can be bounded by \cref{lem:test_bias}, while the the variance term can be bounded by \cref{lem:test_var} . Combining the results and using a union bound gives the result.
\end{proof}

We now present the lemmas which allow us to bound the variance terms in the aforementioned error decomposition. For the following two results we also track the conditioning dependence with respect $\Cmin$ and $\Cmax$. We first control the term arising from the projection of the additive noise onto the empirical covariance matrix.
\begin{lemma}
	Let the sequence of $n$ i.i.d. covariates $\x_i$ and $n$ i.i.d. additive noise variables $\epsilon_i$ satisfy \cref{assump:sg_design}. Then if $n \gtrsim \cond^2 r \log n$,
	\begin{align*}
		\norm{\Bone(\Bone^\top \X^\top \X \Bone)^{-1} \Bone^\top \X^\top \bepsilon}^2 \leq O \left( \frac{r \log n}{\Cmin n} \right),
	\end{align*}
	with probability at least $1-O(n^{-100})$.
	\label{lem:test_bias}
\end{lemma}
\begin{proof}
Since $\norm{\Bone(\Bone^\top \X^\top \X \Bone)^{-1} \Bone^\top \X^\top \bepsilon}^2 \leq \norm{(\Bone^\top \X^\top \X \Bone)^{-1} \Bone^\top \X^\top \bepsilon}^2$, it suffices to bound the latter term. Consider $\bepsilon^\top \underbrace{\frac{1}{n} \frac{\X \Bone}{\sqrt{n}} (\Bone^\top \frac{\X^\top \X}{n} \Bone)^{-2} \frac{\Bone^\top \X^\top}{\sqrt{n}}}_{\A} \bepsilon$. So applying the Hanson-Wright inequality \citep[Theorem 6.2.1]{vershynin2018high} (conditionally on $\X$) to conclude that $\Pr[\abs{\bepsilon^\top \A \bepsilon - \E[\bepsilon^\top \A \bepsilon]} \geq t] \leq 2 \exp(-c \min(\frac{t^2}{\norm{\A}_F^2}, \frac{t}{\norm{\A}}))$. Hence $\bepsilon^\top \A \bepsilon \leq \E[\bepsilon^\top \A \bepsilon] + O(\norm{\A}_F \sqrt{\log(2/\delta_1)})+ O(\norm{\A} \log(2/\delta_1))$ with probability at least $1-\delta_1$.

	 Now using cyclicity of the trace we have that $\mE[\bepsilon^\top \A \bepsilon] = \frac{1}{n} \tr[(\Bone^\top \frac{\X^\top \X}{n} \Bone)^{-1}]$. Similarly $\norm{\A} = \frac{1}{n}\norm{(\Bone^\top \frac{\X^\top \X}{n} \Bone)^{-1}} = \frac{1}{n}\norm{(\E + \Bone^\top \mSigma \Bone)^{-1}}$. Applying \cref{lem:two_subspace_conc} to the matrix $\E$ with $\delta=n^{-200}$ and assuming $n \gtrsim \cond^2 r \log(1/\delta) \gtrsim \cond^2 r \log n$ shows that $\norm{(\Bone^\top \mSigma \Bone)^{-1} \E} \leq \frac{1}{4}$. Also note that on this event and this regime of sufficiently large $n$, this concentration result shows that $\sigma_{\min}(\Bone^\top \frac{\X^\top \X}{n} \Bone) > \Cmin/2$,  so the matrix is invertible. Hence an application of \cref{lem:mat_pert} shows that $\norm{\A} \leq \frac{1}{n} (\frac{1}{\Cmin} \cdot(1+\cond  \sqrt{\frac{r \log n}{n}} )) \leq O(\frac{1}{\Cmin n})$. Similarly since $\frac{\X \Bone}{\sqrt{n}}$ is rank $r$ and invertible on this event, it follows $\fnorm{\A} \leq \sqrt{r} \norm{\A}\leq O(\frac{\sqrt{r}}{\Cmin n})$ and that $\frac{1}{n} \tr[(\Bone^\top \frac{\X^\top \X}{n} \Bone)^{-1}] \leq \frac{r}{\Cmin n}$.

	Hence taking $\delta_1 = n^{-200}$, and using the union bound, we conclude that $\bepsilon^\top \A \bepsilon \leq \frac{1}{\Cmin} \cdot O(\frac{r}{n}) + O(\frac{\sqrt{r \log n}}{n}) + O(\frac{\log n}{n}) \leq O(\frac{r \log n}{\Cmin n})$ with probability at least $1-O(n^{-100})$.
\end{proof}

We now control the error term which arises both from the variance in the random design matrix $\X$ and the bias due to mismatch between $\Bone$ and $\B$.

\begin{lemma}
	Let the sequence of $n$ i.i.d. covariates $\x_i$ satisfy the design assumptions in \cref{assump:sg_design}, and assume $\sin(\Bone, \B) \leq \delta \leq 1$. Then if $n \gtrsim \cond^2 r \log n$,
	\begin{align*}
		\norm{\Bone (\Bone^\top \X^\top \X \Bone)^{-1} \Bone \X^\top \X \B \balpha_0 - \B \balpha_0}^2 \leq O( \norm{\balpha_0}^2 \cond^2 \delta^2 ),
	\end{align*}
	with probability at least $1-O(n^{-100})$.
	\label{lem:test_var}
\end{lemma}
\begin{proof}
To control this term we first insert a copy of the identity $\I_d = \Bone \Bone^\top+\Btwo \Btwo^\top$ to allow the variance term in the design cancel appropriately in the span of $\Bone$; formally, 
	\begin{align*}
		& \Bone (\Bone^\top \X^\top \X \Bone)^{-1} \Bone \X^\top \X \B \balpha_0 -\B \balpha_0 = \Bone (\Bone^\top \X^\top \X \Bone)^{-1} \Bone \X^\top \X (\Bone \Bone^\top + \Btwo \Btwo^\top) \B \balpha_0 - \B \balpha_0 = \\
		& (\Bone \Bone^\top - \I) \B \balpha_0 + \Bone (\Bone^\top \X^\top \X \Bone)^{-1} \Bone \X^\top \X \Btwo \Btwo^\top \B \balpha_0 = \\
		& \Btwo \Btwo^\top \B \balpha_0 + \Bone (\Bone^\top \X^\top \X \Bone)^{-1} \Bone \X^\top \X \Btwo \Btwo^\top \B \balpha_0 \implies \\
		& \norm{\Bone (\Bone^\top \X^\top \X \Bone)^{-1} \Bone \X^\top \X \B \balpha_0 -\B \balpha_0}^2 \leq 2(\norm{\balpha_0}^2 \delta^2 + \norm{(\Bone^\top \X^\top \X \Bone)^{-1} \Bone \X^\top \X \Btwo}^2 \delta^2 \norm{\balpha_0}^2).
	\end{align*}
	We now turn to bounding the second error term, $\norm{(\Bone^\top \X^\top \X \Bone)^{-1} \Bone^\top \X^\top \X \B}^2$. Let $\E_1 = \Bone^\top \frac{\X^\top \X}{n} \Bone-\Bone^\top \mSigma \Bone$ and $\E_2 = \Bone^\top \frac{\X^\top \X}{n} \B - \Bone \mSigma \B$.  Applying \cref{lem:two_subspace_conc} to the matrix $\E_1$ with $\delta = n^{-200}$ and assuming $n \gtrsim \cond^2 r \log(1/\delta) \gtrsim \cond^2 r \log n$ shows that $\norm{(\Bone^\top \mSigma \Bone)^{-1} \E_1} \leq \frac{1}{4}$ and $\norm{\E_1} \leq O(\Cmax \sqrt{\frac{r \log n}{n}})$ with probability at least $1-O(n^{-100})$. A further application of \cref{lem:mat_pert} shows that $(\Bone^\top \frac{\X^\top \X}{n} \Bone)^{-1} = (\E_1 + \Bone^\top \mSigma \Bone)^{-1} = (\Bone^\top \mSigma \Bone)^{-1} + \F $, where $\norm{\F} \leq \frac{4}{3} \norm{(\Bone^\top \mSigma \Bone)^{-1}} \norm{\E_1 (\Bone^\top \mSigma \Bone)^{-1}}$ on this event. Similarly, defining $\Bone^\top \frac{\X^\top \X}{n} \B = \E_2 + \Bone^\top \mSigma \B$	and applying \cref{lem:two_subspace_conc} again but to the matrix $\E_2$ with $\delta = n^{-200}$ and assuming $n \gtrsim \cond^2 r \log(1/\delta) \gtrsim \cond^2 r \log n$,	guarantees that $\norm{\E_2} \leq O(\Cmax (\sqrt{\frac{r \log n}{n}})$ with probability at least $1-O(n^{-100})$. 
		
		Hence on the intersection of these two events, 
		\begin{align*}& \norm{(\Bone^\top \frac{\X^\top \X}{n} \Bone)^{-1} \Bone^\top \frac{\X^\top \X}{n} \B} = \norm{((\Bone^\top \mSigma \Bone)^{-1}+\F)(\Bone^\top \mSigma \B + \E_2)} \leq \\
		& \norm{(\Bone^\top \mSigma \Bone)^{-1} \Bone^\top \mSigma \B} + \norm{(\Bone^\top \mSigma \Bone)^{-1} \E_2} + \norm{\F \Bone^\top \mSigma \B} + \norm{\E_2 \F} \leq \\
		& \cond + O(\cond \sqrt{\frac{r \log n}{n}}) + O(\cond^2 \sqrt{\frac{r \log n}{n}}) + O(\cond^2 \frac{r \log n}{n}) \leq \\
		& \cond + O(\cond^2 \sqrt{\frac{r \log n}{n}}) = O(\cond),
		\end{align*}
		under the condition $n \gtrsim \cond^2 \frac{r \log n}{n}$.
		Taking a union bound over the aforementioned events and combining terms gives the result.
\end{proof}

Finally we present a concentration result for random matrices showing concentration when the matrices are projected along two (potentially different) subspaces.

\begin{lemma}
	Suppose a sequence of i.i.d. covariates $\x_i$ satisfy the design assumptions in \cref{assump:sg_design}. Then, if $\A$ and $\B$ are both rank $r$ orthonormal projection matrices,	
	\begin{align*}
		\norm{(\A^\top \frac{\X^\top \X}{n} \B) - \A^\top \mSigma \B)}\leq O(\Cmax (\sqrt{\frac{r}{n}}+\frac{r}{n} + \sqrt{\frac{\log(1/\delta)}{n}} + \frac{\log(1/\delta)}{n})),
	\end{align*}
	with probability at least $1-\delta$. 
	\label{lem:two_subspace_conc}
\end{lemma}
\begin{proof}
The result follows by a standard sub-exponential tail bound and covering argument. First note that for any fixed $\u, \v \in \mathbb{S}^{r-1}$, we have that $\u^\top \A^\top \x_i$ and $\v^\top \B^\top \x_i$ are both $\sG(\sqrt{\Cmax})$. Hence for any fixed $\u, \v \in \mathbb{S}^{r-1}$, $\u^\top \A^\top \x_i \x_i^\top \B \v - \u^\top \A^\top \mSigma \B \v$ is $\sE(8 \Cmax, 8 \Cmax)$.

Now, let $S$ denote a $\epsilon$-cover of $\mathbb{S}^{r-1}$ which has cardinality at most $(\frac{3}{\epsilon})^r$ by a volume-covering argument. Hence for $\epsilon=\frac{1}{5}$,
\begin{align*}
	\Pr[\sup_{\u \in S^{r-1}, \v \in S^{r-1}}  \u^\top \A^\top \x_i \x_i^\top \B \v - \u^\top \A^\top \mSigma \B \v \geq t ] \leq 225^r \exp(- c n \min(t^2/\Cmax^2, t/\Cmax))),
\end{align*}
using a union bound over the covers and a sub-exponential tail bound. Taking $t = C \cdot \Cmax (\sqrt{\frac{r}{n}}+\frac{r}{n} + \sqrt{\frac{\log(1/\delta)}{n}} + \frac{\log(1/\delta)}{n})$ for sufficiently large $C$, shows that $225^r \exp(- c n \min(t^2/\Cmax^2, t/\Cmax)))
 \leq \delta$. Finally a standard Lipschitz continuity argument yields
\begin{align*}
\lefteqn{ \sup_{\u \in \mathbb{S}^{r-1}, \v \in \mathbb{S}^{r-1}} \u^\top \frac{1}{n} \sum_{i=1}^{n} \u^\top \A^\top \x_i \x_i^\top \B \v - \u^\top \A^\top \mSigma \B \v } \\
& \leq \frac{1}{1-3 \epsilon} \sup_{\u \in S^{r-1}, \v \in S^{r-1}} \u^\top \A^\top \x_i \x_i^\top \B \v - \u^\top \A^\top \mSigma \B \v,
\end{align*}
which gives the result.
\end{proof}

\section{Proofs for \cref{sec:lf_lb}}

We begin by presenting the proof of the main statistical lower bound for recovering the feature matrix $\B$ and relevant auxiliary results. Following this we provide relevant background on Grassmann manifolds.

As mentioned in the main text our main tool is to use is a non-standard variant of the Fano method, along with suitable bounds on the cardinality of the packing number  and the distributional covering number, to obtain minimax lower bound on the difficulty of estimating $\B$. We instantiate the $f$-divergence based lower bound below (which we instantiate with $\chi^2$-divergence). We restate this result for convenience. 
\begin{lemma}\citep[Theorem 4.1]{guntuboyina2011lower}
For any increasing function $\ell : [0, \infty) \to [0, \infty)$,
\begin{align}
\inf_{\hat{\theta}} \sup_{\theta \in \Theta} \Pr_{\theta}[\ell(\rho(\hat{\theta}, \theta)) \geq  \ell(\eta/2)] \geq \sup_{\eta > 0, \epsilon > 0} \left \{ 1- \left(\frac{1}{N(\eta)} + \sqrt{\frac{(1+\epsilon^2) M_C(\epsilon, \Theta)}{N(\eta)}} \right) \right \}. \nonumber
\end{align}
\label{thm:yang_barron}
\end{lemma}
In the context of the previous result $N(\eta)$ denotes a lower bound on the $\eta$-packing number of the metric space $(\Theta, \rho)$. Moreover, $M_C(\epsilon, \Theta)$ is a positive real number for which there exists a set $G$ with cardinality $\leq M_C(\epsilon, S)$ and probability measures $Q_{\alpha}$, $\alpha \in G$ such that $\sup_{\theta \in S} \min_{\alpha \in G} \chi^2(\Pr_{\theta}, Q_{\alpha}) \leq \epsilon^2$, where $\chi^2$ denotes the chi-squared divergence. In words, $M_{C}(\epsilon, S)$ is an upper bound on the $\epsilon$-covering on the space $\{ \Pr_{\theta} : \theta \in S \}$ when distances are measured by the square root of the $\chi^2$-divergence.

We obtain the other term capturing the difficulty of estimating $\hB$ with respect to the task diversity minimum eigenvalue by constructing a pair of feature matrices which are hard to distinguish for a particular ill-conditioned task matrix $\A$.
Using these two results we provide the proof of our main lower bound.

\begin{proof}[Proof of \cref{thm:feature_lb_yb}]
	 For our present purposes we simply take $\ell(\cdot)$ to be the identity function in our application of \cref{thm:yang_barron} as we obtain the second term in the lower bound. Then by the duality between packing and covering numbers we have that $\log N \geq \log M$ at the same scale (see for example \citet[Lemma 5.5]{wainwright2019high}), so once again by \cref{prop:pajor} we have that $\log N(\eta) \geq  r(d-r) \log(\frac{c_1}{\eta})$. Then applying \cref{lem:mut_info_yb} we have that $M_C(\epsilon, \Theta') \leq (\frac{c_2 n }{\log(1+\epsilon^2)})^{r(d-r)/2}$. For convenience we set $k=r(d-r)$ in the following. We now choose the pair $\eta$ and $\epsilon$ appropriately in \cref{thm:yang_barron}. The lower bound writes as,
	 \begin{align*}
	     1-\left( \frac{1}{N(\eta)} + \sqrt{\frac{(1+\epsilon^2) M_{C}(\epsilon, \Theta)}{N(\eta)}} \right) \geq 1- \left( (\frac{\eta}{c_2})^k + (\frac{\eta}{c_2})^{k/2} \cdot (c_1 n)^{k/4} \frac{(1+\epsilon^2)^{1/2}}{\log(1+\epsilon^2)^{k/4}}\right),
	 \end{align*}
	 with the implicit constraint that $\epsilon' = \sqrt{\frac{2}{n} \log(1+\epsilon^2)} < 1$. A simple calculus argument shows that $\epsilon \to \sqrt{1+\epsilon^2}/(\log(1+\epsilon^2))^{k/4}$ is minimized when $1+\epsilon^2 = e^{k/2}$ (subject to $\sqrt{\frac{1}{2n} \log(1+\epsilon^2)} < 1)$. This constraint can always be ensured by taking $n > \frac{k}{4}$. We then have that the lower bound becomes
	 \begin{align}
	     1-\left( (\frac{\eta}{c_2})^k + (2 e \frac{c_1}{c_2^2} \eta^2 n/k)^{k/4} \right).
	 \end{align}
	 We now take $\eta = C \sqrt{k/n}$ so the bound simplifies to $1-\left( (\frac{\eta}{c_2})^k + (2 e \frac{c_1}{c_2^2} \eta^2 n /k)^{k/4} \right) = 1-\left( (\frac{C}{c_2} \sqrt{k/n})^{k} + (2e \frac{c_1}{c_2^2} C^2)^{k/4} \right)$. By choosing $C$ to be sufficiently small and taking $n > k$ we ensure that $ 1-\left( (\frac{C}{c_2} \sqrt{k/n})^{k} + (2e \frac{c_1}{c_2^2} C^2)^{k/4} \right) \geq \frac{99}{100}$. Finally, under the condition $r \leq \frac{d}{2}$ we have that $k \geq \frac{dr}{2}$. Combining with \cref{thm:yang_barron} gives the result for the second term.
	 
	 The first term is lower bounded using the LeCam two-point method in an indepedent fashion as a consequence of \cref{lem:lecam_two_point}. A union bound over the events on which the lower bounds hold give the result. Note that a single choice of $\A$ matrix can in fact be used for both lower bounds by simply opting for the choice of $\A$ used in \cref{lem:lecam_two_point}.
	 
	\end{proof}

To implement the lower bound we require an upper bound on the covering number in the space of distributions of $\Pr_{\B}$. Throughout we use standard properties of the $\chi^2$-divergence which can be found in \citet[Section 2.4]{tsybakov2008introduction}. 

\begin{lemma}
		\label{lem:mut_info_yb}
	Suppose $n$ data points, $(\x_i, y_i)$, are generated from the model in \eqref{eq:model_main} with i.i.d. covariates $\x_i \sim \cN(0, \I_d)$ and independent i.i.d. noise $\epsilon_i \sim \cN(0,1)$. Further, assume the task parameters satisfy \cref{assump:task} with each task normalized to $\norm{\balpha_i}=\frac{1}{2}$. Then if $r \leq \frac{d}{2}$,
	\begin{align*}
	    M_{C}(\epsilon, \Theta') \leq  \left (\frac{c n }{\log(1+\epsilon^2)} \right)^{r(d-r)/2},
	\end{align*}
whenever $\sqrt{\frac{1}{2n} \log(1+\epsilon^2)} < 1$.
\end{lemma}
\begin{proof}
	We first upper bound the $\chi^2$ divergence between two data distributions for two distinct $\B^{i}$ and $\B^{j}$. Now the joint distribution over the observations for each measure $\Pr_{\B^{i}}$ can be written as $\Pr_{\B^{i}} \equiv \Pi_{k=1}^{t} p(\X_k) p(\y|\X_k, \B^{i}, \balpha_j)$, where $p(\X_k)$ corresponds to the density of the Gaussian design matrix, and $p(\y|\X_k, \B^{i}, \balpha_k)$ the Gaussian conditionals of the observations $\y$.
	So using standard properties of the $\chi^2$-divergence we find that,
	\begin{align*}
		& \chi^2(\Pr_{\B^{i}}, \Pr_{\B^{j}}) = \Pi_{k=1}^{t} (1+ \mE_{\X_k} [\chi^2( p(\y | \X_k, \B^{i}), p(\y | \X_k, \B^{j}))])-1 = \\
		& \Pi_{k=1}^{t} \mE_{\X_k} \left[\exp \left(\norm{\X_k (\B^{i} \balpha_k - \B^{j} \balpha_k)}^2 \right) \right]-1.
	\end{align*}
	Now note that $\norm{(\B^{i} \balpha_k - \B^{j} \balpha_k)}^2 \leq 2\norm{\balpha_k}^2 \sigma_{1}(\I_r-(\B^{i})^\top \B^{j})$. Recognizing $\sigma_{r}( (\B^{i})^\top \B^{j}) = \cos \theta_1(\B^{i}, \B^{j})$ and $\norm{(\B_{\perp}^{i})^\top \B^{j}} = \sin \theta_1$, where $\theta_1$ is largest principal angle between the subspaces, we have that $1-\sigma_{r}( (\B^{i})^\top \B^{j}) = 1-\sqrt{1-\norm{(\B_{\perp}^{i})^\top \B^{j}}^2} \leq \norm{(\B_{\perp}^{i})^\top \B^{j}}^2 \leq 1$, using the elementary inequality $1-\sqrt{1-x^2} \leq x^2$ for $0 \leq x \leq 1$. Thus, $\norm{(\B^{i} \balpha_t - \B^{j} \balpha_t)}^2 \leq \frac{1}{2}$.
	
	Now we use the identity that for $\x \sim \mathcal{N}(0,\I_d)$, and $\norm{\v} \leq \frac{1}{2}$ that $\mE_{\x} \exp((\v^\top \x)^2) = \mE_{\x} [\exp (\norm{\v}^2 ((\frac{\v^\top}{\norm{\v}} \x)^2 - 1)] \exp (\norm{\v}^2) = \frac{\exp(-\norm{\v}^2))}{\sqrt{1-2 \norm{\v}^2}} \cdot \exp (\norm{\v}^2) \leq \exp(2 \norm{\v}^4) \cdot \exp (\norm{\v}^2) \leq \exp(2 \norm{\v}^2)$. Hence combining the above two facts we obtain, 
	\begin{align*}
		& \chi^2(\Pr_{\B^{i}}, \Pr_{\B^{j}}) \leq \exp \left(2 \sum_{k=1}^{t} n_t \norm{\B^{i} \balpha_k - \B^{j} \balpha_k}^2 \right)-1 \leq \exp(2n \cdot \sin^2 \theta(\B^i, \B^j))-1
	\end{align*}
	applying \cref{lem:lb_trace} in the inequality with a rescaling.
	Hence to ensure that $\chi^2(\Pr_{\B^{i}, \B^{j}}) \leq \epsilon^2$ we take $\B^{i}$ to be the closest element in a $\epsilon'$ cover, $S$, of $\Gr$ to $\B^{j}$. Since further $\chi^2(\Pr_{\B^{i}}, \Pr_{\B^{j}}) \leq \exp(2 n (\epsilon')^2)-1$ this is satisfied by taking $\epsilon' = \sqrt{\frac{1}{2n} \log(1+\epsilon^2)}$ (with the constraint we have $\epsilon' < 1$). Using \cref{prop:pajor}, we then obtain that,
	\begin{align*}
	    M_{C}(\epsilon, \Theta') \leq  \left (\frac{c n }{\log(1+\epsilon^2)} \right)^{r(d-r)/2},
	\end{align*}
	for a universal constant $c$.
\end{proof}

\begin{lemma}
    \label{lem:lb_trace}
    Let the task parameters $\balpha_i$ each satisfy $\norm{\balpha_i}=1$ with parameter $\nu = \sigma_r(\frac{\A^\top \A}{t}) > 0$, and let $\B^i$ and $\B^j$ be distinct, rank-$r$ orthonormal feature matrices. Then, 
    \begin{align*}
    \sum_{k=1}^{t} n_t \norm{\B^{i} \balpha_k - \B^{j} \balpha_k}^2  \leq n \sin^2 \theta(\B^i, \B^j).
    \end{align*}
    where $n=t \cdot n_t$.
\end{lemma}
\begin{proof}
We can simplify the expression as follows,
\begin{align*}
        \sum_{k=1}^{t} n_t \norm{\B^{i} \balpha_k - \B^{j} \balpha_k}^2 = n \cdot \frac{1}{t} \sum_{j=1}^t 2 \cdot (\balpha_k^\top \balpha_k-\balpha_k^\top (\B^i )^\top \B^j \balpha_k) = n \cdot \tr \left( \left(\I_r-(\B^i)^\top \B^j \right) \frac{\A^\top \A}{t} \right)
\end{align*}
using the fact that $\B^\top \B = \I_r$ for an orthonormal feature matrix, the normalization of $\norm{\balpha_k}=1$, and the cyclic property of the trace. Now use the fact that $\tr \left( \left(\I_r-(\B^i)^\top \B^j \right) \frac{\A^\top \A}{t} \right)) \leq \sigma_{\max}(\I_r-(\B^i)^\top \B^j)) \cdot  \tr(\frac{\A^\top \A}{t}) \leq \frac{1}{4}(1-\cos \theta_1) =  \frac{1}{4} (1-\sqrt{1-\sin^2 \theta_1}) \leq \frac{1}{4} \sin^2 \theta_1$,  using the elementary inequality $1-\sqrt{1-x^2} \leq x^2$ for $0 \leq x \leq 1$. Note that $\tr(\A^\top \A/t) \leq \frac{1}{4}$ follows from the normalization of the $\balpha_j$.
\end{proof}

We now provide brief background on $\Gr$ and establish several pieces of notation relevant to the discussion. We denote the Grassmann manifold, which consists of the the set of $r$-dimensional subspaces within the underlying $d$-dimensional space, as $\Gr$. Another way to define it is as the homogeneous space of the orthogonal group $O(d)$ in the sense that,
	\begin{align*}
		\Gr \cong O(d)/(O(r) \times O(d-r)),
	\end{align*}
	which defines its geometric structure. The underlying measure on the manifold $G_{r, n}(\mR)$ is the associated, normalized invariant (or Haar) measure.

 Note that each orthonormal feature matrix $\B$, is contained in an equivalence class (under orthogonal rotation) of an element in $\Gr$. To define distance in $\Gr$ we define the notion of a principal angle between two subspaces $p$ and $q$. If $\C$ is an orthonormal matrix whose columns form an orthonormal basis of $p$ and $\D$ is an orthonormal matrix whose columns form an orthonormal basis of $q$, then the singular values of the decomposition of $\C^\top \D = \U \D \V^\top$ defines the principal angles as follows:
\begin{align*}
	\D = \diag(\cos \theta_1, \cos \theta_2, \hdots, \cos \theta_k),
\end{align*}
where $0 \leq \theta_k \leq \hdots \leq \theta_1 \leq \frac{\pi}{2}$. As shorthand we let $\btheta = (\theta_1, \theta_2, \hdots, \theta_k)$, and let $\sin$ and $\cos$ act element-wise on its components. The subspace angle distance which is induced by $\ell_{\infty}$ norms on the vector $\sin \theta$.  We refer the reader to \citet{pajor1998metric} for geometric background on coding and packing/covering bounds in the context of Grassmann manifolds relevant to our discussion here.

In the following we let $M(\Gr, \sin \theta_1, \eta)$ denote the $\eta$-covering number of $\Gr$ in the subspace angle distance.
\begin{proposition}\citep[Adapted from Proposition 8]{pajor1998metric}
\label{prop:pajor}
    For any integers $1 \leq r \leq \frac{d}{2}$ and every $\epsilon > 0$, we have that,
    \begin{align*}
        r(d-r) \log(\frac{c_1}{\eta}) \leq \log M(\Gr, \sin \theta_1, \eta) \leq r(d-r) \log (\frac{c_2}{\eta}),
    \end{align*}
    for universal constants $c_1, c_2 > 0$.
\end{proposition}
\begin{proof}
Define for a linear operator $T$, $\sigma_q(T) = (\sum_{i \geq 1} \abs{s_i(T)}^p)^{1/p}$ for all $1 \leq q \leq \infty$ where $s_i(T)$ denotes its $i$th singular value.
Note that Proposition 8 in \citet{pajor1998metric} states the result in the distance metric $d(E, F) = \sigma_q(P_{E}-P_{F})$ where $P_{E}$ and $P_{F}$ denotes the projection operator onto the subspace $E$ and $F$ respectively, and $\sigma_q(T) = (\sum_{i \geq 1} \abs{s_i(T)}^p)^{1/p}$. However, as the computation in Proposition 6 of \citet{pajor1998metric} establishes, we have that $\sigma_q(P_{E}-P_{F}) = (2 \sum_{i=1}^r (1-\cos^2 \theta_i)^{q/2})^{1/q}$; taking $q \to \infty$ implies $\sigma_q(P_{E}-P_{F}) = \sin \theta_1$, and hence directly translating the result gives the statement of the proposition.
\end{proof}

Finally, we include the proof of the lower bound which captures the dependence on the task diversity parameter. The proof uses the LeCam two-point method between two problem instances which are difficult to distinguish for a particular, ill-conditioned task matrix.

\begin{lemma}
    Under the conditions of \cref{thm:feature_lb_yb}, for $n \geq \frac{1}{8 \nu}$,
    \begin{align*}
    \inf_{\Bone} \sup_{\B \in \Gr} \sin \theta(\Bone, \B) \geq \Omega \left(\sqrt{\frac{1}{\nu}} \sqrt{\frac{1}{n}} \right)
    \end{align*}
    with probability at least $\frac{3}{10}$.
    \label{lem:lecam_two_point}
\end{lemma}
\begin{proof}
First consider an ill-conditioned task matrix where the sequence of $\balpha_i = \frac{1}{2} \e_i$ for $i \in [r-1]$ but then $\balpha_r = \frac{1}{2}(\sqrt{1-b^2} \e_{r-1} + b \e_{r})$ for $0 < b < 1$ where $\e_i$ are the standard basis in $\mR^r$. Now, consider two task models for two different subspaces $\B_1$ and $\B_2$ which are distinct in a single direction. Namely we take  $\B_1 = [\e_1, \hdots, \e_r]$ and $\B_2 = [\e_1, \hdots, \e_{r-1}, \sqrt{1-a^2} \e_r  + a \e_{r+1}]$, for $0<a<1$, where $\e_i$ refer to the standard basis in $\mR^d$. Here we have $\cos \theta_1 = \norm{\B_2^\top \B_1} = \sqrt{1-a^2} \implies \sin \theta_1 = a$, where $\theta_1$ refers to the largest principle angle between the two subspaces. 

Data is generated from the two linear models as, 
\begin{align*}
	y_i = \x_i^\top \B_1 \balpha_j + \epsilon_i \ i= 1, \hdots, n \\
	y_i = \x_i^\top \B_2 \balpha_j + \epsilon_i \ i= 1, \hdots, n
\end{align*}
with $n$ total samples generated evenly from each of $j$ in $[t]$ tasks ($n_t$ from each task) inducing two measures $\Pr_1$ and $\Pr_2$ over their respective data. The LeCam two-point method (see \citet[Ch. 15]{wainwright2019high} for example)
shows that,
\begin{align}
    \inf_{\hB} \sup_{\B \in \Gr} \Pr_{\B}[\sin \theta(\hB, \B) \geq a] \geq \frac{1}{2}(1-\norm{\Pr_1-\Pr_2}_{\TV})
\end{align}
for the $\B_1$ and $\B_2$ of our choosing as above.

We can now upper bound the total variation distance (via the Pinsker inequality) similar to as in \cref{lem:mut_info_yb},
\begin{align}
	\norm{\Pr_1 - \Pr_2}_{\TV}^2 \leq \frac{1}{2} \KL(\Pr_1 | \Pr_2) = \frac{1}{4} \sum_{j=1}^t \sum_{i=1}^{n_t}\norm{\B_1 \balpha_j - \B_2 \balpha_j}_2^2 = \frac{1}{2} n_t \sum_{j=1}^t (\norm{\balpha_j}^2 - \balpha_j^\top \B_1^\top \B_2 \balpha_j ) = \frac{n}{2} \cdot (\frac{1}{4} - \tr(\B_1^\top \B_2 \C))
\end{align}
using cyclicity of the trace. Straightforward calculations show that given the $\A$ matrix, $\C = \frac{1}{4r} \begin{bmatrix}
    \I_{r-2} & \textbf{0} \\
    \textbf{0} & \M_1
\end{bmatrix}$, where 
$\M_1=
\begin{bmatrix}
		2-b^2 & b \sqrt{1-b^2} \\
		b \sqrt{1-b^2} & b^2 \\ 
\end{bmatrix}$.
Similarly $\B_1^\top \B_2 = \begin{bmatrix} \I_{r-2} & \textbf{0} \\
\textbf{0} & \M_2
\end{bmatrix}$ where $\M_2 = \begin{bmatrix}
		1 & 0 \\
		0 & \sqrt{1-a^2} \\ 
\end{bmatrix}$. Computing the trace term,
\begin{align}
	& \frac{1}{4}-\tr((\B_1^\top \B_2)\C) = \frac{1}{4}- \frac{1}{4}\left (\frac{r-2}{r} + \frac{1}{r} \cdot \tr \Big( \begin{bmatrix}
		1 & 0 \\
		0 & \sqrt{1-a^2} \\ 
	\end{bmatrix} \begin{bmatrix}
		2-b^2 & b \sqrt{1-b^2} \\
		b \sqrt{1-b^2} & b^2 \\ 
	\end{bmatrix} \Big) \right) = \\
	& \frac{1}{4r} b^2 (1-\sqrt{1-a^2}).
\end{align}
Hence,
$
\norm{\Pr_1 - \Pr_2}_{\TV}^2 \leq  \frac{n}{2r} b^2(1-\sqrt{1-a^2}) \leq \frac{n}{2r} b^2 a^2$. Combining with the LeCam two-point lemma shows that,
\begin{align}
	    \inf_{\hB} \sup_{\B \in \Gr} \Pr_{\B}[\sin \theta(\hB, \B) \geq a] \geq \frac{1}{2}(1-\norm{\Pr_1-\Pr_2}_{\TV}) \geq \frac{1}{2}( 1- \sqrt{\frac{n}{2r}} b a).
\end{align}
Taking $a=\frac{1}{2} \sqrt{\frac{r}{n}} \frac{1}{b}  < 1$ suffices to ensure the lower bound with probability at least $\frac{3}{10}$. This induces the constraint $\frac{1}{2} \sqrt{\frac{r}{n}} \frac{1}{b} < 1 \implies n > \frac{r}{4b^2}$. As a last remark note that the $\C$ matrix has maximum and minimum eigenvalues $\frac{1}{4r}(1+\sqrt{1-b^2})$ and $\frac{1}{4r}(1-\sqrt{1-b^2})$. So $\nu = \frac{1}{4r}(1-\sqrt{1-b^2}) \implies \sqrt{2} \sqrt{r \nu} \leq b \leq 2\sqrt{2} \sqrt{r \nu}$ for $0<b<1$ using the inequality $x^2/2 \leq 1-\sqrt{1-x^2} \leq x^2$. Hence it follows $a \geq \frac{1}{8} \frac{1}{\sqrt{\nu}} \sqrt{\frac{1}{n}}$ as well. Similarly the constraint can reduce too $n > \frac{1}{8 \nu}$.
\end{proof}

\section{Auxiliary Results}
Here we collect several auxiliary results. We begin by stating a simple matrix perturbation result.
\begin{lemma}
\label{lem:mat_pert}
	Let $\A$ be a positive-definite matrix and $\E$ another matrix which satisfies $\norm{\E \A^{-1}} \leq \frac{1}{4}$, then,
	\begin{align*}
		(\A + \E)^{-1} = \A^{-1} + \F,
	\end{align*}
	where $\norm{\F} \leq \frac{4}{3} \norm{\A^{-1}} \norm{\E \A^{-1}}$.
\end{lemma}
\begin{proof}
\begin{align*}
	(\A+\E)^{-1}=\A^{-1}(\I+\E \A^{-1})^{-1}.
\end{align*}
Under the condition, $\norm{\E \A^{-1}} \leq \frac{1}{4}$, $\I+\E \A^{-1}$ is invertible and has a convergent power series expansion so
\begin{align*}
	(\A+\E)^{-1} = \A^{-1}(\I-\E \A^{-1}+(\E \A^{-1})^2+\hdots) = \A^{-1} + \F,
\end{align*}
where $\F = \A^{-1} \cdot (-\E \A^{-1}+(\E \A^{-1})^2+\hdots)$. Moreover, 
\begin{align*}
	\norm{\F} \leq \norm{\A^{-1}} (\norm{\E \A^{-1}}+\norm{\E \A^{-1}}^2+\hdots) \leq \frac{\norm{\A^{-1}} \norm{\E \A^{-1}}}{1-\norm{\E\A^{-1}}} \leq \frac{4}{3} \norm{\A^{-1}} \norm{\E \A^{-1}}.
\end{align*}
\end{proof}

We now present several results related to the concentration of measure.

\begin{lemma}
    Let $x, y$ be mean-zero random variables that are both sub-gaussian with parameters $\kappa_1$ and $\kappa_2$ respectively. Then $z = xy-\mE[xy] \sim \sE(8 \kappa_1 \kappa_2, 8 \kappa_1 \kappa_2)$.
    \label{lem:prod_sg}
\end{lemma}

The proof is a standard argument and omitted. Next we prove a matrix concentration result for the individually rescaled covariance matrices of i.i.d. random variables. The proof uses a standard covering argument.
\begin{lemma}
	\label{lem:scaled_sgvec}
	Let $\X \in \mR^{n \times d}$ be a random matrix with rows $a_i \x_i$, where $\x_i$ are i.i.d. random vectors satisfying the design conditions in \cref{assump:design}. Then, 
	\begin{align*} 
		\norm{\frac{1}{n} \X^\top \X - \sumton a_i^2 \mSigma} \leq \norm{\mSigma} K^2 \max(\delta, \delta^2) \quad \text{ for } \delta = C(\sqrt{d/n}+t/\sqrt{n}),
	\end{align*}
	with probability at least $1-2\exp(-t^2)$. Here $C$ denotes a universal constant and $K = \max_i \abs{a_i}$.

\end{lemma}
\begin{proof}
First note that we bring all the vectors to isotropic position by rotating so that $\norm{\frac{1}{n} \X^\top \X - \sumton a_i^2 \mSigma} \leq \norm{\mSigma} \norm{\sumton a_i^2 (\mSigma^{-1/2} \x_i)(\mSigma^{-1/2} \x_i)^\top - \sumton a_i^2 \I_d}$. Now by definition for any fixed $\v \in \mathbb{R}^d : \norm{\v}=1$, each $\v^\top \mSigma^{-1/2} \x_i$ is $\sG(1)$ and hence $a_i^2 (\v^\top \mSigma^{-1/2} \x_i)(\mSigma^{-1/2} \x_i)^\top \v$ is $\sE(8a_i^2, 8a_i^2)$ by \cref{lem:prod_sg}. For the latter quantity \citet[Theorem 4.6.1, Eq. 4.22]{vershynin2018high} proves the result when $a_i=1$ using a standard covering argument along with a sub-exponential tail bound. A close inspection of the proof of \citet[Theorem 4.6.1, Eq. 4.22]{vershynin2018high} shows that the aforementioned analogous statement holds when the sequence of random vectors is scaled by $a_i$.
\end{proof}

We now include two useful results on operator norm bounds of higher-order matrices. The results only require the condition of $O(1)$-L4-L2 hypercontractivity (which is directly implied by \cref{assump:design}---see for example the sub-gaussian moment bounds in \citet[Theorem 2.6]{wainwright2019high}). Formally, we say a random vector $\x$ is $L$-L4-L2 hypercontractive if $\mE[\langle \v, \x \rangle^4] \leq L^2 (\mE[\langle \v, \x \rangle^2])^2$ for all unit vectors $\v$. Also note that if $\x$ is hypercontractive this immediately implies that $\pv \x$ is also hypercontractive with the same constant where $\pv$ is an orthogonal projection operator. 
\begin{lemma} \label{lem:l4l2norm}
    Let $\x$ be a mean-zero random vector from a distribution that is $L$-L$4$-L$2$ hypercontractive with covariance $\mSigma$ and let $\pv$ be an orthogonal projection operator onto a rank-$r$ subspace. Then 
    \begin{equation}
    \norm{\mE[\norm{\x}^2 \x\x^\top]} \leq L \tr(\mSigma) \norm{\mSigma}; \quad \norm{\mE[\norm{\pv \x}^2 \x\x^\top]} \leq L r \norm{\mSigma}^2 ; \quad \norm{\mE[\norm{\x}^2 \pv \x (\pv \x)^\top} \leq L \tr(\mSigma) \norm{\mSigma}.
    \end{equation}
\end{lemma}
\begin{proof}[Proof of \cref{lem:l4l2norm}]
     We introduce a vector $\v$ with $\norm{\v} \leq 1$. Then,
\begin{align}
    & \mE[\langle \v, \norm{\x}^2 \x \x^\top \v \rangle] = \mE[\norm{\x}^2 \langle \v, \x \rangle^2] \leq (\mE[\norm{\x}^4])^{1/2} (\mE[\langle \v, \x \rangle^4])^{1/2},
\end{align}
by the Cauchy-Schwarz inequality.
For the first term we have $(\mE[\norm{\x}^4])^{1/2} \leq \sqrt{L} \tr (\mSigma)$ by \cref{lem:l4l2vec}. For the second term once again using $L$-L$4$-L$2$ hypercontractivity we have
$(\mE[\langle \v, \x \rangle^4])^{1/2} \leq \sqrt{L} \mE[\langle \v, \x \rangle]^2 \leq \sqrt{L} \norm{\mSigma}$. Maximizing over $\v$ gives the result.
The remaining statements follow using an identical calculation and appealing to \cref{lem:l4l2vec}.
\end{proof}

\begin{lemma} \label[lemma]{lem:l4l2vec}
	Let $\x$ be a mean-zero random vector from a distribution that is $L$-L$4$-L$2$ hypercontractive with covariance $\mSigma$ and let $\pv$ be an orthogonal projection operator onto a rank-$r$ subspace. Then 
    \begin{align*}
      \mE[\norm{\x}^4] \leq L (\tr \mSigma)^2 ; \quad \mE[\norm{\pv \x}^2] \leq r \norm{\mSigma} ; \quad \mE[\norm{\pv \x}^2 \norm{\x}^2] \leq L r \norm{\mSigma} (\tr \mSigma) ; \norm{\mE[\norm{\pv \x}^4]} \leq L \norm{\mSigma}^2 r^2.
    \end{align*}
\end{lemma}

\begin{proof}[Proof of \cref{lem:l4l2vec}]
A short computation using the Cauchy-Schwarz inequality and L4-L2 equivalence shows that, 
\begin{align*}
     & \mE[\norm{\x}^4] = \mE[(\sum_{i=1}^d \langle \x, \e_i \rangle^2)^2] = \mE[\sum_{a,b} \langle \x, \e_a \rangle^2 \langle \x, \e_b \rangle^2] 
    \leq \sum_{a,b} (\mE[\langle \x, \e_a \rangle^4] \mE[\langle \x, \e_b \rangle^4])^{1/2} 
     \leq \\
     & L \sum_{a,b} \mE[\langle \x, \e_a \rangle^2] \mE[\langle \x, \e_b \rangle^2] \leq L (\tr (\mSigma))^2.
\end{align*}

The second statement follows since $\mE[\norm{\pv \x}^2] = \mE[\tr(\pv \x\x^\top] = \tr(\pv \mSigma) \leq r \norm{\mSigma}$
where the last line follows by the von Neumann trace inequality and the fact $\pv$ is a projection operator. The final statements follow by combining the previous calculations.
\end{proof}

\begin{lemma}
    \label{lem:lensg}
    Let $\x$ be a random vector in $d$ dimensions from a distribution satisfying \cref{assump:sg_design}. Then, we have:
    \begin{equation*}
         \norm{\x} \leq O \left( \sqrt{\Cmax} (  \sqrt{d} + \sqrt{\log 1 / \delta}) \right),
    \end{equation*}
    with probability at least $1-\delta$.
\end{lemma}
\begin{proof}[Proof of \cref{lem:lensg}]
Note that by rotating the vectors into isotropic position $\y =\mSigma^{-1/2} \x$ is $\I_d$-subgaussian in the sense of \cref{assump:sg_design}. Since $\norm{\x} \leq \sqrt{\Cmax} \norm{\mSigma^{-1/2}\x}$ it suffices to bound the norm of $\norm{\y}$.
First note that $\mE[\norm{\y}] \leq \sqrt{\mE[\norm{\y}^2]} = \sqrt{d}$.
Now, take an $1/2$-net over the unit sphere, $G$; by a standard covering argument the number of elements in $G$ can be upper bounded by $6^d$ \citep[Chapter 5]{wainwright2019high}. By definition of $\I_d$-subgaussianity, for any $\v \in G$, we have that,
    \begin{equation*}
        \Pr [\abs{\langle \v, \y \rangle} \geq t] \leq 2 \exp \left(- \frac{t^2}{2} \right)
    \end{equation*}
    By a standard continuity argument it follows that $\max_{\v \in S^{d-1}} \abs{\langle \v, \y \rangle} \leq 3 \max_{\v \in G} \abs{\langle \v, \y \rangle}$. So by a union bound, $\Pr[\norm{\y} \geq t] \leq  (12)^d \exp(-t^2/20)$.
    Therefore, by taking $t = C(\sqrt{d} + \sqrt{\log 1 / \delta})$ for large-enough constant $C$ we can ensure that $(12)^d \exp(-t^2/20) \leq \delta$. Re-arranging gives the conclusion.
\end{proof}

Finally, we prove a truncated version of the matrix Bernstein inequality we can apply to matrices that are unbounded in spectral norm. This is our primary technical tool used to show concentration of the higher-order moments used in the algorithm to recover the feature matrix $\B$.
\begin{lemma}
	\label{lem:trunc_mb}
	Consider a truncation level $R > 0$. If $Z_i$ is a sequence of symmetric independent random matrices and if $Z_i' = Z_i \Ind[\norm{Z_i} \leq R]$, then
	\begin{align*}
		\Pr[\norm{ \frac{1}{n} \sum_{i=1}^{n} Z_i - \mE[Z_i] } \geq t] \leq \Pr[ \norm{\frac{1}{n} \sum_{i=1}^{n} Z'_i - \mE[Z_i']} \geq t-\Delta] + n \Pr[\norm{Z_i} \geq R],
	\end{align*}
	where $\Delta \geq \norm{\mE[Z_i]-\mE[Z_i']}$. Further, for $t \geq \Delta$, we have that, 
	\begin{align*}
		\Pr[ \norm{\frac{1}{n} \sum_{i=1}^{n} Z'_i - \mE[Z_i']} \geq t-\Delta] \leq 2d \exp \left(\frac{n^2(t-\Delta)^2}{\sigma^2 + 2Rn(t-\Delta)/3} \right),
	\end{align*}
	where $\sigma^2 = \norm{\sum_{i=1}^n \mE[(Z_i'-\mE[Z_i'])^2]} \leq \norm{\sum_{i=1}^n \mE[Z_i^2]}$.
\end{lemma}
\begin{proof}
	The first statement follows by splitting on the event $\{ \norm{ Z_i} \leq R :  \forall i \in [n] \}$, along with a union bound, and an application of the triangle inequality to the first term. The second is simply a restatement of the matrix Bernstein inequality in \citet{tropp2012user} along with the almost sure upper bound $\norm{Z_i'-\mE[Z_i']} \leq \norm{Z_i'} + \norm{\mE[Z_i']} \leq R+\mE[\norm{Z_i'}] \leq 2R$. The final bound on the matrix variance follows from the facts that for the p.s.d. matrix $\sum_{i=1}^n \mE[(Z_i'-\mE[Z_i'])^2] \preceq \sum_{i=1}^{n} \mE[(Z_i')^2]$ and that $(Z_i')^2 \preceq Z_i^2$.
\end{proof}

\section{Experimental Details}
\label{sec:app_expt}

In our experiments we did find that gradient descent was able to decrease the loss in \cref{eq:objective_main}, but the algorithm was slow to converge. In practice, we found using the L-BFGS algorithm required no tuning and optimized the loss in \cref{eq:objective_main} to high-precision in far fewer iterations \citep{liu1989limited}. Hence we used this first-order method throughout our experiments as our optimization routine. Our implementation is in Python, and we leveraged the \texttt{autograd} package to compute derivatives of the objective in \cref{eq:objective_main}, and the package \texttt{Ray} to parallelize our experiments \citep{maclaurin2015autograd,moritz2018ray}. Each experiment is averaged over 30 repetitions with error bars representing $\pm 1$ standard deviation over the repetitions. All the experiments herein were run on computer with 48 cores and 256 GB of RAM.

Note that after optimizing \cref{eq:objective_main} directly using a first-order method in the variable ($\U, \V$), we can extract an estimate $\Bone$ of $\B$ by computing the column space of $\V$ (for example using the SVD of $\V$ or applying the Gram-Schmidt algorithm).

\end{document}